\documentclass[nohyperref]{article}

\usepackage{microtype}
\usepackage{graphicx}
\usepackage{subcaption}

\usepackage{booktabs} %

\usepackage{hyperref}
\usepackage{enumitem}

\usepackage[accepted]{icml2023}

\usepackage{amsmath}
\usepackage{amssymb}
\usepackage{mathtools}
\usepackage{amsthm}

\usepackage[capitalize,noabbrev]{cleveref}

\theoremstyle{plain}
\newtheorem{theorem}{Theorem}[section]
\newtheorem{proposition}[theorem]{Proposition}

\theoremstyle{definition}

\theoremstyle{remark}
\newtheorem{remark}[theorem]{Remark}

\usepackage[textsize=tiny]{todonotes}

\DeclareMathOperator{\sgn}{sgn}
\DeclareMathOperator*{\argmin}{arg\,min}

\DeclarePairedDelimiterX{\inp}[2]{\langle}{\rangle}{#1, #2}

\usepackage{xspace}
\usepackage{multirow}
\usepackage{makecell}
\usepackage{cleveref}

\usepackage[super]{nth}
\newcommand{\eps}[0]{\varepsilon}
\newcommand{\rqone}[0]{\texttt{RQ1}\xspace}
\newcommand{\rqtwo}[0]{\texttt{RQ2}\xspace}
\newcommand{\rqthree}[0]{\texttt{RQ3}\xspace}
\newcommand{\ansone}[0]{\texttt{A1}\xspace}
\newcommand{\anstwo}[0]{\texttt{A2}\xspace}
\newcommand{\ansthree}[0]{\texttt{A3}\xspace}
\newcommand{\tstd}[0]{\text{std}}
\newcommand{\trob}[0]{\text{rob}}
\newcommand{\ie}[0]{\textit{i.e.}}
\newcommand{\eg}[0]{\textit{e.g.}}

\definecolor{dkgreen}{rgb}{0,0.6,0}
\definecolor{gray}{rgb}{0.5,0.5,0.5}
\definecolor{mauve}{rgb}{0.58,0,0.82}

\icmltitlerunning{Understanding the Impact of Adversarial Robustness on Accuracy Disparity}

\begin{document}

\twocolumn[
\icmltitle{Understanding the Impact of Adversarial Robustness on Accuracy Disparity}

\icmlsetsymbol{equal}{*}

\begin{icmlauthorlist}
\icmlauthor{Yuzheng Hu}{xxx}
\icmlauthor{Fan Wu}{xxx}
\icmlauthor{Hongyang Zhang}{yyy}
\icmlauthor{Han Zhao}{xxx}
\end{icmlauthorlist}

\icmlaffiliation{xxx}{Department of Computer Science, University of Illinois at Urbana Champaign, Urbana, IL, USA}
\icmlaffiliation{yyy}{David R. Cheriton School of Computer Science, University of Waterloo, Waterloo, ON, Canada}

\icmlcorrespondingauthor{Yuzheng Hu}{yh46@illinois.edu}

\icmlkeywords{Machine Learning, ICML}

\vskip 0.3in
]

\printAffiliationsAndNotice{}  %

\begin{abstract}
  While it has long been empirically observed that adversarial robustness may be at odds with standard accuracy and may have further disparate impacts on different classes, it remains an open question to what extent such observations hold and how the class imbalance plays a role within. In this paper, we attempt to understand this question of accuracy disparity by taking a closer look at linear classifiers under a Gaussian mixture model. We decompose the impact of adversarial robustness into two parts: an inherent effect that will degrade the standard accuracy on all classes due to the robustness constraint, and the other caused by the class imbalance ratio, which will increase the accuracy disparity compared to standard training. Furthermore, we also show that such effects extend beyond the Gaussian mixture model, by generalizing our data model to the general family of stable distributions. More specifically, we demonstrate that while the constraint of adversarial robustness consistently degrades the standard accuracy in the balanced class setting, the class imbalance ratio plays a fundamentally different role in accuracy disparity compared to the Gaussian case, due to the heavy tail of the stable distribution. We additionally perform experiments on both synthetic and real-world datasets to corroborate our theoretical findings. Our empirical results also suggest that the implications may extend to nonlinear models over real-world datasets. Our code is publicly available on GitHub\footnote{ \url{https://github.com/Accuracy-Disparity/AT-on-AD}}. 
\end{abstract}

\section{Introduction}
The existence and prevalence of adversarial examples~\citep{dalvi2004adversarial, szegedy2013intriguing, DBLP:journals/corr/GoodfellowSS14} in the state-of-the-art deep learning models have made adversarial robustness an active field of research, where human-imperceptible perturbations to the original data can arbitrarily disrupt the model prediction.  
However, it has been empirically observed that the improvement in robustness usually comes with costs in accuracy. In particular, there might exist a trade-off between robust accuracy and standard accuracy~\citep{tsipras2018robustness, tutorial_madry_and_kolter}, and it may also lead to the so-called \textit{accuracy disparity}~\citep{chi2021understanding}, a notion of unfairness in the literature of algorithmic fairness. In particular, it is shown that when compared to standard training, the constraint of adversarial robustness might further exacerbate the \textit{discrepancy} of standard accuracy among difference classes~\citep{croce2021robustbench, benz2021triangular, benz2021robustness}.

Despite fruitful and intriguing empirical observations, rigorous understanding of how adversarial robustness affects standard accuracy or accuracy disparity has not been extensively explored from a theoretical perspective. In fact, it is not clear whether such a trade-off is inherent, even in the linear setting under a Gaussian mixture model. If it is, then what are the fundamental factors that contribute to this potential drop of accuracy and the increase of accuracy disparity?
To the best of our knowledge, there are only a few works~\citep{tsipras2018robustness, xu2021robust,ma2022on} that partially attempt to approach these problems. However, the existing analyses are restricted to examples with specific choices of parameters, which oversimplifies the problem and makes it 
unclear whether the conclusions continue to hold in more general settings. We provide further discussions on the related work in~\Cref{sec:morerw}.

\textbf{Our Contributions.}~~ 
Towards answering the above questions, we provide a theoretical study of the impact of adversarial robustness on accuracy disparity in the presence of class imbalance~\citep{johnson2019survey}. We consider the classification problem under a common Gaussian mixture model with linear classifiers. We then further generalize our analysis to a broader family of stable distributions. For each data distribution, we decompose the impact of adversarial robustness into two parts, an intrinsic part that leads to the drop of standard accuracy and the class imbalance part which will affect the accuracy disparity. 
Our main contributions are summarized as follows:
\begin{itemize}%
    \item Under a common Gaussian mixture model and linear classifiers, we reveal two effects of enforcing adversarial robustness in binary classification. The first part is inherent to the constraint of adversarial robustness itself, which will degrade the standard accuracy due to a change of direction of the optimal linear classifier. The second part is caused by the class imbalance ratio between the two classes under consideration, which will increase the accuracy disparity compared to standard training due to a reduction of norm of the optimal linear classifier. 
    
    \item Inspired by our analysis, we further point out the equivalence between robust training in classification and regularized linear regression. Our observation helps to explain the norm-shrinkage effect that happens in robust learning, which could be of independent theoretical interest.

    \item Going beyond the Gaussian mixture model, we show that when the data follow a mixture of heavy-tailed stable distributions, the intrinsic effect of robustness persists and the standard accuracy consistently decreases even when the data is balanced among the two classes. On the other hand, the class imbalance ratio plays a fundamentally different role compared to that in the Gaussian case, suggesting that the tail property of the distribution also serves as a crucial factor in the accuracy disparity problem.
    
    \item We conduct experiments on both synthetic and real-world datasets. The empirical results not only corroborate our theoretical findings, but also suggest that the implications may extend to nonlinear models over real-world datasets.
\end{itemize}

\section{Preliminaries}
\label{sec:pre}
We first give an overview of the problem studied in this paper. We then proceed to introduce the necessary background and notation used throughout the paper.

\textbf{Problem setup.}\quad
\label{subsec:pre-notation}
For the ease of presentation, we focus on a binary classification task in this paper. To start with, we assume that the data are generated through a mixture of distribution $\mathcal{P}^+$ and $\mathcal{P}^-$: conditioning on $y=\pm 1$, we have $X \sim \mathcal{P}^\pm$, respectively.  
As we shall see shortly, one crucial ingredient in understanding the impact of adversarial robustness on accuracy disparity is the imbalance factor $R > 1$ between the marginal probabilities of different classes:
$R := \Pr(y=-1) / \Pr(y=+1)$,
meaning that there is a larger portion of negative-class examples in the population. %
Following prior works~\citep{tsipras2018robustness, xu2021robust}, for the model, we consider a linear classifier and couple it with a sign function $\sgn$ to obtain the output
$f(x;w,b) := \sgn(w^\top x + b)$. Finally, since we are mainly interested in understanding the inherent impact caused by adversarial robustness, throughout the paper we shall focus on the infinite data regime to remove the noise introduced by finite samples, meaning that we will study the population losses instead of their empirical counterparts.

\textbf{Objective functions.}\quad
The standard $0$-$1$ population loss (\textit{standard loss}) is defined as follows:
\vspace{-2mm}
\begin{center}
\resizebox{.98\linewidth}{!}{
\begin{minipage}{\linewidth}
\begin{align*}
    &\ell_{\text{std}}(w,b) := \Pr(f(x;w,b) \neq y) \\
    =&\frac{R}{R+1} \underbrace{\mathbb{P}(w^{\top}x+ b \ge 0 \mid y = -1)}_{\text{Part \ I}: \ \ell_{\text{std}}^-} + \frac{1}{R+1}\underbrace{\mathbb{P}(w^{\top}x+ b \le 0 \mid y = +1)}_{\text{Part \ II}: \ \ell_{\text{std}}^+}.
\end{align*}
\end{minipage}
}
\end{center}
Similarly, the $0$-$1$ adversarial loss~\citep{kurakin2016adversarial, madry2017towards} under $\ell_p$-perturbation $(p \geq 1)$ and radius $\varepsilon$ (\textit{robust $\ell_p$ loss}), is defined as:
\vspace{-2mm}
\begin{center}
\resizebox{.85\linewidth}{!}{
\begin{minipage}{\linewidth}
\begin{align*}
    &\ell_{\text{rob}, p, \varepsilon}(w,b) := \Pr(\exists \|\delta\|_p \le \varepsilon, \  s.t. \ f(x+\delta;w,b) \ne y)   \\
    =& \frac{R}{R+1}{\mathbb{P}(w^{\top}x+ b \ge -\varepsilon \|w\|_q \mid y = -1)} \\+& \frac{1}{R+1}{\mathbb{P}(w^{\top}x+ b \le \varepsilon \|w\|_q \mid y = +1)},
\end{align*}
\end{minipage}
}
\end{center}
where $1/p+1/q = 1$ and the second equation follows from the H\"older's inequality. When the context is clear, we will omit the use of $\varepsilon$ in $\ell_{\text{rob}, p, \varepsilon}$. It is straightforward to see that minimizing $\ell_{\text{rob}, p}$ will lead to adversarial robustness.

\textbf{Accuracy disparity.} \quad
Note that Part II and I in the definition of the standard loss are exactly the population loss for the minority and majority classes, which we denote as $\ell_{\text{std}}^+$ and $\ell_{\text{std}}^-$, respectively. The standard accuracy for both classes then writes as $acc^\pm := 1 - \ell_{\text{std}}^\pm$.
The key quantity that we will focus on in this paper is the \textit{accuracy disparity}~\citep{chi2021understanding} between the two classes, defined as
\resizebox{.98\linewidth}{!}{
\begin{minipage}{\linewidth}
\begin{align*}
    AD(w,b) := acc^-(w,b) - acc^+(w,b) = \ell_{\text{std}}^+(w,b) - \ell_{\text{std}}^-(w,b).
\end{align*}
\end{minipage}
}

The notion of accuracy disparity in our context focuses on the performance gap of a model on different sub-groups of the overall population, where each group is indexed by the corresponding class label~\citep{santurkar2021breeds,xu2021robust}. Accuracy disparity has recently gained more attention in the literature of algorithmic fairness~\citep{buolamwini2018gender,chi2021understanding,nanda2021fairness}, and we are interested in understanding the role of robustness 
in accuracy disparity. To proceed, we first characterize the optimal solutions of the standard loss and the robust $\ell_p$ loss
\resizebox{.98\linewidth}{!}{
\begin{minipage}{\linewidth}
\begin{align*}
    w_{\text{std}}, b_{\text{std}} = \argmin_{w,b} \ell_{\text{std}}(w,b), \ w_{\text{rob}, p}, b_{\text{rob},p} = \argmin_{w,b}  \ell_{\text{rob}, p}(w,b).
\end{align*}
\end{minipage}
}

We then analyze the changes of $\ell_{\text{std}}^+$ and $\ell_{\text{std}}^-$ as well as the accuracy disparity when we \textit{switch} the {measurements} from the optimal standard classifier to the optimal robust classifier. Specifically, we are most interested in how these changes depend on the class imbalance ratio $R$. In other words, we will try to understand the price (\ie, the decrease of the standard loss) paid by a robust classifier in the presence of class imbalance. We emphasize that  accuracy disparity is defined over the standard loss for both the standard classifier and robust classifier. There is another concept, known as \textit{robustness disparity}~\citep{nanda2021fairness}, which is defined over the robust loss; it is orthogonal to accuracy disparity and will not be covered in this paper.

\textbf{Notation.}\quad
We use $\Phi$ to denote the cumulative distribution function of %
the standard normal distribution. For a vector $u$, $|u|$ means taking the absolute value per coordinate, $u_i$ denotes the $i$-th coordinate of $u$, and $\|u\|_p$ represents the $\ell_p$-norm of $u$. For two vectors $v$ and $v'$, we use the notation $v \parallel v'$ to indicate they are parallel and $v \nparallel v'$ when they are not. We denote $1_d$ as the all-one vector with dimension $d$. For $1 \le p \le \infty$, we use $q$ to denote its dual index, which is defined through $1/p + 1/q = 1$. Finally, for a multivariate function $f$, %
the subdifferential set at $x$ is defined as $\partial f(x) := \{g: \forall y, f(y) \ge f(x) + g^{\top}(y-x)\}$, and each element within is called a subgradient.

\section{Gaussian Mixture: Robustness Implies Accuracy Disparity} \label{Section AT}

In this section, we will closely examine the accuracy disparity of robust classifiers when the data are drawn according to a Gaussian mixture~\citep{reynolds2009gaussian}. Specifically, we are interested in how enforcing the model to be robust affects the accuracy disparity compared to standard training.
While it has been shown in a previous work~\citep{xu2021robust} that adversarial robustness does introduce severe accuracy disparity when different classes exhibit different ``difficulty levels'' of learning (\ie, different magnitude of variance) in a toy example (as indicated by specific choices of mean, variance, as well as $p=\infty$), in this section we consider a more general setting where \textit{class imbalance} is present, and we shall provide a comprehensive analysis by considering the following objectives and data distributions: 1) adversarial robustness with general $\ell_p$-constraint ($1\le p \le \infty$); 2) a Gaussian mixture distribution with arbitrary mean and covariance matrix, specifically, $\mathcal{P}^+ = \mathcal{N}(\theta^+, \Sigma)$ and $\mathcal{P}^- = \mathcal{N}(\theta^-, \Sigma)$. 

In a nutshell, the overall effects of adversarial robustness are separated into two parts: an inherent one that will \textbf{decrease the standard accuracy} on all classes, and the other caused additionally by the class imbalance ratio that will \textbf{increase the accuracy disparity} compared to standard training. 

Before introducing the main results, we highlight that linear models are sufficiently powerful for the classification of Gaussian mixtures in both the standard and adversarial sense. In fact, it is well known that the Bayes-optimal classifier of the standard loss is linear due to the Fisher’s linear discriminant~\citep{johnson2014applied}; and it is further shown in~\citep{dobriban2020provable} that the Bayes-optimal robust classifier is also linear.

\medskip

\subsection{Main Results}  \label{subsec.main}
\looseness=-1
In what follows we will present a general result (Theorem \ref{thm:general}) which characterizes the class-wise standard loss for the optimal standard and robust classifiers. We will then discuss several  
implications. Specifically, Theorem \ref{corollary.degrade} (a direct corollary of Proposition \ref{Claim.direction}) demonstrates the effect of ``decreased standard accuracy'', while Theorem \ref{cor:imbalance} (derived form Proposition \ref{Claim norm}) demonstrates the effect of ``increased accuracy disparity''. 
The proofs are deferred to Appendix \ref{appsub.main}.
\begin{theorem} \label{thm:general}
Given the means $\theta^+, \theta^-$, covariance matrix $\Sigma$ and $\ell_p$-constraint, let $u,v \in \mathbb{R}^d$ satisfy
\begin{align} \label{Eq.optimal}
    \Sigma u = \theta^+ - \theta^-, \quad 
 \quad \Sigma v = \theta^+ - \theta^- - 2\varepsilon\partial{\|v\|_q},
\end{align}
 and $q$ satisfy $1/p + 1/q = 1$. We further set $r^2 := u^{\top}\Sigma u$, and $s^2 = v^{\top}\Sigma v$. Then the class-wise standard loss $\ell_{\emph{std}}^\pm$ %
of the optimal standard classifier $(w_{\text{std}} := u / r, b_{\text{std}})$ satisfy %
\vspace{-3mm}
\begin{center}
\resizebox{.99\linewidth}{!}{
\begin{minipage}{\linewidth}
\begin{align*}
    \ell_{\emph{std}}^+(w_{\emph{std}}, b_{\emph{std}}) &= \Phi\left(\frac{-\langle u, \theta^+ -\theta^- \rangle + 2\log R}{2r}\right), \\ %
    \ell_{\emph{std}}^-(w_{\emph{std}}, b_{\emph{std}}) &= \Phi\left(\frac{-\langle u, \theta^+ -\theta^- \rangle - 2\log R}{2r}\right),
\end{align*}
\end{minipage}
}
\end{center}
and the class-wise standard loss $\ell_{\emph{std}}^\pm$ %
of the optimal robust $\ell_p$ classifier
$(w_{\text{rob},p} := v / s, b_{\text{rob},p})$ satisfy %
\vspace{-3mm}
\begin{center}
\resizebox{.99\linewidth}{!}{
\begin{minipage}{\linewidth}
\begin{align*}
    \ell_{\emph{std}}^+(w_{\emph{rob}, p}, b_{\emph{rob}, p}) &= \Phi\left(\frac{-\langle v, \theta^+ -\theta^- \rangle + 2\log R}{2s}\right), \\
    \ell_{\emph{std}}^-(w_{\emph{rob}, p}, b_{\emph{rob}, p}) &= \Phi\left(\frac{-\langle v, \theta^+ -\theta^- \rangle - 2\log R}{2s}\right).
\end{align*}
\end{minipage}
}
\end{center}
\end{theorem}
It is straightforward to see that the overall effects consist of two parts: one intrinsic due to adversarial robustness ($\frac{\langle u, \theta^+ -\theta^- \rangle}{2r}$ v.s. $\frac{\langle v, \theta^+ -\theta^- \rangle}{2s}$), and the other caused by the class imbalance ratio ($\frac{\log R}{r}$ v.s. $\frac{\log R}{s}$). To compare the optimal robust classifier and the optimal standard classifier, we will show in the following analysis that 1) the intrinsic part corresponds to a change in \textbf{direction}, and will degrade the standard performance on both classes --- exactly the price paid by a classifier to be robust; 2) the class imbalance part corresponds to a change in \textbf{norm}, and will increase
the error on the minority class while decreasing the error on the majority class --- hence exacerbating the accuracy disparity.

\paragraph{The intrinsic part---decreasing the standard accuracy.} The intrinsic part corresponds to a change in \textit{direction} --- in fact, we will show that the optimal solution moves to a direction that incurs larger error due to the constraint of robustness. This helps to explain the long-observed empirical phenomenon that adversarial training~\citep{kurakin2016adversarial, madry2017towards}, which is an effective algorithm for empirical robustness, often leads to degraded standard accuracy~\citep{tsipras2018robustness, tutorial_madry_and_kolter}.

\begin{proposition} \label{Claim.direction}
For the intrinsic part, we have
\begin{align*}
    \frac{\langle u, \theta^+ -\theta^- \rangle}{2r} \ge \frac{\langle v, \theta^+ -\theta^- \rangle}{2s}.
\end{align*}
Furthermore, so long as $\Sigma^{\frac{1}{2}}w_{\emph{std}} \nparallel \Sigma^{\frac{1}{2}}w_{\emph{rob},p}$, the inequality holds strictly.
\end{proposition}

When the covariance matrix is invertible, we are  comparing the directions of the two optimal classifiers in the standard sense. For general covariance matrix, it is therefore natural for us to interpret $\Sigma^{\frac{1}{2}}w$ as a general ``direction'' of $w$. Proposition \ref{Claim.direction} now directly leads to the following result.
\begin{theorem}  \label{corollary.degrade}
When there is no class imbalance, \ie, $R=1$, {enforcing adversarial robustness with $\ell_p$-constraint will degrade the standard accuracy on both classes}, so long as the ``direction'' of the optimal robust classifier is not parallel to its counterpart, \ie, $\Sigma^{\frac{1}{2}}w_{\emph{std}} \nparallel \Sigma^{\frac{1}{2}}w_{\emph{rob},p}$. %
\end{theorem}
\begin{remark}
Intuitively, we expect that the direction of the optimal classifier to be different for the standard and robust loss. We will provide more discussions about this observation in Section~\ref{subsect.positive}, where we demonstrate that this is indeed the situation for diagonal matrices; but we also identify a few cases where we can actually get a win-win from both worlds.
\end{remark}

\textbf{The class imbalance part---increasing the accuracy disparity.}\quad
The class imbalance part corresponds to a change in \textit{norm}, which appears as the square root of quadratic form. We will show: due to a \textit{shrinkage of norm}, the standard loss of the minority class increases while the opposite happens for the majority class. 

\begin{proposition} \label{Claim norm}
For the class imbalance part, we have $u^{\top}\Sigma u > v^{\top}\Sigma v$, which implies $r>s$.
\end{proposition}

In the special case $\Sigma = \mathbb{I}_d$,  the square root of the quadratic form becomes the standard Euclidean norm. It is therefore natural for us to interpret the quantities $r$ and $s$ as realizations of a general ``norm'' (in fact, it is a seminorm defined by $\Sigma$) . With Proposition \ref{Claim norm} at hand, we are now ready to state the following result, which says that {class imbalance will increase the accuracy disparity due to the constraint of robustness}, and such growth is monotonic with respect to $R$ in a reasonable range. This demonstrates an \textbf{inherent trade-off} of adversarial robustness and accuracy parity.
\begin{theorem}\label{cor:imbalance}
Define 
$
    g(R) := AD(w_{\emph{rob},p}, b_{\emph{rob},p}) - AD(w_{\emph{std}}, b_{\emph{std}})
$
as the accuracy disparity gap. Then
\begin{itemize}[leftmargin=*, noitemsep, nolistsep]
    \item When $R>1$, we have $g(R) > 0$, meaning that the accuracy disparity of the optimal robust classifier is larger than that of the optimal standard classifier;
    \item When $R$ satisfies
$\ell_{\emph{std}}^+(w_{\emph{rob}, p}, b_{\emph{rob}, p}) \le 0.5$,
i.e., the four class-wise losses defined in Theorem \ref{thm:general} are upper-bounded by $0.5$, $g(R)$ is an increasing function w.r.t. $R$.
\end{itemize}
\end{theorem}
Note that the direction and norm of the normal vector of a linear classifier determines its decision boundary. Hence, Theorem~\ref{thm:general} and the above discussions completely characterize and provide a fine-grained analysis of the impact of adversarial robustness on the accuracy disparity of linear classifiers over mixture of Gaussian distributions.

\subsection{An Illustrating Example} 
We will now use an example that has been proposed and studied in~\citep{tsipras2018robustness} to illustrate our main results and demonstrate that Theorem~\ref{thm:general} can be used to recover and refine existing claims on this example. Specifically, let 
\begin{align*}
{\theta^+} = (\underbrace{\eta, \cdots, \eta}_{\text{dim} = m}, \underbrace{\gamma, \cdots, \gamma}_{\text{dim} = n})^{\top}, \quad \theta^- = -\theta^+,    
\end{align*}
where $\gamma < \varepsilon < \eta$ and $m+n=d$. This corresponds to two sets of features in the input space: the \textit{robust features} (coordinates with value $\eta$) and the \textit{non-robust features} (coordinates with value $\gamma$). Compared to~\citet{tsipras2018robustness}, here we slightly modify the setting by making all coordinates Gaussian and allowing the number of robust feature $m$ to be larger than one, though in general we would still expect $m \ll n$.     %
For the other assumptions, we follow~\citet{tsipras2018robustness} to set $\Sigma= \mathbb{I}_d$ and consider the standard perturbation scheme, \ie, $p = \infty$.

\textbf{The optimal standard and robust classifier.} \quad
We can assume without loss of generality that $\|w\|_2=1$, and the optimal slope and intercept for the standard and robust classifier are given by (details deferred to Appendix \ref{appsub.toy}): %

\vspace{-5mm}

\begin{center}
\resizebox{.95\linewidth}{!}{
\begin{minipage}{\linewidth}
\begin{align*}
    &b_{\text{std}} = -\frac{\log R}{2\sqrt{m\eta^2 + n\gamma^2}}, \\
    &w_{\text{std},1} = \cdots = w_{\text{std}, m} = \frac{\eta}{\sqrt{m\eta^2 + n\gamma^2}}, \\
    &w_{\text{std}, m+1} = \cdots = w_{\text{std}, m+n} = \frac{\gamma}{\sqrt{m\eta^2 + n\gamma^2}},
     \\
    &b_{\text{rob},\infty} = -\frac{\log R}{2(\eta-\varepsilon)\sqrt{m}},\\
    &w_{\text{rob}, \infty, 1} = \cdots = w_{\text{rob}, \infty, m} = \frac{1}{\sqrt{m}}, \\ 
    &w_{\text{rob}, \infty, m+1} = \cdots = w_{\text{rob}, \infty, m+n} = 0.
\end{align*}
\end{minipage}
}
\end{center}

Therefore, the standard loss for both classes are
\vspace{-3mm}
\begin{center}
\resizebox{.85\linewidth}{!}{
\begin{minipage}{\linewidth}
\begin{align*}
&\ell_{\text{std}}^+(w_{\text{std}}, b_{\text{std}}) = \Phi\left(\frac{\log R}{2\sqrt{m\eta^2 + n\gamma^2}} - \sqrt{m\eta^2 + n\gamma^2}\right),\\ %
&\ell_{\text{std}}^-(w_{\text{std}}, b_{\text{std}}) = \Phi\left(-\frac{\log R}{2\sqrt{m\eta^2 + n\gamma^2}} - \sqrt{m\eta^2 + n\gamma^2}\right), \\
    &\ell_{\text{std}}^+(w_{\text{rob},\infty}, b_{\text{rob},\infty}) = \Phi\left(\frac{\log R}{2\sqrt{m(\eta-\varepsilon)^2}} - \sqrt{m\eta^2}\right),\\
    &\ell_{\text{std}}^-(w_{\text{rob},\infty}, b_{\text{rob},\infty}) = \Phi\left(-\frac{\log R}{2\sqrt{m(\eta-\varepsilon)^2}} - \sqrt{m\eta^2}\right).
\end{align*}
\end{minipage}
}
\end{center}
By comparing $\ell_{\text{std}}^+$ and $\ell_{\text{std}}^-$ of the optimal standard classifier as well as the optimal robust classifier, we shall see that the effects of adversarial robustness indeed match the main results. Specifically,
\begin{itemize}%
    \item When $R=1$, enforcing adversarial robustness will decrease the standard accuracy on \textit{both} classes due to a shrinkage on the non-robust features ($\sqrt{m\eta^2 + n\gamma^2}$ v.s. $\sqrt{m\eta^2}$), which reflects ``a change in direction'';%
    \item When $R>1$, class imbalance will increase the accuracy disparity due to a shrinkage on both the robust features and the non-robust features ($\sqrt{m\eta^2 + n\gamma^2}$ v.s. $\sqrt{m(\eta-\varepsilon)^2}$ in the denominator), which reflects ``a reduction of norm''. %

\end{itemize}

Furthermore, when the number of robust features is small and the number of non-robust features is relatively large (corresponding to the case $m \ll n$~\citep{tsipras2018robustness, ilyas2019adversarial}), the {overall} effect tends to be a \textit{significant growth in accuracy disparity}. As a concrete example, if we set $m=4, n=48,\eta=1, \varepsilon = 0.75, \gamma=0.5, R = e^2$, we have $\ell_{\text{std}}^+(w_{\text{std}}, b_{\text{std}}), \ell_{\text{std}}^-(w_{\text{std}},  b_{\text{std}}),\ell_{\text{std}}^-(w_{\text{rob},\infty}, b_{\text{rob},\infty})  < 0.001 $ whereas $\ell_{\text{std}}^+(w_{\text{rob},\infty}, b_{\text{rob},\infty}) = 0.5$, which means that $
        0.5\approx AD(w_{\text{rob},\infty}, b_{\text{rob},\infty}) \gg AD(w_{\text{std}}, b_{\text{std}}) \approx 0.
$

\subsection{Connection to (Regularized) Linear Regression} \label{subsection connection}
Here we delve deeper into the \textit{fundamental cause} of the norm shrinkage effect as demonstrated in the previous section. Our main findings can be summarized into one sentence: {Eq.~\eqref{Eq.optimal}, which is used to solve for the optimal slopes} $w_{\text{std}}$ {and} ${w_{\text{rob},p}}$, enjoys the same form as the optimal conditions of (regularized) linear regression. Specifically, consider the following optimization problems
\resizebox{.9\linewidth}{!}{
\begin{minipage}{\linewidth}
\begin{align} \label{Eq.optimization}
    \argmin_\beta \frac{1}{2N}\|Y-X\beta\|_2^2, \ \text{and} \  \argmin_\beta \frac{1}{2N}\|Y-X\beta\|_2^2 + \lambda \|\beta\|_q,
\end{align}
\end{minipage}
}

where $X \in \mathbb{R}^{N\times d}$ is the design matrix, $Y \in \mathbb{R}^N$ is the label vector, and $\beta\in\mathbb{R}^d$ is the estimator. The first-order conditions of Eq.~\eqref{Eq.optimization} give us
\begin{align*}
    X^{\top}X\beta = X^{\top}Y \quad \text{and} \quad X^{\top}X\beta = X^{\top}Y - N\lambda\partial\|\beta\|_q.
\end{align*}
Therefore, Eq.~\eqref{Eq.optimal} has the same form as Eq.~\eqref{Eq.optimization} by setting 
\begin{equation*}
  \Sigma= X^{\top}X, \quad \theta^+-\theta^- = X^{\top}Y, \quad \varepsilon = N\lambda/2.  
\end{equation*}
As a consequence, solving for the optimal robust $\ell_p$ classifier is essentially performing linear regression with $\ell_q$-regularization, and this explains the norm shrinkage effect in Proposition \ref{Claim norm}. 

The connection between robust training and regularized linear regression has been noted in a prior work~\citep{xu2008robust} for regression. 
As a comparison, here we demonstrate that such equivalence also holds true for classification problems under Gaussian mixture distributions, which requires a delicate analysis of the KKT conditions in addition to directly exploiting the duality between the data and parameters as in the regression problem.
Furthermore, explicitly formalizing this connection allows us to interpret the trade-off of adversarial robustness and accuracy parity from 
the following perspective---A larger $\varepsilon$ implies better robustness, but also results in a strong regularization (\ie, norm shrinkage) effect as $\varepsilon \propto \lambda$, hence leading to an increased accuracy disparity.

\section{Beyond Gaussian Mixture --- Stable Distributions, and Polynomial Tail}  \label{Section SD}
In this section, we will go beyond the Gaussian mixture distribution, and examine whether the conclusions we have drawn so far hold true for a broader class of data distributions. In particular, we explore a family of distributions that includes the Gaussian distribution as a special case: the \textit{symmetric $\alpha$-stable ($S\alpha S$) distribution}~\citep{levy1954theorie, fama1968some, fama1971parameter}. The motivation for us to study the $S\alpha S$ distribution is twofold. First of all, it is a natural generalization of the Gaussian distribution and preserves an important property of Gaussian: closed under linear transformation since the characteristic function is closed under multiplication. This property then allows us to obtain a precise characterization of the standard/robust loss in terms of the cumulative function. Second, by varying the choice of $\alpha$, we can better understand whether the findings that we have obtained thus far are specific to Gaussian distribution, or hold true in the presence of heavy-tail.

\looseness=-1
The results in this section are mixed: while the conclusion of ``decreased standard accuracy'' generalizes to the $S\alpha S$ distribution, the ``increased accuracy disparity'' phenomenon disappears, and we shall see that the class imbalance ratio will play a fundamentally different role in affecting accuracy disparity when heavy tail is present. To start with, we assume $\varepsilon \le \frac{\kappa}{2}\|\theta^+ - \theta^-\|_\infty$ for some $\kappa<1$ throughout this section, meaning there exists at least one dimension such that the two balls do not intersect.

\subsection{A Brief Review of the \boldmath${S\alpha S}$ distribution}
\label{subsec:pre-sas}  
The probability distribution function of a univariate $S\alpha S$ distribution with \textit{location}, \textit{scale} and \textit{stability} parameters $\mu, c, \alpha$, is defined through %
\begin{align*}
    f(x; \alpha, c, \mu) = \frac{1}{2\pi}\int_\mathbb{R} \varphi(t; \mu, c, \alpha)e^{-ixt}\mathrm{d}t,
\end{align*}
with $\varphi(t; \mu, c, \alpha) = \exp\left(it\mu - |ct|^\alpha\right)$ being its characteristic function. The most important quantity here is the tail-index $\alpha \in (0,2]$ which measures the \textit{concentration} of the corresponding stable distribution, and we recover the Gaussian and Cauchy distribution by setting $\alpha=2$ and $\alpha=1$, respectively. We will mainly focus on multivariate $S\alpha S$ distribution with independent components throughout this section, meaning that each coordinate is independent from the others and follows $f(x; \alpha, c_i, \mu_i)$ with $c_i > 0$, and we denote it as $S\alpha S_{IC}(\mu, C)$ where $C = \text{diag}\{c_i\}$. We refer interested readers to Appendix~\ref{appsub.ec} for discussions of a different multivariate $S\alpha S$ distribution.

\subsection{Adversarial Robustness (Still) Hurts the Standard Accuracy for Balanced Dataset} \label{subsect.positive}
\looseness=-1
We start with the balanced case, \ie, $R=1$, and examine whether enforcing adversarial robustness \textit{provably} hurts the standard accuracy. 
We assume that the data are generated through a mixture of multivariate $S\alpha S$ distributions with independent components and scale parameters $c_i=1$: $\mathcal{P}^+ = S\alpha S_{IC}(\theta^+, \mathbb{I}_d)$ and $\mathcal{P}^- = S\alpha S_{IC}(\theta^-, \mathbb{I}_d)$.
For general choices of $c_i$, we can scale the coordinates of $\theta^+$ and $\theta^-$ inverse-proportionally to obtain the same conclusion. 

We first identify two corner cases in Theorem \ref{Theorem same}, where the optimal robust classifier achieves the same standard accuracy as the optimal standard classifier. The detailed analyses are deferred to Appendix \ref{appsub.ic}.

\begin{theorem} \label{Theorem same}
Under one of the following conditions: 1) $q=\alpha$, meaning that the dual index equals the tail index; 2) $\bar{\theta}:=\theta^+-\theta^-$ is isotropic, \ie, $|\bar{\theta}| \parallel \mathrm{1}_d$ and $\alpha \ge 1$, the optimal robust classifier enjoys the same standard accuracy as the optimal standard classifier when there is no class imbalance.
\end{theorem}

\begin{remark}
It is pointed out in~\emph{\citep{dobriban2020provable}} that enforcing adversarial robustness with $\ell_2$-constraint will not sacrifice the standard accuracy for balanced Gaussian mixture; this corresponds to $q=\alpha=2$ in the first corner case. Similarly, we can also get a win-win for Cauchy mixture with $\ell_\infty$-perturbation, \ie, $q=\alpha=1$.
\end{remark}

Except for the two corner cases above, we show that for general $\alpha$ and $q$, enforcing adversarial robustness will degrade the standard accuracy as stated in the following theorem.  

\begin{theorem} \label{Theorem degrade}
Suppose $\alpha > 1$, $1< q <\infty$, $q \neq \alpha $ and $|\bar{\theta}| \nparallel \mathrm{1}_d$. Then we have $w_{\text{rob},p}^{\top}(\theta^+ - \theta^-) <  w_{\text{std}}^{\top}(\theta^+ - \theta^-)$. In other words, adversarial robustness hurts the accuracy on both classes when there is no class imbalance. 
\end{theorem}

\begin{remark}
We do not discuss other choices of $\alpha$ and $q$ (e.g. $q=1, \infty$ or $\alpha \le 1$) in details as they involve a number of pathological corner cases, mainly due to the non-uniqueness of subgradients. However, we can still expect the conclusion to hold true in general, as the two optimization problems differ by exactly one term, which significantly reduces the possibility of overlapped optimal solution.  
\end{remark}

\subsection{Polynomial Tail Perplexes the Effect of Class Imbalance on Accuracy Disparity} \label{subsec.perplex}
Finally, we will study the effect of class imbalance on accuracy disparity using the multivariate Cauchy distribution (\ie, $\alpha=1$) with independent components and scale parameters $c_i=1$. Specifically, the data are generated through $\mathcal{P}^+ = S1 S_{IC}(\theta^+, \mathbb{I}_d)$ and $\mathcal{P}^- = S1 S_{IC}(\theta^-, \mathbb{I}_d)$,
and we focus on the conventional perturbation scheme $p=\infty$. The proofs are deferred to Appendix \ref{appsub.perplex}. 

Our first result shows that: when heavy-tail is present, the imbalance ratio will result in a \textit{fundamentally different} behavior in terms of the accuracy disparity compared to the Gaussian case, in the sense that both the optimal standard and robust classifier achieve the \textit{same} accuracy disparity for large $R$.
\begin{theorem} \label{Theorem zero-one}
Suppose the class imbalance ratio $R \ge 2+ 4\|\theta^+-\theta^-\|_\infty^2$,
then both the optimal standard classifier as well as the optimal robust $\ell_\infty$ classifier will assign a negative label to all data. As a result, both classifiers will incur zero loss on the majority class and zero accuracy on the minority class. In terms of accuracy disparity, there is no difference between the optimal standard and robust classifier.
\end{theorem}
\begin{remark}
We highlight that the reason for observing such a phenomenon is due to the fact that the Cauchy distribution has a polynomial tail. In contrast, this phenomenon does not exist in distributions with exponentially-decayed tail such as the Gaussian distribution.
\end{remark}
\looseness=-1
Our second result shows that: when the distance between the two means $\|\theta^+ - \theta^-\|_\infty$ is relatively large, adversarial robustness will decrease the accuracy on the majority class while increasing the accuracy on the minority class compared to standard training, hence \textit{reducing} the accuracy disparity.

\begin{theorem} \label{theorem.reduce}
Assume the optimal intercepts $b_{\text{std}}$ and $b_{\text{rob},\infty}$ are finite and 
$\|\theta^+-\theta^-\|_\infty^2 > (R+1)^2  / R(1-\kappa)^2$,
then adversarial robustness will increase the error on the majority class and decrease the error on the minority class, which further reduces the accuracy disparity.
\end{theorem}

\begin{remark}
According to Theorem \ref{Theorem same} (setting $\alpha = q = 1$ in the first corner case), both the optimal standard and robust classifier achieve the same loss on both classes when $R=1$. Therefore, the changes of accuracy disparity as shown in Theorem \ref{Theorem zero-one} and Theorem \ref{theorem.reduce} are mainly due to the class imbalance part. Contrary to the Gaussian mixture distribution where this factor \textbf{consistently} enlarges the accuracy disparity, here we obtain different observations where it stays the same or even decreases, suggesting that the accuracy disparity not only concerns the class imbalance ratio, but is also heavily influenced by the tail property of the corresponding distribution.
\end{remark}

\renewcommand{\thesubfigure}{\alph{subfigure}}

\begin{figure*}

\newlength{\utilheight}
\settoheight{\utilheight}{\includegraphics[width=.25\linewidth]{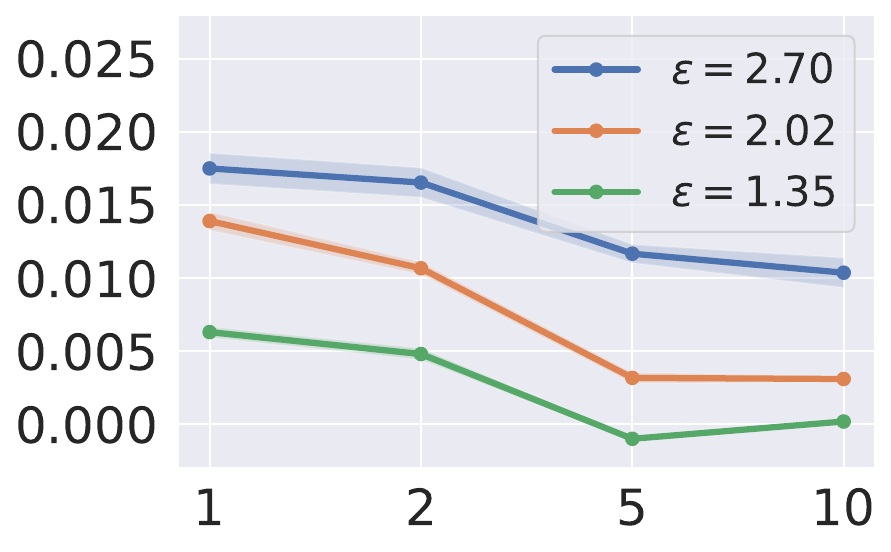}}%

\newlength{\legendheight}
\setlength{\legendheight}{0.4\utilheight}%

\newcommand{\rowname}[1]%
{\rotatebox{90}{\makebox[\utilheight][c]{\tiny #1}}}

\centering

{
\renewcommand{\tabcolsep}{10pt}

\begin{subtable}[]{\linewidth}
\centering
\resizebox{\linewidth}{!}{%
\begin{tabular}{@{}p{3.5mm}@{}c@{}c@{}c@{}c@{}c@{}c@{}c@{}c@{}}
        & \makecell{\small{\textbf{Synthetic, $p=\infty$}}}
        & \makecell{\small{\textbf{Synthetic, $p=2$}}}
        & \makecell{\small{\textbf{MNIST, $p=\infty$}}}
        & \makecell{\small{\textbf{MNIST, $p=2$}}}
        & \makecell{\small{\textbf{CIFAR, $p=\infty$}}}
        & \makecell{\small{\textbf{CIFAR, $p=2$}}}
        \\
\rowname{\makecell{\scriptsize $AD^{R}_{\trob,p,\eps} - AD^{R}_{\tstd}$}}&
\includegraphics[height=\utilheight]{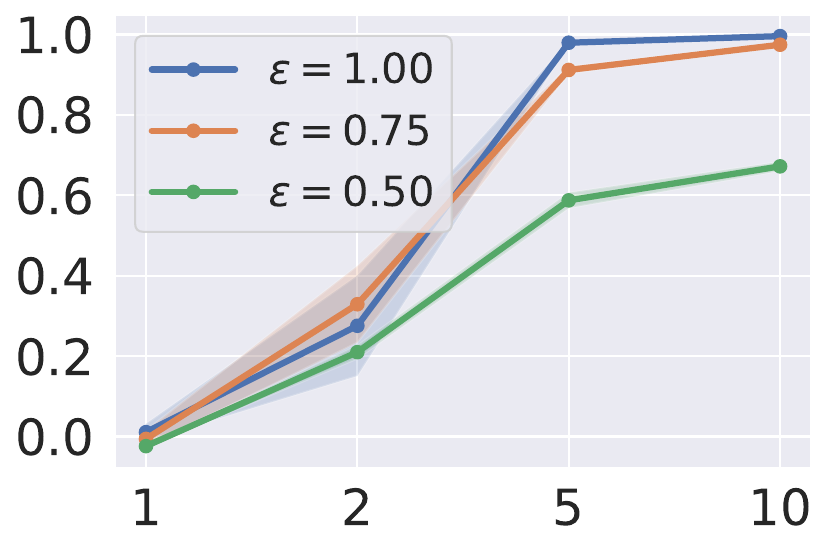}&
\includegraphics[height=\utilheight]{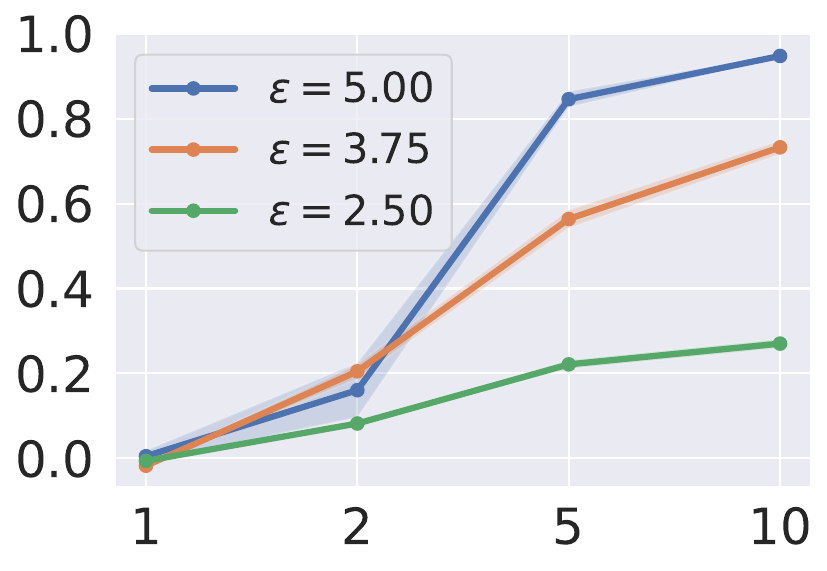}&
\includegraphics[height=\utilheight]{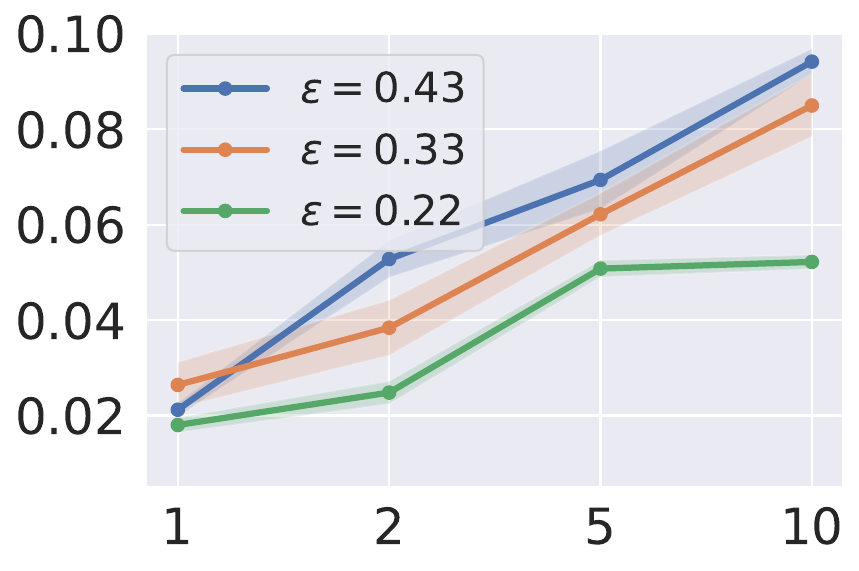}&
\includegraphics[height=\utilheight]{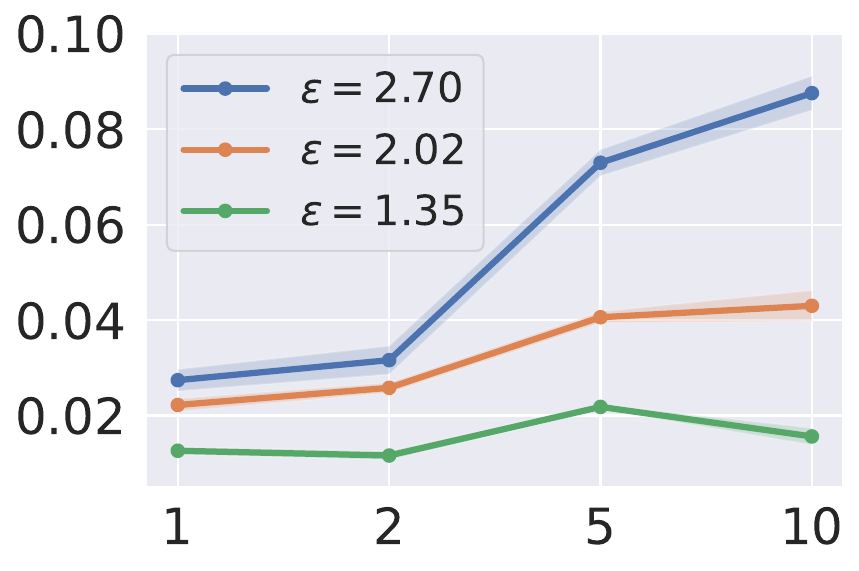}&
\includegraphics[height=\utilheight]{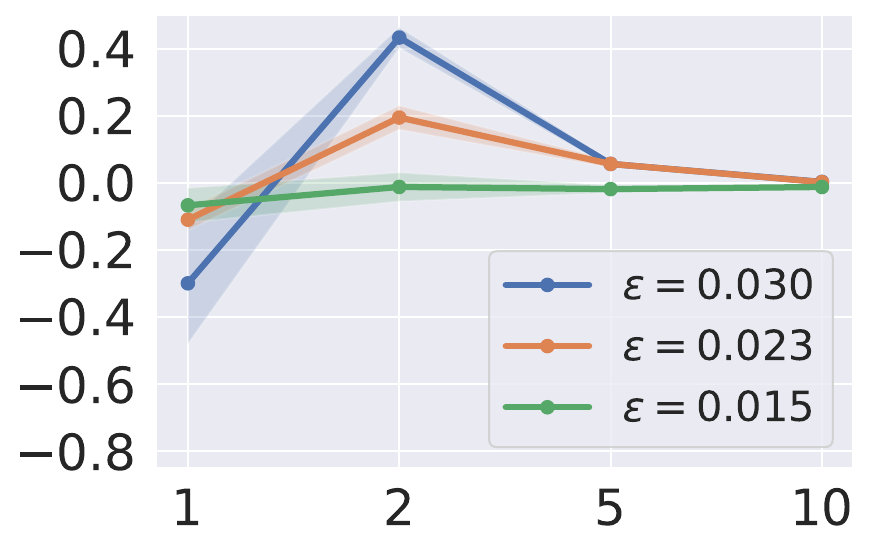}&
\includegraphics[height=\utilheight]{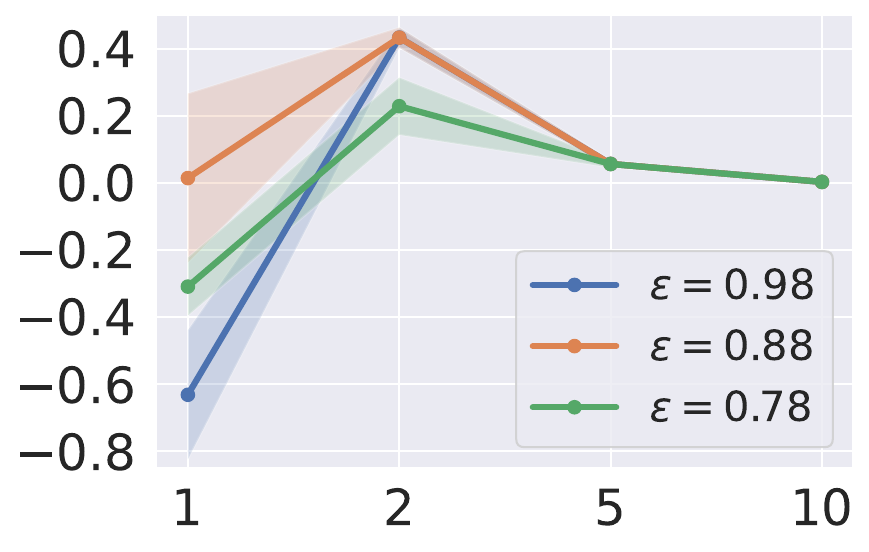}&
\\
\rowname{\makecell{\scriptsize $acc^{R,\cdot}_{\tstd} - acc^{R,\cdot}_{\trob,p,\eps}$}}&
\includegraphics[height=\utilheight]{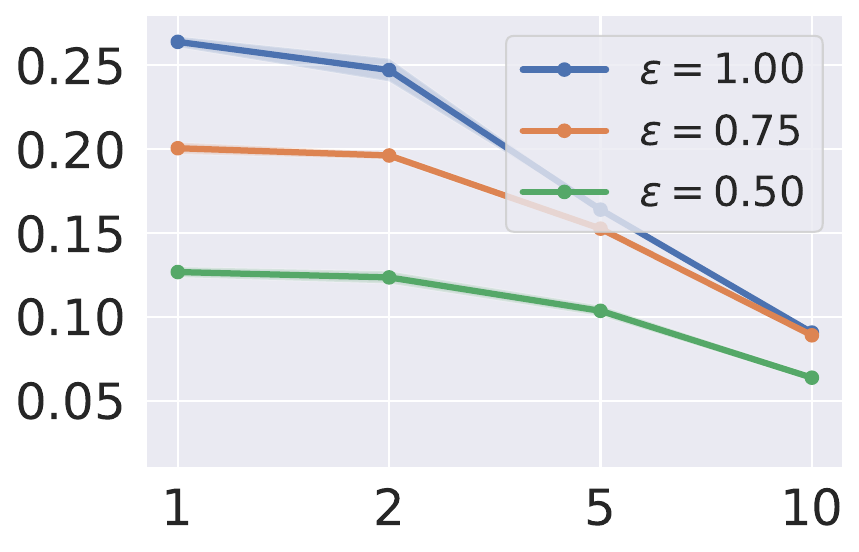}&
\includegraphics[height=\utilheight]{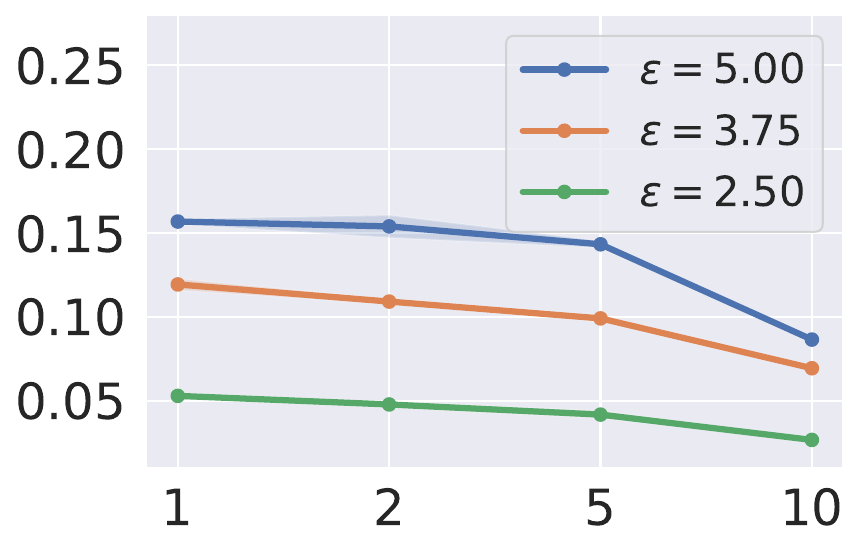}&
\includegraphics[height=\utilheight]{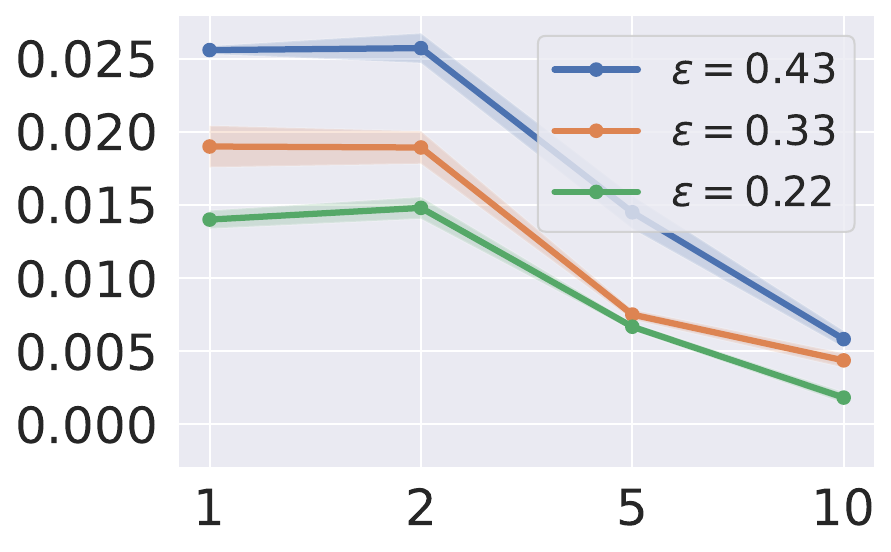}&
\includegraphics[height=\utilheight]{figs/mnist_l2_acc.pdf}&
\includegraphics[height=\utilheight]{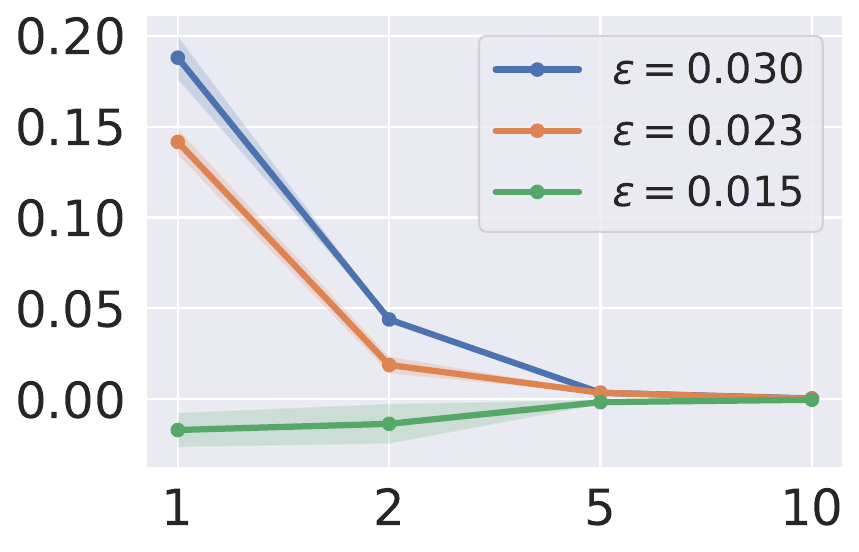}&
\includegraphics[height=\utilheight]{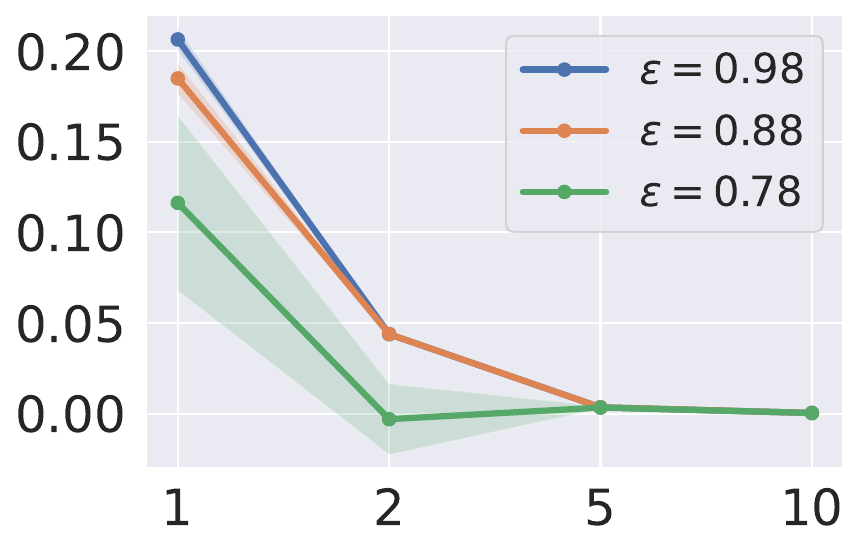}&
\\[-2mm]
        & \makecell{\scriptsize{$R$}}
        & \makecell{\scriptsize{$R$}}
        & \makecell{\scriptsize{$R$}}
        & \makecell{\scriptsize{$R$}}
        & \makecell{\scriptsize{$R$}}
        & \makecell{\scriptsize{$R$}}
\end{tabular}
}
\end{subtable}
}
\caption{\small The gap of accuracy disparity $AD^R_{\trob,p,\eps} - AD^R_{\tstd}$ (\nth{1} row, \rqtwo) and standard accuracy $acc^{R,\cdot}_{\tstd} - acc^{R,\cdot}_{\trob,p,\eps}$ (\nth{2} row, \rqthree) between robust and standard classifiers, w.r.t. the imbalance ratio $R$.
Different columns correspond to different datasets or $\ell_p$ norms of adversarial training. 
For each $\ell_p$ norm, we experiment with multiple perturbation scales $\eps$.
The shaded area in each subfigure represents the standard error of $5$ runs.
}%
\label{fig:main}
\end{figure*}

\section{Experiments}
\label{sec:exp}
We corroborate and strengthen our theoretical results regarding \textit{accuracy disparity} and \textit{standard accuracy} via experiments on one synthetic dataset and two real-world datasets. In what follows, we first introduce the experiment setup and then lay out the research questions we shall investigate.

\textbf{Adversarial Training.}\quad
Our theoretical findings are algorithm agnostic and only concern the definition of adversarial robustness. In the experiments, we choose a popular algorithm, adversarial training~\citep{kurakin2016adversarial, madry2017towards}, to perform the robust training.

\textbf{Metrics.}\quad
We use $acc^{R,+}_{\tstd},acc^{R,-}_{\tstd},acc^{R,\cdot}_{\text{std}}$ to denote the \textit{standard accuracy} of the standard classifier trained on the dataset with imbalance ratio $R$, measured on the minority class, and the majority class, and both classes, respectively.
Likewise, we use $acc^{R,+}_{\trob,p,\eps},acc^{R,-}_{\trob,p,\eps},acc^{R,\cdot}_{\trob,p,\eps}$ to denote the standard accuracy of the robust classifier trained with $\ell_p$ perturbations of scale $\eps$, calculated on the three types of populations.
We then use $AD^{R}_{\text{std}}, AD^{R}_{\trob,p,\eps}$ to denote the \textit{accuracy disparity} of the standard classifiers and robust classifiers, formally defined as 
$AD^{R}_{\text{std}}  = acc^{R,-}_{\text{std}}-acc^{R,+}_{\text{std}}$ and $ AD^{R}_{\trob,p,\eps} = acc^{R,-}_{\text{rob},p,\eps}-acc^{R,+}_{\text{rob},p,\eps}$.

\textbf{Research Questions.}\quad
We lay out the research questions (RQs) based on Section \ref{Section AT} and \ref{Section SD}.

\begin{quote}
\rqone. 
Will adversarial training exacerbate accuracy disparity compared with standard training, \ie, $AD^{R}_{\trob,p,\eps} > AD^{R}_{\tstd}$, and when?
\medskip

\rqtwo. Will a more severe class imbalance (\ie, a larger imbalance ratio $R$) lead to a more significant accuracy disparity gap (\ie, a larger $ (AD^{R}_{\trob,p,\eps} - AD^{R}_{\tstd})$), and when?
\medskip

\rqthree. 
Will adversarial training worsen the standard accuracy, \ie, $acc^{R,\cdot}_{\trob, p, \eps} < acc^{R,\cdot}_{\tstd}$, and when? 
\end{quote}

The basis of \rqone and \rqtwo from the theoretical side are~\Cref{cor:imbalance,Theorem zero-one};
the basis of \rqthree are~\Cref{corollary.degrade,Theorem degrade}. 
We next investigate these questions from the empirical side to gain insights into how well and how far the theoretical results can be supported and extended.

\textbf{Experimental Setup.}\quad
We evaluate the above three questions using three datasets: 
a synthetic dataset of a mixture of Gaussians, as well as two real-world datasets MNIST~\citep{lecun1998gradient} and CIFAR-10~\citep{krizhevsky2009learning}.
Results on additional datasets, including two synthetic datasets featuring stable distributions (Cauchy and Holtsmark), as well as two more real-world datasets, Fashion-MNIST~\citep{xiao2017fashion} and ImageNet~\citep{deng2009imagenet}, are provided in~\Cref{adxsubsec:exp-stable,adxsubsec:exp-fmnist}. 
For each dataset, we investigate both the balanced case ($R=1$) and the imbalanced cases $R \in \{2,5,10\}$.
For the synthetic dataset and MNIST, we use a linear classifier; for CIFAR, we use a neural network with two linear layers.
When performing adversarial training, we use the fast gradient method~(FGM)~\citep{goodfellow2014explaining} and projected gradient descent~(PGD)~\citep{madry2017towards} to craft the adversarial examples.
We experiment with both $p=2,\infty$, each with multiple perturbation scales $\eps$.
For each set of experiments, we report results averaged over 5 runs with different random seeds to account for variability.
More details are deferred to Appendix \ref{adxsubsec:exp-setup}.

\subsection{Analysis of the Increased Accuracy Disparity \textmd{(\rqone and \rqtwo)}}
\label{subsec:exp-analysis-1}

By comparing the accuracy disparity of standard classifier and a variety of adversarial classifiers, we offer the following answers to \rqone and \rqtwo.

\begin{quote}
\ansone: Yes, when $R>1$.

\medskip

\anstwo: The increase of the accuracy disparity gap with the imbalance ratio consistently happens for the synthetic dataset and MNIST, but not for CIFAR.

\end{quote}

We now present the concrete experimental results along with detailed discussions.
We plot the accuracy disparity gap in the \nth{1} row of Figure \ref{fig:main}; 
the raw numbers can be referred to in Appendix \ref{adxsubsec:exp-result}
We draw the following conclusions.

Regarding \rqone, the accuracy disparity gap is invariably larger than $0$ in the class imbalance setting (\ie, $R>1$) on all three datasets. 
This provides an affirmative answer for \rqone and matches our theoretical result in Theorem \ref{cor:imbalance}.
Actually, in the balanced case, the gap is also close to $0$ in most of the cases, apart from an intriguingly low number for CIFAR.
We figure that this is associated with the relative difficulty of learning the two classes---in the synthetic dataset and MNIST, the difficulty levels of learning the two classes are alike; for CIFAR, the class `cat' is much harder to learn than the class `dog', as already demonstrated in previous work~\citep{croce2021robustbench}.

Regarding \rqtwo, the accuracy disparity gap grows with the class imbalance ratio on both the synthetic dataset and MNIST. 
However, the same phenomenon does not occur for CIFAR, where the gap drops to almost $0$ for larger $R$ (\ie, $R=5,10$).
We note that the result on the synthetic dataset and MNIST aligns well with Theorem \ref{cor:imbalance}; 
this serves as a verification for the theoretical result on exact Gaussian mixtures and an implication on the potential of extending the result to real-world datasets that can be roughly modeled as Gaussian mixtures.
{In comparison, what we observe on CIFAR resembles the theoretical analysis on the heavy-tail distribution (Theorem \ref{Theorem zero-one}), where the standard classifier can only achieve accuracy close to $0$ on the minority class, accounting for the accuracy disparity gap of nearly $0$. We conjecture that the \textit{distributional difference} between MNIST and CIFAR contributes to the distinction in the experimental results. To strengthen our hypothesis, we first perform two additional experiments in Appendix \ref{adxsubsec:exp-analysis}, ruling out alternative explanation regarding the insufficiency of training on CIFAR (caused either by the small dataset size or the limited model capacity, in a relative sense compared to MNIST). 
We additionally provide an empirical comparison of the statistical properties between MNIST and CIFAR in Appendix \ref{adxsubsec:exp-analysis}, demonstrating that the empirical distribution of CIFAR is indeed more scattered than MNIST. }

\subsection{Analysis of the Decreased Standard Accuracy \textmd{(\rqthree)}}
\label{subsec:exp-analysis-2}

\looseness=-1
In this part, 
we shift our focus from the accuracy \textit{disparity} (defined as the gap of the per class standard accuracy) to the \textit{overall} standard accuracy (measured on both classes). In a nutshell:

\begin{quote}
\ansthree: Yes, adversarial training almost always hurts the standard accuracy in the scenarios we experiment with.
\end{quote}

\looseness=-1
We plot the gap of the standard accuracy between the standard classifier and the robust classifier in the \nth{2} row of Figure \ref{fig:main} (raw numbers in Appendix \ref{adxsubsec:exp-result}).
First, we see that in the class imbalance setting (\ie, $R=1$), the gap is almost always larger than $0$, consistent with the theoretical results in \Cref{corollary.degrade,Theorem degrade}.
Furthermore, in the class imbalance setting which is beyond the scope of our theoretical results, we observe an interesting decrease of the accuracy gap with the increase of the imbalance ratio $R$.
We offer the following explanation.
The impact of class imbalance gradually takes the dominance (in the sense of encouraging the prediction to favor the majority); in this case, whether performing standard or adversarial training will not have much influence on the outcome.

\section{Related Work}
\label{sec:morerw}
We study the interactions of fairness, adversarial robustness, and accuracy within machine learning models. As such, we will mainly review the works that focus on the relation between these angles.

\looseness=-1
\paragraph{Robustness and Fairness.}
What initially motivates this work is the observation of a trade-off between adversarial robustness and fairness~\citep{liu2021trustworthy}. Here, fairness refers to the \textit{class-wise} performance of the robust classifier (a broader definition would be the performance across subgroups defined by sensitive attributes~\citep{hardt2016equality, zafar2017fairness}), and is measured by either the robust accuracy or the standard accuracy. The former corresponds to a phenomenon known as ``robustness disparity''~\citep{nanda2021fairness}, %
and is verified on a wide range of datasets, model architectures, as well as attacks and defenses~\citep{nanda2021fairness, tian2021analysis}. The latter concerns the ``accuracy disparity''~\citep{chi2021understanding} of the robust classifier, and it is empirically shown that not only does such accuracy disparity exists~\citep{croce2021robustbench, benz2021triangular}, but is further exacerbated compared to the standard classifier~\citep{benz2021robustness}. As a complement to these empirical observations, we aim to provide an in-depth theoretic study towards {understanding} the impact of adversarial robustness on accuracy disparity. 

\textit{Detailed comparisons with two closest works.}\quad
\citet{xu2021robust,ma2022on} identify and analyze the significant disparity of standard accuracy and robust accuracy among different classes or subgroups of data for adversarially trained models.
Our work differs from theirs in several key aspects.
\textit{First}, their theoretical analysis is restricted to specific choices of parameters, resulting in an oversimplified problem and less convincing conclusions. 
In contrast, our approach accommodates arbitrary means, covariance matrices, and perturbation types.
Additionally, differences exist in the settings and targets of study.
\citet{xu2021robust} focus on a balanced class setting but with varying ``difficulty levels'' as measured by the magnitude of variance, whereas we address class imbalance.
\citet{ma2022on} measure fairness through the variance of robust risk, while we measure fairness by class-wise accuracy disparity (corresponding to the standard risk).
Importantly, we have identified \textit{critical flaws} in the proof of Theorem 5.7 in~\citet{ma2022on}.
On page 18 of their publication\footnote{\url{https://openreview.net/attachment?id=LqGA2JMLwBw\&name=supplementary\_material}}, 
a non-trivial gap exists between their Equations (51) and (52) even given their unnatural assumptions on $\epsilon_{train}$ and $\epsilon_{test}$, which the authors made no attempt to address. 
Furthermore, the last inequality in Equation (52) is not correct, as the term is plainly smaller than $1$ for small $\epsilon_{train}$ and $\epsilon_{test}$. 
These critical issues undermine the validity of their findings.

\paragraph{Robustness and (Standard) Accuracy.} Ever since the seminal work~\citep{tsipras2018robustness}, there has been a line of research studying the fundamental trade-off between
adversarial robustness and standard accuracy. This includes some empirical evaluations~\citep{raghunathan2019adversarial, su2018robustness}, theoretical results regarding the statistical/information limit or sample complexity for robust classification~\citep{bhagoji2019lower, chen2020more, dan2020sharp}, as well as algorithms that explicitly exploit such trade-off based on theoretical insights~\citep{zhang2019theoretically, raghunathan20a, zhang2020attacks, yang2020closer}. Despite the fruitful results that have been achieved in this field thus far, a rigorous understanding towards how enforcing adversarial robustness {decreases} the standard accuracy is still lacking. Our work does not target this problem directly, but as a by-product,
we show that for balanced dataset and stable distributions, robustness \textit{in general} comes at a cost of degraded performance for standard accuracy due to a change of ``direction'' (see Subsection \ref{subsec.main}), thus offering a new perspective towards interpreting this intriguing phenomenon.

\section{Conclusion, Limitation and Future Directions} \label{sec:conclusion}
In this work, we provide an in-depth and fine-grained study towards understanding the impact of adversarial robustness on accuracy disparity when class imbalance is present. To this end, we offer a complete characterization regarding the classification of a Gaussian mixture with linear models, and decompose the overall effect of enforcing adversarial robustness into two disjoint parts: an inherent one that will degrade the standard accuracy due to a change of ``direction'', and the other caused additionally by the class imbalance ratio that will increase the accuracy disparity due to a change of ``norm''. We proceed to analyze the general stable distribution. While the intrinsic effect of robustness can generalize and consistently decrease the standard accuracy even for the balanced class setting, we uncover that the imbalance ratio plays a fundamentally different role in the accuracy disparity due to the heavy tail of the stable distribution. Finally, we support and strengthen our theoretical results with experiments on both synthetic and real-world datasets. 

\noindent \textbf{Limitation.} \quad An obvious limitation of the paper is that the analyses are restricted to binary classification. Generalization to multi-class classification requires modifying the decision rule; for instance, using argmax of the logits. However, this will lead to a Voronoi diagram partition of the space for the $k>2$ classes, which is challenging to precisely and analytically characterize (\eg, it is no longer easy to compute the probability mass on each convex body within the diagram and in general this partition does not have analytical characterization). 

\looseness=-1
\noindent \textbf{Future directions.} \quad As real-world datasets contain both class imbalance and discrepancy of class-wise distributions, a more complete theory should consider the usage of different covariance matrices and analyze its interaction with the class imbalance ratio. Some additional future directions include 1) introduce the protected attribute $A$ under each label $Y$, which will make the results more appealing to the fairness community; 2) allow for different test and training distributions, and check whether robustness provably helps in the presence of distribution/subpopulation shifts.      
On the empirical side, our theoretical insights could lead to the design of future robust training algorithms that aim to achieve a certain notion of accuracy parity among classes. 
Furthermore, the distributions of real-world datasets in the feature space might be closer to a GMM, and hence one could enforce the robustness constraint on the feature distributions.
Last but not least, MNIST and CIFAR exhibit significantly different conclusions on \rqtwo; 
the nice correspondences between them and the findings in Gaussian and stable distributions could motivate a deeper understanding towards the distributional characteristics of real-world datasets.

\newpage

\bibliography{references}
\bibliographystyle{icml2023}

\newpage

\appendix

\onecolumn

\section{Omitted Proofs from Section \ref{Section AT}} \label{Appendix AT}

\subsection{Proof of the Main Results} \label{appsub.main}
\begin{proof}[Proof of Theorem \ref{thm:general}]
For the standard loss, we have
\begin{align*}
    \ell_{\text{std}}(w,b) = \frac{R}{R+1}\Phi\left(\frac{b+w^{\top}\theta^-}{\sqrt{w^{\top}\Sigma w}}\right) + \frac{1}{R+1} \Phi\left(\frac{-b-w^{\top}\theta^+}{\sqrt{w^{\top}\Sigma w}}\right).
\end{align*}
Since $\ell_{\text{std}}$ is scale-invariant, we can assume w.l.o.g. that $w^{\top}\Sigma w = 1$. We then obtain an equivalent constrained optimization problem 
\begin{align*}
    \min_{w,b}&\quad R\Phi(b+w^{\top}\theta^-) + \Phi(-b - w^{\top}\theta^+) \\
\text{s.t.} &\quad w^{\top}\Sigma w = 1.
\end{align*}
The Lagrangian can be written as 
\begin{align*}
    \mathcal{L}_{\text{std}}(w, b, \nu) = R\Phi(b+w^{\top}\theta^-) + \Phi(-b - w^{\top}\theta^+) + \nu(w^{\top}\Sigma w - 1),
\end{align*}
and the KKT conditions give us
\begin{align} \label{Eq w}
    \frac{R}{\sqrt{2\pi}}\exp\left(-\frac{(b+w^{\top}\theta^-)^2}{2}\right)\theta^- - \frac{1}{\sqrt{2\pi}} \exp\left(-\frac{(b+w^{\top}\theta^+)^2}{2}\right)\theta^+ + 2\nu \Sigma w = 0
\end{align}
and
\begin{align} \label{Eq b}
    \frac{R}{\sqrt{2\pi}}\exp\left(-\frac{(b+w^{\top}\theta^-)^2}{2}\right) - \frac{1}{\sqrt{2\pi}} \exp\left(-\frac{(b+w^{\top}\theta^+)^2}{2}\right) = 0.
\end{align}
Plugging Eq. (\ref{Eq b}) into Eq. (\ref{Eq w}), we can conclude that
\begin{align*}
    \Sigma w_{\text{std}} = (\theta^+ - \theta^-) \cdot C_{\text{std}}
\end{align*}
for some positive constant $C_{\text{std}}$, and $w_{\text{std}}$ additionally satisfies
\begin{align*}
    w_{\text{std}}^{\top}\Sigma w_{\text{std}} = 1
\end{align*}
Therefore, following the statement of Theorem \ref{thm:general}, suppose $\Sigma u = \theta^+ - \theta^-$, then we can pick 
\begin{align*}
    w_{\text{std}} = \frac{u}{r},
\end{align*}
where $u^{\top}\Sigma u = r^2$. After determining $w_{\text{std}}$, $b_{\text{std}}$ can be solved directly from Eq. (\ref{Eq b}), which gives
\begin{align*}
    b_{\text{std}}=-\frac{2\log R + (w_{\text{std}}^{\top}\theta^+)^2 - (w_{\text{std}}^{\top}\theta^-)^2}{2(w_{\text{std}}^{\top}\theta^+ - w_{\text{std}}^{\top}\theta^-)}.
\end{align*}
Note
\begin{align*}
    w_{\text{std}}^{\top}\theta^+ - w_{\text{std}}^{\top}\theta^- = \bigg< \frac{u}{r}, \Sigma u \bigg> = r.
\end{align*}
Therefore,
\begin{align*}
    \ell_{\text{std}}^+(w_{\text{std}}, b_{\text{std}}) &= \Phi\left(-b_{\text{std}}-w_{\text{std}}^{\top}\theta^+\right)\\
    &= \Phi\left(\frac{-r^2 + 2\log R}{2r}\right) \\
    &= \Phi\left(\frac{-\langle u, \theta^+ -\theta^- \rangle + 2\log R}{2r}\right), 
\end{align*}
whereas
\begin{align*}
    \ell_{\text{std}}^-(w_{\text{std}}, b_{\text{std}}) &= \Phi\left(b_{\text{std}}+w_{\text{std}}^{\top}\theta^-\right) \\
    &= \Phi\left(\frac{-r^2 - 2\log R}{2r}\right) \\
    &= \Phi\left(\frac{-\langle u, \theta^+ -\theta^- \rangle - 2\log R}{2r}\right).
\end{align*}
Similarly, for the robust $\ell_p$ loss  
\begin{align*}
    \ell_{\text{rob}}(w,b) = \frac{R}{R+1}\Phi\left(\frac{b+w^{\top}\theta^- +  \varepsilon \|w\|_q}{\sqrt{w^{\top}\Sigma w}}\right) + \frac{1}{R+1} \Phi\left(\frac{-b-w^{\top}\theta^+ +  \varepsilon \|w\|_q}{\sqrt{w^{\top}\Sigma w}}\right),
\end{align*}
we can assume w.l.o.g. that $w^{\top}\Sigma w = 1$ and write down the Lagrangian 
\begin{align*}
    \mathcal{L}_{\text{rob}}(w,b,\lambda) = R\Phi(b+w^{\top}\theta^-+  \varepsilon \|w\|_q) + \Phi(-b - w^{\top}\theta^++  \varepsilon \|w\|_q) + \mu(w^{\top}\Sigma w - 1),
\end{align*}
and the KKT conditions give us 
\begin{align*}
    \Sigma w_{\text{rob}, p} = (\theta^+-\theta^- - 2\varepsilon\partial{\|w_{\text{rob}, p} \|_q}) \cdot C_{\text{rob},p}.
\end{align*}
for some positive constant $C_{\text{rob},p}$. A crucial observation here is that the subdifferential set
$\partial{\|w_{\text{rob}, p} \|_q}$ is invariant when scaled by a positive constant (guaranteed by Danskin's theorem), so if we follow the statement of Theorem \ref{thm:general} and suppose $\Sigma v = \theta^+ - \theta^- - 2\varepsilon\partial{\|v\|_q}$, then to satisfy 
the constraint $w_{\text{rob}, p}^{\top}\Sigma w_{\text{rob}, p} = 1$, we can pick
\begin{align*}
    w_{\text{rob}, p} = \frac{v}{s},
\end{align*}
where $v^{\top}\Sigma v = s^2$. Similarly, $b_{\text{rob}, p}$ can be derived from the KKT conditions as
\begin{align*}
    b_{\text{rob}, p} = -\frac{2\log R + (w_{\text{rob}, p}^{\top}\theta^+)^2 - (w_{\text{rob}, p}^{\top}\theta^-)^2 - 2\epsilon\|w_{\text{rob}, p}\|_q(w_{\text{rob}, p}^{\top}\theta^+ + w_{\text{rob}, p}^{\top}\theta^-)}{2(w_{\text{rob}, p}^{\top}\theta^+ - w_{\text{rob}, p}^{\top}\theta^- - 2\epsilon\|w_{\text{rob}, p}\|_q)}.
\end{align*}
Note
\begin{align*}
    \langle w_{\text{rob}, p}, \theta^+-\theta^-\rangle &= \bigg<\frac{v}{s}, \Sigma v + 2\epsilon\partial{\|v\|_q}\bigg> \\
&=s +2\epsilon\frac{\langle v,\partial{\|v\|_q} \rangle}{s} \\
&=s + 2\epsilon \frac{\|v\|_q}{s} \\
&= s+2\epsilon\|w_{\text{rob}, p}\|_q,
\end{align*}
where we use Danskin's theorem again in the second-to-last inequality. Therefore,
\begin{align*}
    \ell_{\text{std}}^+(w_{\text{rob}, p}, b_{\text{rob}, p}) = \Phi\left(-b_{\text{rob}, p}-w_{\text{rob}, p}^{\top}\theta^+\right) =  \Phi\left(\frac{-\langle v, \theta^+ -\theta^- \rangle + 2\log R}{2s}\right), 
\end{align*}
whereas
\begin{align*}
    \ell_{\text{std}}^-(w_{\text{rob}, p}, b_{\text{rob}, p}) = \Phi\left(b_{\text{rob}, p}+w_{\text{rob}, p}^{\top}\theta^-\right) =  \Phi\left(\frac{-\langle v, \theta^+ -\theta^- \rangle - 2\log R}{2s}\right).
\end{align*}
This finishes the proof as desired.
\end{proof}

\begin{proof}[Proof of Proposition \ref{Claim.direction}]
Plugging in the equation $\Sigma u = \theta^+-\theta^-$ as well as the definitions of $r$ and $s$, it suffices to show
\begin{align*}
    \sqrt{u^{\top}\Sigma u}\sqrt{v^{\top}\Sigma v} \ge v^{\top}\Sigma u.
\end{align*}
Since $\Sigma$ is positive semi-definite, $\Sigma^{\frac{1}{2}}$ is well-defined. Now denote $u' = \Sigma^{\frac{1}{2}} u$ and $v' = \Sigma^{\frac{1}{2}} v$, then the above inequality is equivalent to 
\begin{align*}
    \|u'\|_2 \|v'\|_2 \ge v'^{\top} u',
\end{align*}
which holds due to the Cauchy-Schwarz inequality. Furthermore, the inequality holds strictly as long as $u'$ and $v'$ are not parallel. Combining the fact that $u$ is parallel to $w_{\text{std}}$, and $v$ is parallel to $w_{\text{rob},p}$ finishes the proof as desired.

\end{proof}

\begin{proof}[Proof of Theorem \ref{corollary.degrade}]
It follows directly from Proposition \ref{Claim.direction} since $\Phi$ is monotonic.
\end{proof}

\begin{proof}[Proof of Proposition \ref{Claim norm}]
We have
\begin{align*}
    u^{\top}\Sigma u - v^{\top}\Sigma v &= (u-v)^{\top}\Sigma(u-v) + 2(u-v)^{\top}\Sigma v\\
&\ge 2v^{\top}\Sigma (u-v) \\
&= 2\langle v, 2\varepsilon\partial{\|v\|_q}\rangle\\
&=4\varepsilon\|v\|_q >0,
\end{align*}
where we take the difference between 
\begin{align*}
    \Sigma u = \theta^+-\theta^-  \quad \text{and}\quad \Sigma v = \theta^+-\theta^-- 2\varepsilon\partial{\|v\|_q}
\end{align*}
in the second-to-last equality, and use Danskin's theorem \citep{bertsekas1997nonlinear} in the last equality.
\end{proof}

\begin{proof}[Proof of Theorem \ref{cor:imbalance}]
Denote 
\begin{align*}
    p = \frac{-\langle u, \theta^+ -\theta^- \rangle}{2r}, \quad q = \frac{-\langle v, \theta^+ -\theta^- \rangle}{2s}.
\end{align*}
Note $\langle u, \theta^+ -\theta^- \rangle = u^{\top}\Sigma u \ge 0$, $\langle v, \theta^+ -\theta^- \rangle = v^{\top}\Sigma v + 2\varepsilon \|v\|_q > 0$, so Proposition \ref{Claim.direction} implies $p \le q <0$. Further denote
\begin{align*}
     m = \frac{\log R}{r} > 0, \quad k = \frac{r}{s} > 1,
\end{align*}
then $\frac{\log R}{s} = km$. We first show the {``increased accuracy disparity''} part in the theorem, which is equivalent to
\begin{align*}
    \Phi(q+km) - \Phi(q-km) > \Phi(p+m) - \Phi(p-m).
\end{align*}
In fact, we have
\begin{align*}
    \Phi(q+km) - \Phi(q-km) > \Phi(q+m) - \Phi(q-m),
\end{align*}
so it suffices to show
\begin{align*}
    \Phi(q+m) - \Phi(q-m) \ge \Phi(p+m) - \Phi(p-m).
\end{align*}
Define $F(x) = \Phi(x+m) - \Phi(x-m)$, and we will show $F(x)$ is increasing on $(-\infty, 0]$. In fact, we have
\begin{align*}
    F'(x) &= \frac{1}{\sqrt{2\pi}}\left(\exp\left(-\frac{(x+m)^2}{2}\right) - \exp\left(-\frac{(x-m)^2}{2}\right)\right) \\
    &= \frac{1}{\sqrt{2\pi}}\exp\left(-\frac{(x-m)^2}{2}\right)\left(e^{-2xm}-1\right) \ge 0
\end{align*}
since $m \ge 0$ and $x <0$. 

To show that the gap between the two accuracy disparities is {monotonic}, it suffices to prove
\begin{align*}
    g(R) = G(m) := \Phi(q+km) - \Phi(q-km) - \Phi(p+m) + \Phi(p-m)
\end{align*}
is increasing w.r.t. $m$ when $q+ km <0$. In fact, we have
\begin{align*}
    \sqrt{2\pi}G'(m) &= k\left(\exp\left(-\frac{(q+km)^2}{2}\right) + \exp\left(-\frac{(q-km)^2}{2}\right)\right)\\ 
    &- \left(\exp\left(-\frac{(p+m)^2}{2}\right) + \exp\left(-\frac{(p-m)^2}{2}\right)\right).
\end{align*}
Define 
\begin{align*}
    L(k) := k\left(\exp\left(-\frac{(q+km)^2}{2}\right) + \exp\left(-\frac{(q-km)^2}{2}\right)\right),
\end{align*}
then it suffices to show that $L'(k) > 0$ when $q+km<0$. In fact, we have
\begin{align*}
    L'(k) &= \exp\left(-\frac{(q+km)^2}{2}\right) + \exp\left(-\frac{(q-km)^2}{2}\right) \\
    &-km(q+km)\exp\left(-\frac{(q+km)^2}{2}\right) - km(km-q)\exp\left(-\frac{(q-km)^2}{2}\right) \\
    &=\exp\left(-\frac{(q-km)^2}{2}\right)\left[e^{-2qkm}(1-km(q+km)) + (1-km(km-q))\right].
\end{align*}
Let $km=a$ and $-q=b$, it suffices to show 
\begin{align*}
    e^{2ab} > \frac{ab+(a^2-1)}{ab-(a^2-1)}
\end{align*}
when $b>a >0$. In fact, we have
\begin{align*}
    \frac{ab+(a^2-1)}{ab-(a^2-1)} \le \frac{ab+(ab-1)}{ab-(ab-1)} = 2ab-1 < 2ab+1 \le e^{2ab}.
\end{align*}

\end{proof}

\subsection{Derivation of the Optimal Classifiers for the Toy Example} \label{appsub.toy}

Given $X \sim \mathcal{N}(\theta, \Sigma)$ and some $w \in \mathbb{R}^d$, we have 
$
    w^{\top}X \sim \mathcal{N}(w^{\top}\theta, w^{\top}\Sigma w)
$
due to the fact that Gaussian distribution is closed under linear transformation.

\paragraph{The optimal standard classifier.} Using the property above, we have
\begin{align*}
    \ell_{\text{std}}(w,b) 
=&\frac{R}{R+1}\Phi\left(\frac{b-\eta\sum_{i=1}^m w_i - \gamma \sum_{j=m+1}^{m+n}w_j}{\sqrt{\sum_{k=1}^{m+n} w_k^2}}\right)\\
+ &\frac{1}{R+1}\Phi\left(\frac{-b-\eta\sum_{i=1}^m w_i - \gamma \sum_{j=m+1}^{m+n}w_j}{\sqrt{\sum_{k=1}^{m+n} w_k^2}}\right),
\end{align*}
We can assume w.l.o.g. that $\sum_{k=1}^{m+n} w_k^2 = 1$ as $\ell_{\text{std}}$ is scale-invariant. Hence, by Cauchy-Schwarz,
\begin{align*}
    \eta\sum_{i=1}^m w_i + \gamma \sum_{j=m+1}^{m+n}w_j \le \sqrt{\sum_{k=1}^{m+n} w_k^2}\sqrt{m\eta^2 + n\gamma^2} =\sqrt{m\eta^2 + n\gamma^2}.
\end{align*}
Further calculating the derivative w.r.t. $b$, we have $$
    b_{\text{std}} = -\frac{\log R}{2\sqrt{m\eta^2 + n\gamma^2}}.$$ 
    
\paragraph{The optimal robust classifier.} Similarly, for the robust loss, we have
\begin{center}
\resizebox{.98\linewidth}{!}{
\begin{minipage}{\linewidth}
\begin{align*}
    \ell_{\text{rob},\infty}(w,b) &=\frac{R}{R+1}\Phi\left(\frac{b-(\eta\sum_{i=1}^m w_i - \varepsilon\sum_{i=1}^m |w_i|) - (\gamma\sum_{j=m+1}^{m+n} w_j - \varepsilon\sum_{j=m+1}^{m+n} |w_j|)}{\sqrt{\sum_{k=1}^{m+n} w_k^2}}\right)\\ &+ \frac{1}{R+1}\Phi\left(\frac{-b-(\eta\sum_{i=1}^m w_i - \varepsilon\sum_{i=1}^m |w_i|) - (\gamma\sum_{j=m+1}^{m+n} w_j - \varepsilon\sum_{j=m+1}^{m+n} |w_j|)}{\sqrt{\sum_{k=1}^{m+n} w_k^2}}\right).
\end{align*}
\end{minipage}
}
\end{center}
We assume w.l.o.g. that $\sum_{k=1}^{m+n} w_k^2 = 1$. Using $\gamma < \varepsilon < \eta$ and Cauchy-Schwarz, we have
\begin{align*}
    \left(\eta\sum_{i=1}^m w_i - \varepsilon\sum_{i=1}^m |w_i|\right) + \left(\gamma\sum_{j=m+1}^{m+n} w_j - \varepsilon\sum_{j=m+1}^{m+n} |w_j|\right) \le (\eta-\varepsilon)\sum_{i=1}^m |w_i| \le (\eta-\varepsilon)\sqrt{m},
\end{align*}
and the inequalities are achieved when $w_1 = \cdots = w_m = \frac{1}{\sqrt{m}}$, and $w_{m+1} = \cdots = w_{m+n}=0$. Further calculating the derivative w.r.t. $b$, we have
$$
    b_{\text{rob},\infty} = -\frac{\log R}{2(\eta-\varepsilon)\sqrt{m}}.
$$

\section{Omitted Proofs from Section \ref{Section SD}} \label{Appendix SD}
Throughout this section, we will use $\Phi_\alpha$ to denote the cumulative distribution function of the standard $S\alpha S$ distribution $f(x;\alpha,1,0)$.
\subsection{Detailed Analysis of Subsection \ref{subsect.positive} --- Multivariate $S\alpha S$ Distribution with Independent Components} \label{appsub.ic}
Using the ``closed under linear transformation'' property, we have
    \begin{align*}
\ell_{\text{std}}(w,b) %
=\frac{1}{2}\Phi_\alpha\left(\frac{b+w^{\top}\theta^-}{\|w\|_\alpha}\right) + \frac{1}{2}\Phi_\alpha\left(\frac{-b-w^{\top}\theta^+}{\|w\|_\alpha}\right)
\end{align*}
for the standard loss and
\begin{align*}
    \ell_{\text{rob},p}(w,b) %
=\frac{1}{2}\Phi_\alpha\left(\frac{b+w^{\top}\theta^- + \varepsilon\|w\|_q}{\|w\|_\alpha}\right) + \frac{1}{2}\Phi_\alpha\left(\frac{-b-w^{\top}\theta^+ + \varepsilon\|w\|_q}{\|w\|_\alpha}\right)
\end{align*}
for the robust $\ell_p$ loss. Following the same procedure as in Section \ref{Section AT}, we assume w.l.o.g. that $\|w_{\text{std}}\|_\alpha = \|w_{\text{rob},\infty}\|_\alpha=1$. We use $v_1 \propto v_2$ to describe two vectors $v_1$ and $v_2$ differing by some \textit{positive} constant coordinate-wisely.

\paragraph{Analysis of $w$.} Introducing the Lagrangians and the KKT conditions give us
\begin{align}  \label{eq.KKT.std}
    \partial \|w_{\text{std}}\|_{\alpha} \propto \theta^+ - \theta^-.  
\end{align}
Similarly, for the robust $\ell_p$ loss, we have
\begin{align}  \label{eq.KKT.rob}
\partial \|w_{\text{rob},p}\|_{\alpha} \propto \theta^+ - \theta^- - 2\varepsilon\partial\|w_{\text{rob}}\|_q.
\end{align}

\paragraph{Analysis of $b$.} Fixing some optimal $w$ with $\|w\|_\alpha=1$, we can take partial derivatives w.r.t. $b$, and obtain
\begin{align*}
    \frac{\partial \ell_{\text{std}}(w,b)}{\partial b} = \frac{1}{2}\varphi_\alpha(b+w^{\top}\theta^-) - \frac{1}{2}\varphi_\alpha(-b-w^{\top}\theta^-),
\end{align*}
where we use $\varphi_\alpha$ to denote the probability density function of $f(x;\alpha, 1,0)$. Since $\varphi_\alpha$ is symmetric and monotonically decreasing on $(0,\infty)$ (see Theorem 1 in~\citep{gawronski1984bell}), we have either 
\begin{align*}
    b_{\text{std}} + w_{\text{std}}^{\top}\theta^- = -b_{\text{std}} - w_{\text{std}}^{\top}\theta^+ \quad \text{or} \quad b_{\text{std}} + w_{\text{std}}^{\top}\theta^- = b_{\text{std}} + w_{\text{std}}^{\top}\theta^+,
\end{align*}
and the latter is impossible due to Eq.~\eqref{eq.KKT.std} and Danskin's theorem. As a consequence, we have
$
    b_{\text{std}} = - \frac{w_{\text{std}}^{\top}(\theta^+ + \theta^-)}{2},
$
and similarly 
$
    b_{\text{rob},p} = - \frac{w_{\text{rob},p}^{\top}(\theta^+ + \theta^-)}{2}.
$

\paragraph{Putting together.} Combining the analysis for $w$ and $b$, we have
\begin{align*}
    \ell_{\text{std}}^+(w_{\text{std}},b_{\text{std}}) = \ell_{\text{std}}^-(w_{\text{std}},b_{\text{std}}) =  \Phi_\alpha\left(-\frac{w_{\text{std}}^{\top}(\theta^+ - \theta^-)}{2}\right)
\end{align*}
and
\begin{align*}
    \ell_{\text{nat}}^+(w_{\text{rob},p},b_{\text{rob},p}) = \ell_{\text{nat}}^-(w_{\text{rob},p},b_{\text{rob},p}) = \Phi_\alpha\left(-\frac{w_{\text{rob},p}^{\top}(\theta^+ - \theta^-)}{2}\right).
\end{align*}
To understand whether enforcing adversarial robustness provably hurt standard accuracy, it suffices to check whether the inequality
\begin{align*}
    w_{\text{rob},p}^{\top}(\theta^+ - \theta^-) <  w_{\text{std}}^{\top}(\theta^+ - \theta^-)
\end{align*}
holds. Note under the constraint $\|w\|_\alpha=1$, $w_{\text{std}}$ and $w_{\text{rob}, p}$ solve the following optimization problems respectively: 
\begin{align*}
    w_{\text{std}} = \arg\max_w w^{\top}(\theta^+-\theta^-), \quad w_{\text{rob}, p} = \arg\max_w w^{\top}(\theta^+-\theta^-)-2\varepsilon \|w\|_q.
\end{align*}
We can then prove Theorem \ref{Theorem same} and \ref{Theorem degrade} based on such formulation.

\begin{proof}[Proof of Theorem \ref{Theorem same}] We will discuss the two corner cases separately.

\paragraph{Case 1.} When $q = \alpha$, $w_{\text{rob}, p}$ has the same optimization formulation as $w_{\text{std}}$ since $\|w\|_q = \|w\|_\alpha = 1$ is a constant, so they have the same optimal values.

\paragraph{Case 2.} Suppose $\bar{\theta}:=\theta^+-\theta^-$ is isotropic, \ie, $|\bar{\theta}| \parallel \mathrm{1}_d$ and $\alpha \ge 1$. 
By H\"older's inequality, we have $w_{\text{std}} \parallel \bar{\theta}$. For $w_{\text{rob},p}$, denote $C= |\bar{\theta}_1|$, then
\begin{align*}
    w^{\top}(\theta^+-\theta^-)-2\varepsilon \|w\|_q &\le (C-2\varepsilon C_1) \|w\|_1 \\
    &\le C_2(C-2\varepsilon C_1) \|w\|_\alpha \\
    &=C_2(C-2\varepsilon C_1)
\end{align*}
by H\"older's inequality, where $C_1 = d^{\frac{1}{q}-1}$, $C_2 = d^{\frac{1}{\alpha}-1}$ and
\begin{align*}
    C - 2\varepsilon C_1 \ge C- 2\varepsilon >0 
\end{align*}
by our assumption on the perturbation radius. The two inequalities hold simultaneously when $w_{\text{rob},p} \parallel \bar{\theta}$. 
As a consequence, both the optimal standard and robust classifier achieve the same value.
\end{proof}

\begin{proof}[Proof of Theorem \ref{Theorem degrade}]
Denote the dual index of $\alpha$ as $\alpha'$, then by H\"older's inequality, 
\begin{align*}
    w^{\top}(\theta^{+} - \theta^-) \le \|w\|_\alpha\|\theta^+-\theta^-\|_{\alpha'} = \|\theta^+-\theta^-\|_{\alpha'},
\end{align*}
and there is exactly one $w$ with $\|w\|_{\alpha}=1$ that makes the equality hold. As a consequence, to show
\begin{align*}
    w_{\text{rob}, p}^{\top}(\theta^{+} - \theta^-) <  w_{\text{std}}^{\top}(\theta^{+} - \theta^-)
\end{align*}
it suffices to show $w_{\text{std}} \nparallel w_{\text{rob}, p}$, under the conditions listed in the theorem statement. In fact, if $w_{\text{std}} \parallel w_{\text{rob}, p}$ (we use $w$ to represent the corresponding direction), then by the KKT conditions, we have
\begin{align*}
    \partial\|w\|_q \parallel \partial\|w\|_\alpha,
\end{align*}
implying that $|w_1| = \cdots = |w_d|$ (since $\alpha> 1$, $1<q < \infty$ and $q\neq \alpha$). By the KKT condition of the standard classifier, we have
\begin{align*}
    \partial\|w\|_\alpha \propto \bar{\theta}
\end{align*}
which further implies that $|\bar{\theta}| \parallel \mathrm{1}_d$, and this is already precluded by the assumption in the theorem statement. As a consequence, we must have $w_{\text{std}} \nparallel w_{\text{rob}, p}$, and 
    \begin{align*}
    w_{\text{rob}, p}^{\top}(\theta^{+} - \theta^-) <  w_{\text{std}}^{\top}(\theta^{+} - \theta^-).
\end{align*}

\end{proof}

\subsection{Detailed analysis of Subsection \ref{subsec.perplex}} \label{appsub.perplex}
We will now analyze the intercept $b$ and the slope $w$ respectively. 

\paragraph{Analysis of $b$.} Following the same procedure as in Section \ref{Section AT} and the previous subsection, we assume w.l.o.g. that $\|w_{\text{std}}\|_1 = \|w_{\text{rob},\infty}\|_1=1$, then the standard loss and the robust $\ell_\infty$ loss write as
\begin{align*}
   \ell_{\text{std}}(w,b) = \frac{R}{R+1}\Phi_1\left({b+w^{\top}\theta^-}\right) + \frac{1}{R+1}\Phi_1\left({-b-w^{\top}\theta^+}\right)
\end{align*}
and
\begin{align*}
    \ell_{\text{rob},\infty}(w,b) = \frac{R}{R+1}\Phi_1\left({b+w^{\top}\theta^- +\varepsilon}\right) + \frac{1}{R+1}\Phi_1\left({-b-w^{\top}\theta^++\varepsilon}\right).
\end{align*}
Fixing $w$ with $\|w\|_1=1$, we are interested in finding the optimal $b$ that minimizes $\ell_{\text{std}}$ and $\ell_{\text{rob},\infty}$. Taking partial derivatives w.r.t. $b$ and simplifying the expressions, we have
\begin{align*}
    \text{sgn}\left(\frac{\partial \ell_{\text{std}}(w,b)}{\partial b}\right) = \text{sgn}\left(q_1(b)\right), \quad \text{sgn}\left(\frac{\partial \ell_{\text{rob},\infty}(w,b)}{\partial b}\right) = \text{sgn}\left(q_2(b)\right),
\end{align*}
where $q_1(b)$ and $q_2(b)$ are two quadratic functions defined through
\begin{align*}
    q_1(b) := (R-1)b^2 + \left(2Rw^{\top}\theta^+-2Rw^{\top}\theta^-\right)b + \left(R(w^{\top}\theta^+)^2 - (w^{\top}\theta^-)^2 + R-1\right)
\end{align*}
and 
\begin{align*}
    q_2(b) &:= (R-1)b^2 + \left(2Rw^{\top}\theta^+-2Rw^{\top}\theta^- - (2R+2)\varepsilon\right)b\\
    & \ +\left(R(w^{\top}\theta^+)^2 - (w^{\top}\theta^-)^2 + R-1 + (R-1)\varepsilon^2  - 2Rw^{\top}\theta^+\varepsilon - 2w^{\top}\theta^-\varepsilon\right),
\end{align*}
whose discriminants are given by
$$
\Delta_1 := R\left(w^{\top}\theta^+-w^{\top}\theta^-\right)^2 - (R-1)^2
$$
and
$$
\Delta_2 = R\left(w^{\top}\theta^+-w^{\top}\theta^--2\varepsilon\right)^2 - (R-1)^2 
$$
respectively. The expressions immediately lead us to the following proposition.

\begin{proposition} \label{prop.discriminant}
When the class imbalance ratio satisfies
\begin{align*}
    R \ge 2+ 4\|\theta^+-\theta^-\|_\infty^2,
\end{align*}
we have $\Delta_1, \Delta_2 < 0$, implying that the quadratic functions $q_1(b)$ and $q_2(b)$ are always positive.
\end{proposition}
\begin{proof}
By H\"older's inequality, we have
\begin{align*}
    |w^{\top}\theta^+-w^{\top}\theta^- | \le \|w\|_1\|\theta^+-\theta^-\|_\infty = \|\theta^+-\theta^-\|_\infty.
\end{align*}
The result then follows from the assumption on $\varepsilon$, as well as the fact that
\begin{align*}
    \frac{(R-1)^2}{R} \ge R+2.
\end{align*}
\end{proof}
\vspace{-4mm}
Without further digging into the analysis of $w$, Proposition \ref{prop.discriminant} itself is sufficient to characterize the behavior of the optimal standard classifier as well as the robust $\ell_\infty$ classifier --- since the quadratic functions are always positive and have the same signs as the partial derivatives, the optimal value is attained at $b = -\infty$ for both losses. Theorem \ref{Theorem zero-one} is a direct consequence of such observation.

\paragraph{Analysis of $w$.} We now switch to analyzing $w$. Following similar procedures as in Section \ref{Section AT} and the previous subsection, we can introduce the Lagrangians, and the KKT conditions reveal that both $w_{\text{std}}$ and $w_{\text{rob},\infty}$ satisfy the relation
$$
\partial \|w\|_1 \propto \theta^+- \theta^-.
$$
This further implies 
\begin{align} \label{Eq.cauchy-dual}
   \langle w_{\text{std}},\theta^+-\theta^- \rangle = \langle w_{\text{rob},\infty},\theta^+-\theta^- \rangle = \|\theta^+-\theta^-\|_\infty. 
\end{align}
Now
\begin{align*}
    \Delta_1 = R\|\theta^+-\theta^-\|_\infty^2 - (R-1)^2 \quad \Delta_2 = R\left(\|\theta^+-\theta^-\|_\infty-2\varepsilon\right)^2 - (R-1)^2,
\end{align*}
and we assume they are both non-negative. Plugging the optimal $b_{\text{std}}$ and $b_{\text{rob},\infty}$ (the larger root of the quadratic functions) back to the standard loss, we have 
\begin{align} \label{Eq.cauchy-std}
    \ell^+_{\text{std}}(w_{\text{std}}, b_{\text{std}}) = \Phi_1\left(\frac{-\|\theta^+-\theta^-\|_\infty - \sqrt{\Delta_1}}{R-1}\right), 
\quad \ell^-_{\text{std}}(w_{\text{std}}, b_{\text{std}}) = \Phi_1\left(\frac{-R\|\theta^+-\theta^-\|_\infty + \sqrt{\Delta_1}}{R-1}\right),
\end{align}
whereas 
\begin{align}  \label{Eq.cauchy-rob}
   \ell^+_{\text{std}}(w_{\text{rob},\infty}, b_{\text{rob},\infty}) &= \Phi_1\left(\frac{-\|\theta^+-\theta^-\|_\infty - (R+1)\varepsilon-\sqrt{\Delta_2}}{R-1}\right) \nonumber  \\
 \ell^-_{\text{std}}(w_{\text{rob},\infty}, b_{\text{rob},\infty}) &= \Phi_1\left(\frac{-R\|\theta^+-\theta^-\|_\infty + (R+1)\varepsilon + \sqrt{\Delta_2}}{R-1}\right). 
\end{align}
Define 
$$
d(\varepsilon) := (R+1)\varepsilon +\sqrt{\Delta_2},
$$
then it is straightforward to see that $d(0) = \sqrt{\Delta_1}$. By comparing Eq.~\eqref{Eq.cauchy-std} and~\eqref{Eq.cauchy-rob}, we know that it is essentially the relation between $d(\varepsilon)$ and $d(0)$ that determines the change of accuracy disparity; specifically, to prove Theorem \ref{theorem.reduce} it suffices to show $d(\varepsilon) < d(0)$.

\begin{proof}[Proof of Theorem \ref{theorem.reduce}]
Calculating the derivative, we have
$$
d'(s) = R+1 - \frac{2R}{\sqrt{R - \left(\frac{R-1}{\|\theta^+-\theta^-\|_\infty-2s}\right)^2}}.
$$
When $s \le \frac{\kappa}{2}\|\theta^+-\theta^-\|_\infty$, it is straightforward to see that $d'(s) < 0$, hence $d(\varepsilon) < d(0)$ by our assumption on the perturbation radius. 

\end{proof}

\subsection{Elliptically-Contoured Multivariate $S\alpha S$ Distribution} \label{appsub.ec}
Here we will introduce another multivariate $S\alpha S$ distribution whose joint characteristic function has closed form.

\paragraph{Elliptically-Contoured.} We say a multivariate distribution $X$ is elliptically-contoured~\citep{samorodnitsky1996stable}, if it has joint characteristic function
\begin{align*}
    \mathbb{E}\exp\left(is^{\top}X\right) = \exp\{-(s^{\top}\Sigma s)^{\alpha/2} + is^{\top}\mu\},
\end{align*}
where $\Sigma$ is a positive semi-definite \textit{shape matrix} and $\mu$ is the location parameter. We use $S\alpha S_{EC}(\mu, \Sigma)$ to denote such elliptically-contoured multivariate $S\alpha S$ distribution.
  
Suppose the data are generated through a mixture of elliptically-contoured multivariate $S\alpha S$ distributions: $\mathcal{P}^+ = S\alpha S_{EC}(\theta^+, \Sigma)$ and $\mathcal{P}^- = S\alpha S_{EC}(\theta^-, \Sigma)$.
We will now show in the following theorem that, since the shape matrix is an analogy to the covariance matrix as in the Gaussian case, we can obtain a similar conclusion of Corollary \ref{corollary.degrade} when the data are drawn from a mixture of elliptically-contoured $S\alpha S$ distribution.

\begin{theorem}\label{thm:elliptical}
Suppose the data are generated through a mixture of  elliptically-contoured multivariate $S\alpha S$ distributions. When there is no class imbalance, \ie, $R=1$, enforcing {adversarial robustness with $\ell_p$-constraint will increase the error on both classes}, so long as the ``direction'' of the optimal robust classifier is not parallel to its counterpart, \ie, $\Sigma^{\frac{1}{2}}w_{\emph{std}} \nparallel \Sigma^{\frac{1}{2}}w_{\emph{rob},p}$. Furthermore, if the shape matrix $\Sigma$ is positive definite, then the requirement $\Sigma^{\frac{1}{2}}w_{\emph{std}} \nparallel \Sigma^{\frac{1}{2}}w_{\emph{rob},p}$ is equivalent to $w_{\emph{std}} \nparallel w_{\emph{rob},p}$, \ie, a shift in direction in the standard sense.
\end{theorem}

\begin{proof}[Proof of Theorem \ref{thm:elliptical}]
Note for $X \sim S\alpha S_{EC}(\mu, \Sigma), w \in \mathbb{R}^d$ and arbitrary $t \in \mathbb{R}$, we can set $s = tw$, and the definition of elliptically-contoured multivariate $S\alpha S$ distribution gives us
$
    \mathbb{E}\exp\left(it w^{\top}X\right) = \exp\{-|t\sqrt{w^{\top}\Sigma w}|^{\alpha} + it w^{\top}\mu\},
$
so according to the definition of univariate $S\alpha S$ distribution,
\begin{align*}
    w^{\top}X \sim f(x; \alpha, \sqrt{w^{\top}\Sigma w}, w^{\top}\mu).
\end{align*}
Therefore, the ``closed under linear transformation'' property yields
\begin{align*}
\ell_{\text{std}}(w,b) %
= \frac{1}{2}\Phi_\alpha\left(\frac{b+w^{\top}\theta^-}{\sqrt{w^{\top}\Sigma w}}\right) + \frac{1}{2} \Phi_\alpha\left(\frac{-b-w^{\top}\theta^+}{\sqrt{w^{\top}\Sigma w}}\right),
\end{align*}
and similarly
\begin{align*}
    \ell_{\text{rob},p}(w,b) %
=\frac{1}{2}\Phi_\alpha\left(\frac{b+w^{\top}\theta^- + \varepsilon \|w\|_q}{\sqrt{w^{\top}\Sigma w}}\right) +\frac{1}{2} \Phi_\alpha\left(\frac{-b-w^{\top}\theta^+ + \varepsilon \|w\|_q}{\sqrt{w^{\top}\Sigma w}}\right).
\end{align*}
The only difference compared to the Gaussian case discussed in Section \ref{Section AT} is that the cumulative function is replaced by $\Phi_{\alpha}$; moreover, a closer examination of the proof of Theorem \ref{thm:general} (see Appendix \ref{Appendix AT}) shows that we only resort to the symmetry and monotonicity of $\Phi$, and these properties are still preserved in $\Phi_{\alpha}$. As a consequence, we obtain Theorem \ref{thm:elliptical} as desired.
\end{proof}

\section{Additional Experimental Details, Results, and Analyses}
\label{adxsec:exp}

\subsection{Omitted Details of Experimental Setup}
\label{adxsubsec:exp-setup}

In this part, we provide a concrete description of our experimental setups.
Additionally, we release our code at \url{https://github.com/Accuracy-Disparity/AT-on-AD}, where we include the datasets, the code, and the instructions for reproducing the experiments.

\noindent \textbf{Datasets.}\quad
We altogether experiment with {seven} groups of datasets, including three groups of synthetic datasets and {four} groups of real-world datasets. Here, for ease of reference, we refer to the balanced and imbalanced datasets ($R=1,2,5,10$) constructed from the original balanced dataset as a \textit{group} of dataset, which consists of four datasets corresponding to four choices of $R$.

We will first describe the dataset properties of the original balanced dataset, and then explain how we construct the imbalanced datasets in the dataset group. In the end, we introduce the training, validation, and test split for performing training and evaluation.

\noindent \textit{Dataset Properties.}~
The three balanced \textbf{synthetic} datasets are constructed to be a mixture of two Gaussian, Cauchy, or other stable distributions with $1 < \alpha < 2$. We set the sample size of both classes as $N=10,000$ and the sample dimension size as $d=100$. For the Gaussian case, the two means are sampled from $U[0,1]^d$ and $U[-1,0]^d$, and their (same) variances are set as $AA^\top$ where $A \sim \mathcal{N}(0,I_d)$. 
For the Cauchy case and the $S\alpha S$ stable distribution, we construct the data with \noindent \textit{independent components}.
In both cases, for each dimension, we sample the location parameter from $U[0,0.5]$ (or $U[-0.5,0]$) and set the scale parameter as $1$.
We set the parameter $\alpha=1$ to construct the Cauchy dataset and $\alpha=1.5$ for the other one.
We release our synthetic datasets at \url{https://github.com/Accuracy-Disparity/AT-on-AD}.
The {four} balanced \textbf{real-world} datasets are built upon the handwritten digits dataset \textit{MNIST}~\citep{lecun1998gradient} under the Creative Commons Attribution-Share Alike 3.0 license, the fashion products dataset \textit{Fashion-MNIST}~\citep{xiao2017fashion} under the MIT license, \textit{CIFAR-10}~\citep{krizhevsky2009learning} under the MIT license, {and \textit{ImageNet}}. 
MNIST and Fashion-MNIST consist of grey-scale images of dimensionality $28 \times 28$; CIFAR-10 consists of colored images of dimensionality $32 \times 32 \times 3$. 
Originally, the three datasets are used for 10-class classification. 
To adapt for the binary classification task we consider here, we choose two classes from all ten --- digit 1 and digit 7 for MNIST, T-shirt and trouser for Fashion-MNSIT, and cat and dog for CIFAR-10.
{{
ImageNet consists of colored images of various dimensionality. We use the downsampled dataset with image dimensionality $64 \times 64 \times 3$~\citep{chrabaszcz2017downsampled}.
The dataset contains $1,000$ classes following a semantic hierarchy\footnote{A diagram of the hierarchy can be found at \url{https://observablehq.com/@mbostock/imagenet-hierarchy}}.
For our binary classification task, we choose two ``macro'' classes in the hierarchy, ``car'' and ``edible fruit'', each consisting of $10$ classes. 
We do not choose two lowest level of classes (among all $1,000$ classes) from ImageNet, since the number of samples per class (around $1,000$) is insufficient for training the deep neural network.
}}
By their original construction, these real-world datasets are class balanced. 

\noindent \textit{Construction of the Imbalanced Dataset.}~
We next explain how we obtain the dataset group consisting of datasets of various class imbalance ratios ($R=2,5,10$) from the balanced one ($R=1$) we initially construct. 
Concretely, we obtain datasets with increasing class imbalance ratio sequentially. For each dataset, we hold the majority class samples as fixed (\ie, same as the $R=1$ case), and subsample the minority class samples from the previously constructed set of minority class samples. (For example, we sample from the minority set in the $R=2$ dataset to obtain the minority set for the $R=5$ dataset.)

\noindent \textit{Dataset Partition.}~
For each dataset in each dataset group, we split the dataset into three disjoint partitions: training, validation, and testing. 
For the synthetic dataset, we set the ratio of the three partitions to be 8:1:1, which gives us a total number of 8000 training samples, 1000 validation samples, and 1000 testing samples for the majority class in each dataset. (The sizes of the three partitions for the minority class are the sizes for the majority class divided by the imbalance ratio $R$).
For the real-world datasets {MNIST, Fashion-MNIST and CIFAR}, since the datasets are originally split into training and testing, we further split the training set into training and validation with a ratio of 8:1.
For MNIST and Fashion-MNIST, the numbers of training, validation, and testing samples for the majority class are 5333, 533, and 1000 respectively. 
For CIFAR-10, the numbers are 4444, 556, and 1000 respectively.
{For ImageNet, there are in all 13000 images for each ``macro'' class in the training set. 
For the majority class, the numbers of training, validation, and testing samples are 10000, 2000, and 1000 respectively.}
We release all our datasets at \url{https://github.com/Accuracy-Disparity/AT-on-AD}.

\noindent \textbf{Models.}\quad
For most experiments, we mainly adopt simple models on these datasets --- linear classifiers on the three synthetic datasets, MNIST, and Fashion-MNIST, as well as networks with two linear layers on CIFAR-10.

The reason we mainly experiment with simple models is that they are sufficiently powerful for the binary classification task, as we can see from the accuracy results in Table \ref{tab:acc-overall-1} and \ref{tab:acc-overall-2} in Appendix \ref{adxsubsec:exp-result}.
Actually, on CIFAR-10, we also experiment with a more complicated convolutional neural network VGG-11~\citep{Simonyan2015very}; we derive similar conclusions on this complicated network as we did on simple networks. The details can be found in Appendix \ref{adxsubsec:exp-analysis}.

{Out of interest for deep neural networks on complicated datasets, we also experiment with ImageNet and apply a deep neural network ResNet18~\citep{he2016deep}.
}

\noindent \textbf{Training protocols.}\quad
We perform \textit{standard training} and \textit{adversarial training} on the binary classification task. We first describe the common training protocols for both training schemes, and then go into details on the specificity of adversarial training.

\looseness=-1
We adopt stochastic gradient descent~(SGD) optimizer~\citep{kiefer1952stochastic} or Adam optimizer~\citep{kingma2015adam} for network training. We take the cross entropy loss as the training objective. 
We perform model selection on a held-out validation set (as introduced in the data part).
We perform a grid search for the hyper-parameters including learning rate, batch size, and hidden layer size (when applicable) based on the model’s performance on the validation set. 
The search space for learning rate is $\{0.1,0.05,0.01,0.005,0.001,0.0005,0.0002,0.0001\}$, for batch size is $\{32, 64, 128\}$, and for hidden layer size is $\{100, 200, 500, 1000, 2000\}$.
For each model, we perform training for a maximum of $500$ training epochs; we keep track of the best model throughout the training based on the validation loss and apply early stopping~\citep{prechelt1998early} when the lowest validation loss does not decrease for the past $50$ epochs.
{For ImageNet specifically, we perform early stopping when the loss does not decrease for the past $10$ epochs.}

For adversarial training, we follow Goodfellow et al.~\citep{goodfellow2014explaining} and Madry et al.~\citep{madry2017towards} to craft the adversarial examples via the fast gradient method (FGM) or the projected gradient descent (PGD). We adopt the former for linear classifiers and the latter for two layer neural networks and deep neural networks~(VGG-11 and ResNet18).
For all datasets and all $\ell_p$ norms ($p \in \{ 2, \infty \}$), we experiment with three perturbation scales $\eps$. The perturbation scales are selected based on the $\ell_p$ distances between the empirical means of the two classes. For most of the datasets, we choose the three values as $1/4, 3/8, 1/2$ of the distance. For CIFAR-10 and $\ell_2$ only, we select slightly larger perturbation scales, following the practice in Tsipras et al.~\citep{tsipras2018robustness}.
For PGD attack specifically, we set the step number to be $50$ and the limit on the per step size to be $2.5 \cdot \eps / 50$ following Madry et al.~\citep{madry2017towards}.
{On ImageNet, we limit the step number to $10$ considering the high computation cost.}

\subsection{Additional Experiments on Synthetic Datasets of Stable Distributions}
\label{adxsubsec:exp-stable}

\renewcommand{\thesubfigure}{\alph{subfigure}}

\begin{figure}

\newlength{\utilheightsyn}
\settoheight{\utilheightsyn}{\includegraphics[width=.25\linewidth]{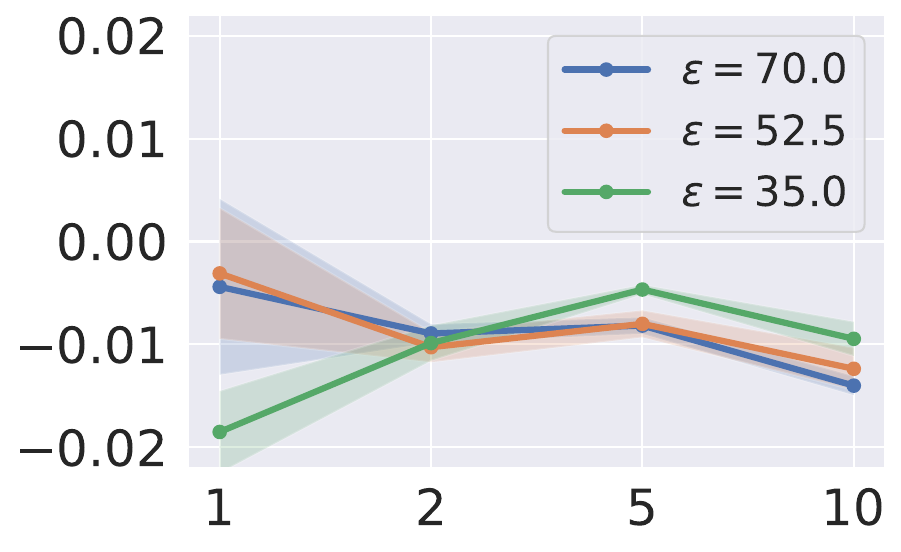}}%

\newcommand{\rowname}[1]%
{\rotatebox{90}{\makebox[\utilheightsyn][c]{\tiny #1}}}

\centering

{
\renewcommand{\tabcolsep}{10pt}

\begin{subtable}[]{\linewidth}
\centering
\resizebox{0.95\linewidth}{!}{%
\begin{tabular}{@{}p{4mm}@{}c@{}c@{}c@{}c@{}c@{}c@{}}
        & \makecell{\small{\textbf{Cauchy, $p=\infty$}}}
        & \makecell{\small{\textbf{Cauchy, $p=2$}}}
        & \makecell{\small{\textbf{Holtsmark, $p=\infty$}}}
        & \makecell{\small{\textbf{Holtsmark, $p=2$}}}
        \\
\rowname{\makecell{\scriptsize $AD^{R}_{\trob,p,\eps} - AD^{R}_{\tstd}$}}&
\includegraphics[height=\utilheightsyn]{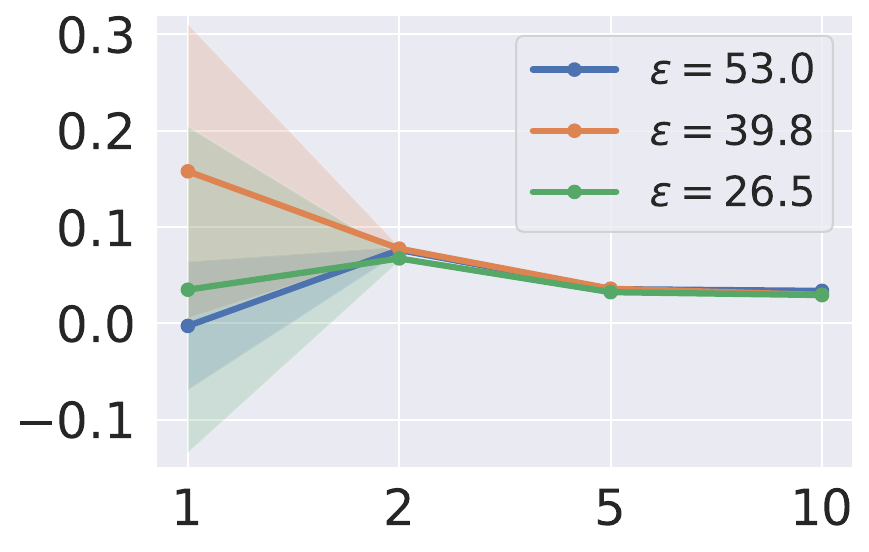}&
\includegraphics[height=\utilheightsyn]{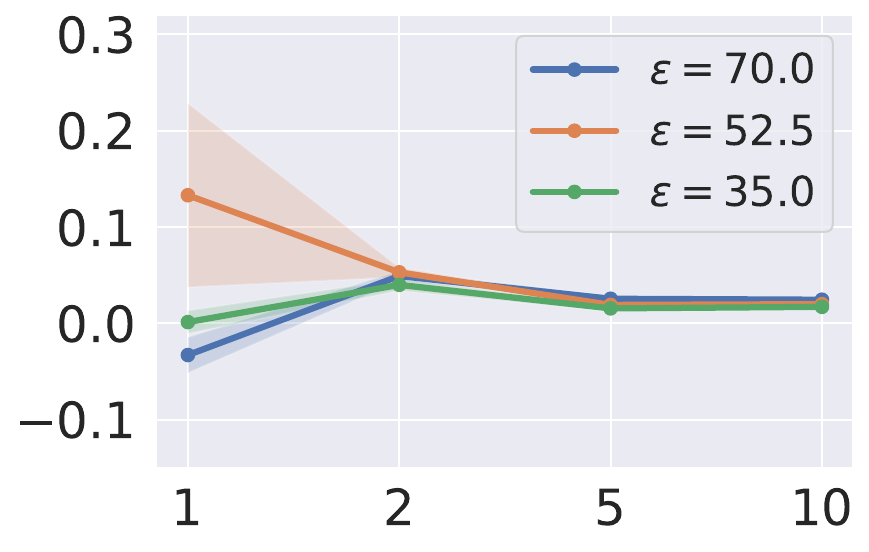}&
\includegraphics[height=\utilheightsyn]{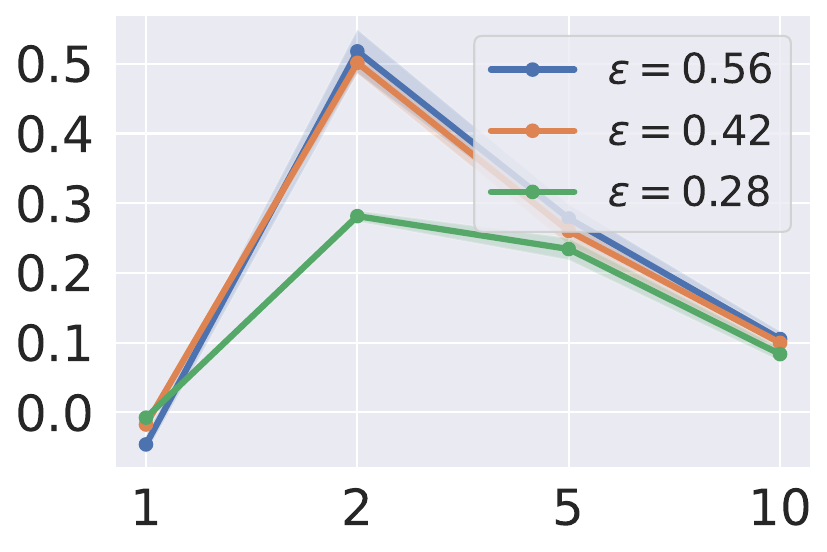}&
\includegraphics[height=\utilheightsyn]{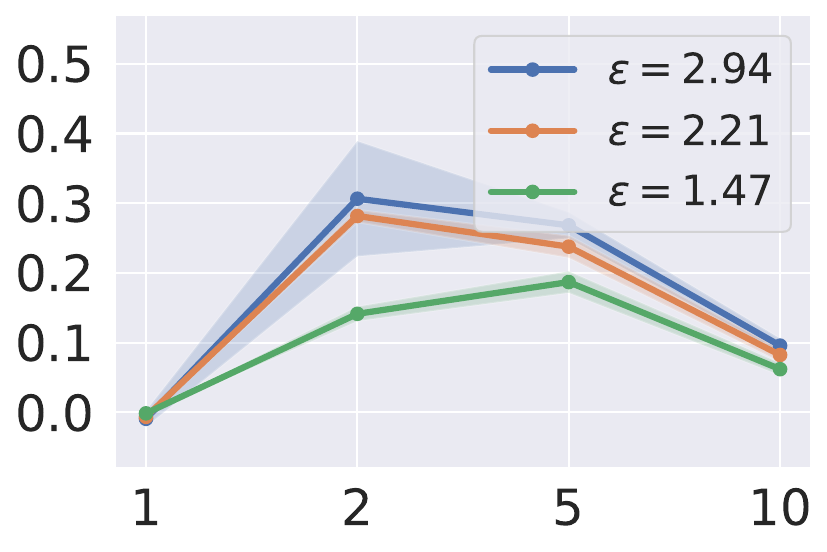}&
\\
\rowname{\makecell{\scriptsize $acc^{R,\cdot}_{\tstd} - acc^{R,\cdot}_{\trob,p,\eps}$}}&
\includegraphics[height=\utilheightsyn]{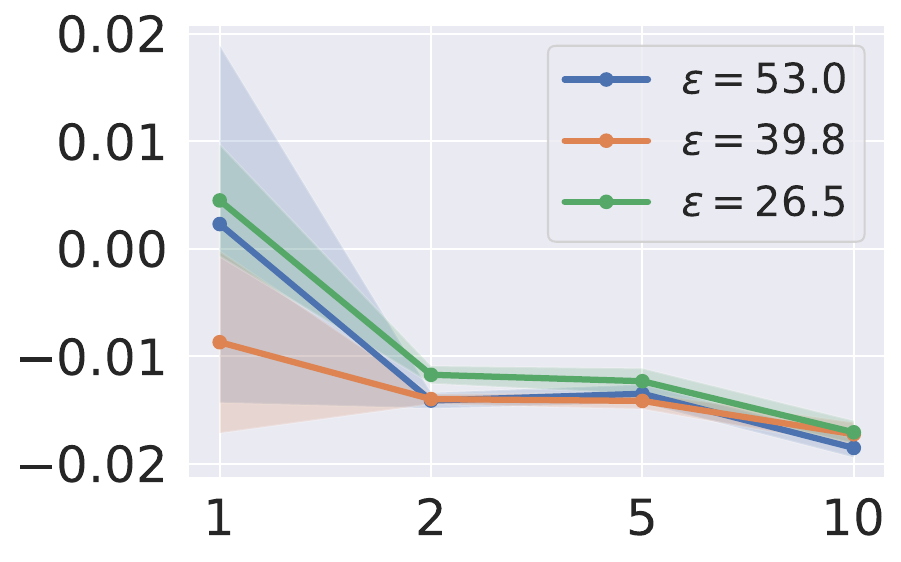}&
\includegraphics[height=\utilheightsyn]{figs/cauchy_2_l2_acc.pdf}&
\includegraphics[height=\utilheightsyn]{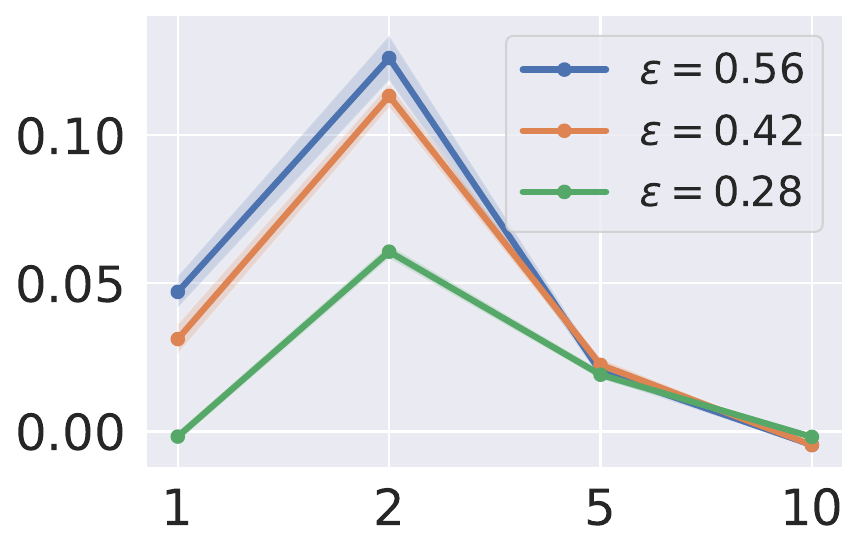}&
\includegraphics[height=\utilheightsyn]{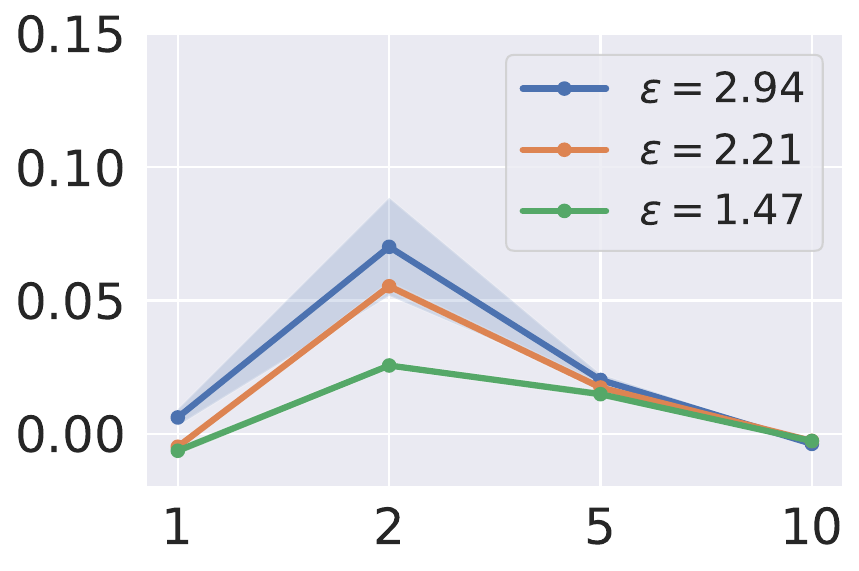}&
\\[-2mm]
        & \makecell{\scriptsize{$R$}}
        & \makecell{\scriptsize{$R$}}
        & \makecell{\scriptsize{$R$}}
        & \makecell{\scriptsize{$R$}}
\end{tabular}
}
\end{subtable}

}
\vspace{-2mm}
\caption{\small The gap of accuracy disparity $AD^R_{\trob,p,\eps} - AD^R_{\tstd}$ (\nth{1} row, \rqtwo) and the gap of standard accuracy $acc^{R,\cdot}_{\tstd} - acc^{R,\cdot}_{\trob,p,\eps}$ (\nth{2} row, \rqthree)  w.r.t. the imbalance ratio $R$.
For the robust classifiers, we consider $p \in \{2, \infty\}$ and multiple perturbation scales.
We present the results on the two synthetic datasets of $S \alpha S$ distribution where $\alpha = 1$ or $1.5$.
The shaded area represents the standard error of $5$ runs.
}%
\label{fig:app-syn}
\end{figure}

As a complement to the Gaussian case evaluated in Section \ref{sec:exp} in the main paper, here, we evaluate general symmetric $\alpha$-stable distributions corresponding to the theoretical results in Section \ref{Section SD}.
We consider two values of $\alpha$ --- the special case $\alpha=1$ which is commonly known as the Cauchy distribution, as well as the case $\alpha=1.5$ which is known as the Holtsmark distribution~\citep{holtsmark1919verbreiterung} and can be viewed as an ``intermediate'' between Gaussian ($\alpha=2$) and Cauchy.
The construction details of the dataset are described in Appendix \ref{adxsubsec:exp-setup}.

We present the results on these two datasets
in Figure \ref{fig:app-syn}.

In the Cauchy case, the two classes in our constructed dataset are barely separable, so in the balanced case $R=1$, the classification accuracy for both classes are close to random guess with value around $0.5$. (Detailed numbers can be found in Table \ref{tab:acc-per-class-syn}.)
In the class imbalance settings ($R=2,5,10$), the predictions of both the standard and the robust classifiers favor the majority class, and so the accuracy disparity is close to $1$ in both cases, leading to the accuracy disparity gap of nearly $0$.
This result provides the answers to \rqone and \rqtwo and aligns well with the theoretical result in Theorem \ref{Theorem zero-one}.
Regarding \rqthree which asks whether and when adversarial training will worsen the standard accuracy, we comment that on this Cauchy dataset we investigate, the classification outcomes are similar for the standard and robust classifiers.
In the balanced class setting, the hardness of the dataset dominates and leads to random guesses in both training scenarios; in the class imbalance setting, the class imbalance dominates and leads to unanimous preference towards the majority class.
Thus, it is difficult to draw conclusions regarding the factor of adversarial training.

For the Holtsmark distribution, unsurprisingly, we find that the phenomenon is in the between of the Guassian and Cauchy in terms of accuracy disparity. 
Concretely, we do see that adversarial training will exacerbate the accuracy disparity in the class imbalance setting, but the gap does not increase with the imbalance ratio $R$. This observation aligns partially with Theorem \ref{cor:imbalance} and partially with Theorem \ref{Theorem zero-one}, but not both.
In terms of the standard accuracy, we see that adversarial training leads to a decrease of standard accuracy, which is consistent with Theorem \ref{Theorem degrade}.

In all, from these experiments on synthetic datasets of stable distributions, we see that 1) class imbalance will be a dominating factor which surpasses the influence of adversarial training for the Cauchy case; 2) $\alpha$ value between $1$ and $2$ will lead to intermediate behavior of Gaussian (Theorem \ref{cor:imbalance}) and Cauchy (Theorem \ref{Theorem zero-one}) regarding the accuracy disparity; 3) adversarial training will invariably decrease the standard accuracy empirically.

\subsection{Additional Experiments on Real-World Datasets Fashion-MNIST {and ImageNet}}
\label{adxsubsec:exp-fmnist}

\renewcommand{\thesubfigure}{\alph{subfigure}}

\begin{figure}

\newlength{\utilheightapp}
\settoheight{\utilheightapp}{\includegraphics[width=.25\linewidth]{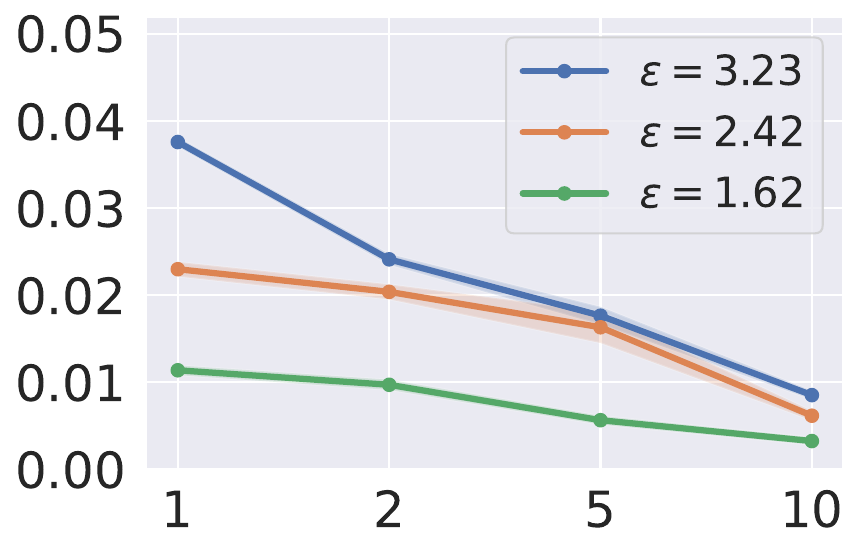}}%

\newcommand{\rowname}[1]%
{\rotatebox{90}{\makebox[\utilheightapp][c]{\tiny #1}}}

\centering

{
\renewcommand{\tabcolsep}{10pt}

\begin{subtable}[]{\linewidth}
\centering
\resizebox{0.95\linewidth}{!}{%
\begin{tabular}{@{}p{4mm}@{}c@{}c@{}c@{}c@{}c@{}c@{}}
        & \makecell{\small{\textbf{Fashion-MNIST, $p=\infty$}}}
        & \makecell{\small{\textbf{Fashion-MNIST, $p=2$}}}
        & \makecell{\small{\textbf{ImageNet, $p=\infty$}}}
        & \makecell{\small{\textbf{ImageNet, $p=2$}}}
        \\
\rowname{\makecell{\scriptsize $AD^{R}_{\trob,p,\eps} - AD^{R}_{\tstd}$}}&
\includegraphics[height=\utilheightapp]{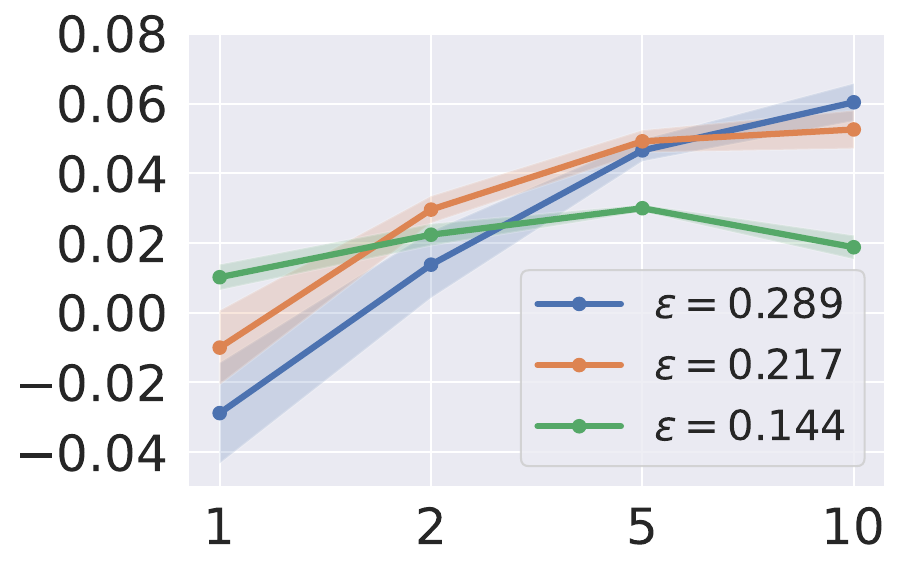}&
\includegraphics[height=\utilheightapp]{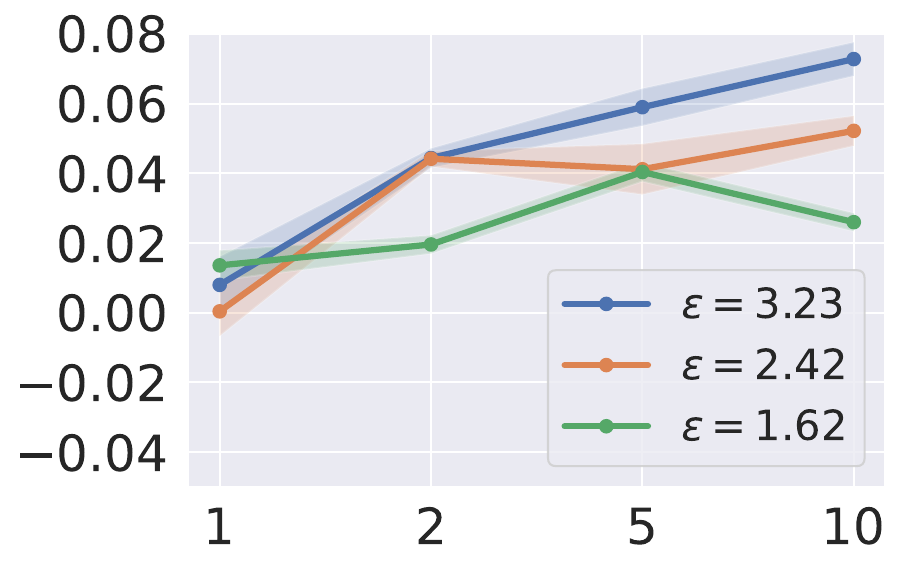}&
\includegraphics[height=\utilheightapp]{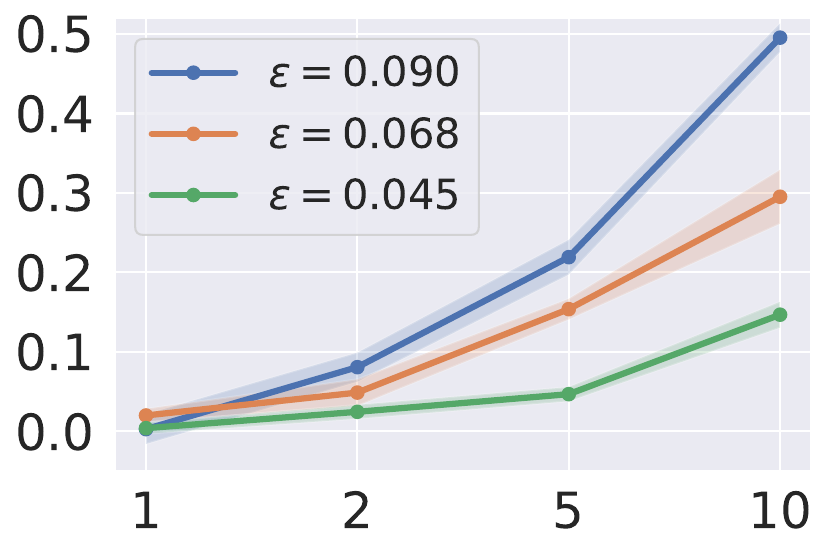}&
\includegraphics[height=\utilheightapp]{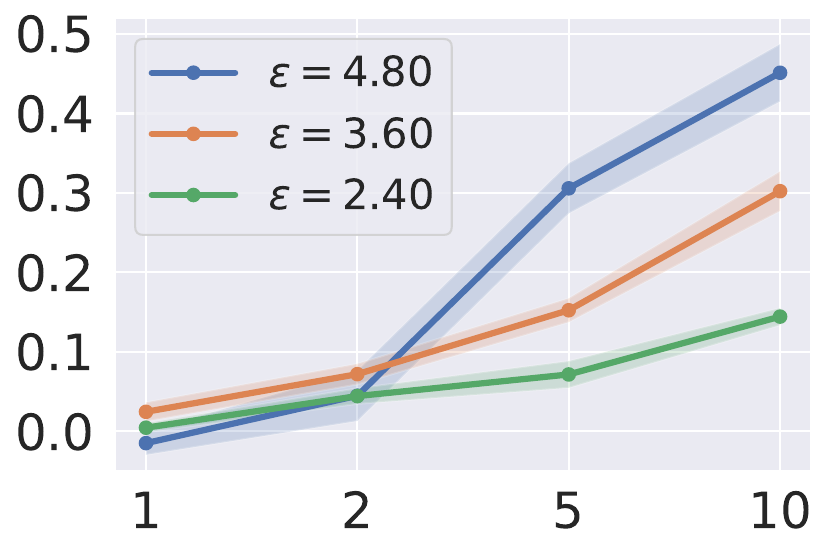}&
\\
\rowname{\makecell{\scriptsize $acc^{R,\cdot}_{\tstd} - acc^{R,\cdot}_{\trob,p,\eps}$}}&
\includegraphics[height=\utilheightapp]{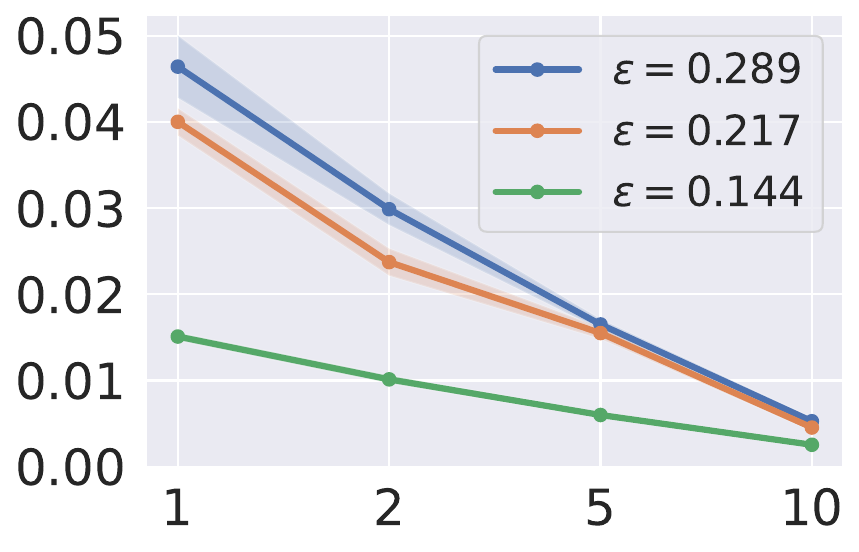}&
\includegraphics[height=\utilheightapp]{figs/fmnist_l2_acc.pdf}&
\includegraphics[height=\utilheightapp]{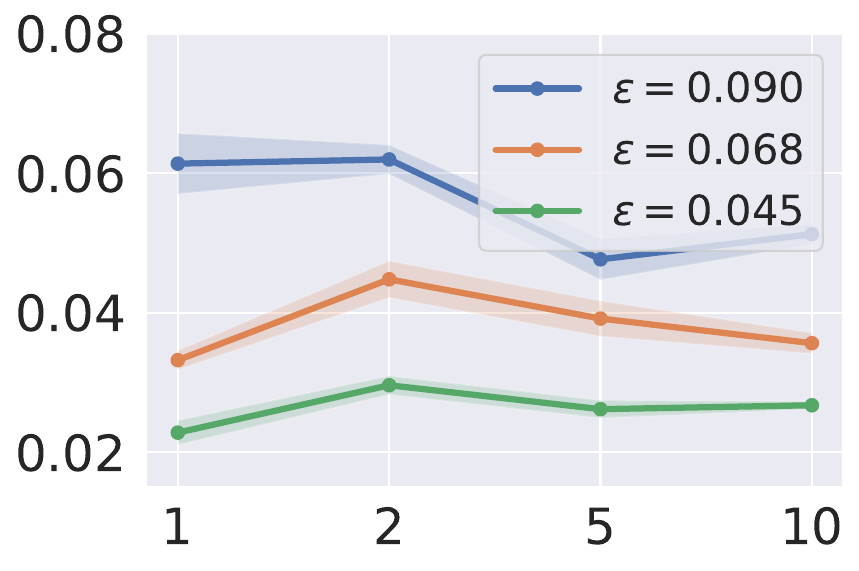}&
\includegraphics[height=\utilheightapp]{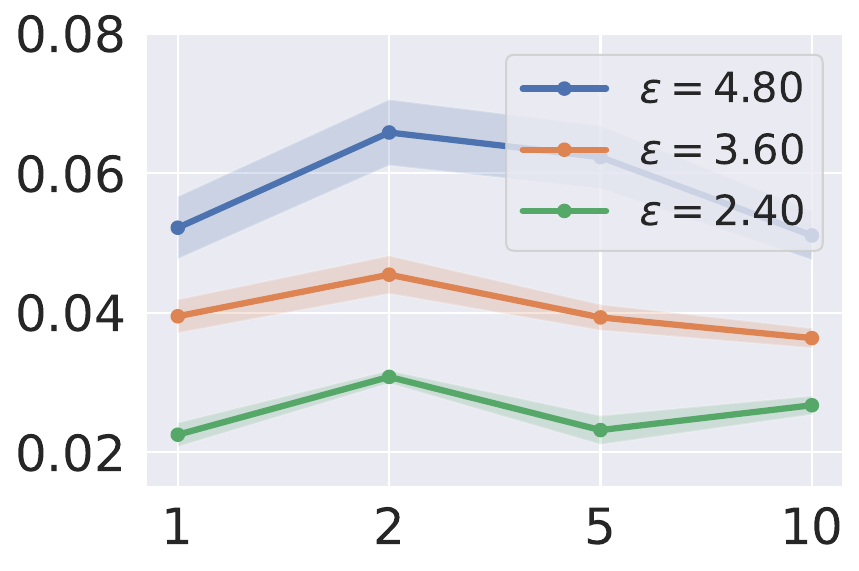}&
\\[-2mm]
        & \makecell{\scriptsize{$R$}}
        & \makecell{\scriptsize{$R$}}
        & \makecell{\scriptsize{$R$}}
        & \makecell{\scriptsize{$R$}}
\end{tabular}
}
\end{subtable}

}
\vspace{-2mm}
\caption{\small The gap of accuracy disparity $AD^R_{\trob,p,\eps} - AD^R_{\tstd}$ (\nth{1} row, \rqtwo) and the gap of standard accuracy $acc^{R,\cdot}_{\tstd} - acc^{R,\cdot}_{\trob,p,\eps}$ (\nth{2} row, \rqthree)  w.r.t. the imbalance ratio $R$.
For the robust classifiers, we consider $p \in \{2, \infty\}$ and multiple perturbation scales.
We present the results on  {two real-world datasets \textbf{Fashion-MNIST} and \textbf{ImageNet}}.
The shaded area represents the standard error of $5$ runs.
}%
\label{fig:app-fmnist}
\end{figure}

\looseness=-1
{We evaluate two additional real-world datasets Fashion-MNIST~\citep{xiao2017fashion} and ImageNet~\citep{deng2009imagenet}.}
Fashion-MNIST is often used as a drop-in replacement for MNIST~\citep{lecun1998gradient}.
{ImageNet is a large-scale dataset consisting of $1,000$ classes of high-dimensional images.
The dataset description and the details on the construction of the dataset are provided in Appendix \ref{adxsubsec:exp-setup}.
}

We present the results in Figure \ref{fig:app-fmnist}.
Comparing the results with that in Figure \ref{fig:main} in Section \ref{sec:exp} of the main paper, we obtain highly similar observations and conclusions w.r.t. all our three research questions.
{Concretely, both Fashion-MNIST and ImageNet display similar behavior with the Gaussian mixture case --- adversarial training exacerbates accuracy disparity compared with standard training which is more severe with increased imbalance (\nth{1} row), and adversarial training worsens the standard accuracy when $R=1$ (\nth{2} row).

In order to understand how ``close'' these two real-world datasets are to the Gaussian mixture, we follow the approach described in Appendix \ref{adxsubsec:exp-analysis} to compute the outlier ratio (\ie, the ratio to samples that are $1$ sigma away from the empirical mean after performing preconditioning).
The outlier ratios are $0.30$ and $0.26$ for the two classes of Fashion-MNIST, and $0.00$ and $0.00$ for ImageNet.
Compared with the outlier ratios $0.52$ and $0.52$ for CIFAR (see Appendix \ref{adxsubsec:exp-analysis}), Fashion-MNIST and ImageNet are indeed less heavy-tailed.

Thus, we show that Fashion-MNIST and ImageNet are additional evidences for the potential of extending the theoretical results to real-world datasets that can be roughly modeled as Gaussian mixtures.
}

\subsection{Additional Analysis on the Statistical Properties of Real-World Datasets}
\label{adxsubsec:exp-analysis}

From Section \ref{subsec:exp-analysis-1} in the main paper, we see that the results on MNIST resemble the theoretical analysis on the Gaussian distribution, while the results on CIFAR resemble the analysis on the Cauchy distribution.
In order to understand whether CIFAR is indeed more heavy-tailed than MNIST, we look into the statistical properties of the two datasets and make a comparison.

We compute the ratio of the outlier of the dataset as a proxy of how ``heavy'' the tail is.
Concretely, we first perform a preconditioning on the dataset such that the covariance of the preconditioned dataset becomes an identity matrix.
We leverage the PCA whitening approach~\citep{friedman1987exploratory} to achieve the goal.
Then, we compute an empirical mean of the dataset and check for how many instances are $1$ sigma away from the empirical mean (\ie, the distance between the instance and the empirical mean is larger than $\sqrt{d}$, where $d$ is the dataset dimensionality).

We follow the above approach to compute the outlier ratio for two classes separately in both datasets.
The outlier ratios are as high as 0.52 and 0.52 for the two classes of CIFAR, while only 0.13 and 0.16 for MNIST.
This means that CIFAR is indeed much more heavy-tailed than MNIST, supporting the experimental results in Section~\ref{subsec:exp-analysis-1}.

\looseness=-1
\noindent \textbf{Ruling out Other Possible Influencing Factors.}~
As we can observe from Table \ref{tab:acc-per-class-real}, the reason why there is only small accuracy disparity gap for CIFAR in the class imbalance case is that the accuracy of the standard classifier on the minority class is close to $0$.
The heavy tailed property is one possible explanation (Theorem \ref{Theorem zero-one}); the other straightforward 
hypothesis is that the standard classifier is not well trained on CIFAR, either because the dataset size is relatively small, or because the model capacity is limited. 
We then separately study the influence of the two factors:

\begin{itemize}%
    \item \textbf{Dataset Size.}~
    \looseness=-1
    Since we cannot enlarge the dataset size for CIFAR, we instead shrink the dataset size for MNIST. 
    We use $53$ majority class samples and $5$ minority class samples (\ie, $R=10$) to train the model, which is $1/100$ of the size of the original training data. 
    Compared with the original result, the majority class accuracy remains $1.00$, and the minority class accuracy drops from $0.97$ to $0.88$.
    \item \textbf{Model Capacity.}~
    Instead of the original two linear layer network, we adopt a deep convolutional network VGG-11~\citep{Simonyan2015very} to train the standard classifier on the $R=10$ case for CIFAR.
    As a result, the minority class accuracy increases from $0.00$ to $0.23$.
\end{itemize}

From the above results, we see that neither shrinking the dataset size nor increasing the model capacity can significantly impact the accuracy disparity gap.
Thus, we can confidently rule out these alternative explanations. 
We conclude that the main contributor that leads to the distinction is the distributional characteristic (specifically, the tail property) of the dataset.

\subsection{Full Experimental Results}
\label{adxsubsec:exp-result}

{
\setlength{\tabcolsep}{3pt} %
\begin{table}[tbp]
    \centering

    \caption{\small \textbf{Accuracy disparity gap} $AD_R^{\trob,p,\eps} - AD_R^{\tstd}$ for various choices of $p$ and $\eps$ on {seven} dataset groups. Results are averaged over $5$ runs with different random seeds. (Corresponding to \nth{1} rows of~\Cref{fig:main,fig:app-fmnist,fig:app-syn}.)
    }
    \label{tab:ad}

\resizebox{.58\linewidth}{!}{%
        \begin{tabular}{lcccccc
}
    \toprule
    \multirow{2}{*}{\makecell[c]{\textbf{Synthetic}\\\textbf{Gaussian}}}
     & \multicolumn{3}{c}{\small $p=\infty$} & \multicolumn{3}{c}{\small $p=2$} \\
    \cmidrule(lr){2-4}\cmidrule(lr){5-7}
     & $\eps=1.00$ & $\eps=0.75$ & $\eps=0.50$ & $\eps=5.00$ & $\eps=3.50$ & $\eps=2.50$ \\ \midrule
$R=1$  & 0.00 {\tiny $\pm$  0.01}  & -0.02 {\tiny $\pm$  0.00}  & -0.01 {\tiny $\pm$  0.00}  & 0.01 {\tiny $\pm$  0.02}  & -0.01 {\tiny $\pm$  0.02}  & -0.02 {\tiny $\pm$  0.00} \\
$R=2$  & 0.16 {\tiny $\pm$  0.06}  & 0.20 {\tiny $\pm$  0.02}  & 0.08 {\tiny $\pm$  0.00}  & 0.28 {\tiny $\pm$  0.12}  & 0.33 {\tiny $\pm$  0.10}  & 0.21 {\tiny $\pm$  0.02} \\
$R=5$  & 0.85 {\tiny $\pm$  0.02}  & 0.56 {\tiny $\pm$  0.02}  & 0.22 {\tiny $\pm$  0.01}  & 0.98 {\tiny $\pm$  0.00}  & 0.91 {\tiny $\pm$  0.01}  & 0.59 {\tiny $\pm$  0.02} \\
$R=10$  & 0.95 {\tiny $\pm$  0.00}  & 0.73 {\tiny $\pm$  0.02}  & 0.27 {\tiny $\pm$  0.01}  & 1.00 {\tiny $\pm$  0.00}  & 0.97 {\tiny $\pm$  0.00}  & 0.67 {\tiny $\pm$  0.01} \\
\bottomrule \\
\end{tabular}
}

\resizebox{.58\linewidth}{!}{%
        \begin{tabular}{lcccccc
}
    \toprule
    \multirow{2}{*}{\makecell[c]{\textbf{Synthetic}\\\textbf{Cauchy}}}
     & \multicolumn{3}{c}{\small $p=\infty$} & \multicolumn{3}{c}{\small $p=2$} \\
    \cmidrule(lr){2-4}\cmidrule(lr){5-7}
     & $\eps=53.00$ & $\eps=39.75$ & $\eps=26.50$ & $\eps=70.0$ & $\eps=52.5$ & $\eps=35.0$ \\ \midrule
$R=1$  & -0.03 {\tiny $\pm$  0.02}  & 0.13 {\tiny $\pm$  0.10}  & 0.00 {\tiny $\pm$  0.01}  & -0.00 {\tiny $\pm$  0.07}  & 0.16 {\tiny $\pm$  0.15}  & 0.03 {\tiny $\pm$  0.17} \\
$R=2$  & 0.05 {\tiny $\pm$  0.01}  & 0.05 {\tiny $\pm$  0.00}  & 0.04 {\tiny $\pm$  0.01}  & 0.08 {\tiny $\pm$  0.00}  & 0.08 {\tiny $\pm$  0.00}  & 0.07 {\tiny $\pm$  0.00} \\
$R=5$  & 0.03 {\tiny $\pm$  0.00}  & 0.02 {\tiny $\pm$  0.00}  & 0.02 {\tiny $\pm$  0.00}  & 0.04 {\tiny $\pm$  0.00}  & 0.04 {\tiny $\pm$  0.00}  & 0.03 {\tiny $\pm$  0.00} \\
$R=10$  & 0.02 {\tiny $\pm$  0.00}  & 0.02 {\tiny $\pm$  0.00}  & 0.02 {\tiny $\pm$  0.00}  & 0.03 {\tiny $\pm$  0.00}  & 0.03 {\tiny $\pm$  0.00}  & 0.03 {\tiny $\pm$  0.00} \\
\bottomrule \\
\end{tabular}
}

\resizebox{.58\linewidth}{!}{%
        \begin{tabular}{lcccccc
}
    \toprule
    \multirow{2}{*}{\makecell[c]{\textbf{Synthetic}\\\textbf{Holtsmark}}}
     & \multicolumn{3}{c}{\small $p=\infty$} & \multicolumn{3}{c}{\small $p=2$} \\
    \cmidrule(lr){2-4}\cmidrule(lr){5-7}
     & $\eps=0.56$ & $\eps=0.42$ & $\eps=0.28$ & $\eps=2.94$ & $\eps=2.21$ & $\eps=1.47$ \\ \midrule
$R=1$  & -0.01 {\tiny $\pm$  0.01}  & -0.01 {\tiny $\pm$  0.00}  & -0.00 {\tiny $\pm$  0.00}  & -0.05 {\tiny $\pm$  0.01}  & -0.02 {\tiny $\pm$  0.00}  & -0.01 {\tiny $\pm$  0.00} \\
$R=2$  & 0.31 {\tiny $\pm$  0.08}  & 0.28 {\tiny $\pm$  0.01}  & 0.14 {\tiny $\pm$  0.01}  & 0.52 {\tiny $\pm$  0.03}  & 0.50 {\tiny $\pm$  0.02}  & 0.28 {\tiny $\pm$  0.01} \\
$R=5$  & 0.27 {\tiny $\pm$  0.02}  & 0.24 {\tiny $\pm$  0.02}  & 0.19 {\tiny $\pm$  0.02}  & 0.28 {\tiny $\pm$  0.02}  & 0.26 {\tiny $\pm$  0.02}  & 0.23 {\tiny $\pm$  0.02} \\
$R=10$  & 0.10 {\tiny $\pm$  0.01}  & 0.08 {\tiny $\pm$  0.01}  & 0.06 {\tiny $\pm$  0.01}  & 0.10 {\tiny $\pm$  0.01}  & 0.10 {\tiny $\pm$  0.01}  & 0.08 {\tiny $\pm$  0.01} \\

\bottomrule \\
\end{tabular}
}

\resizebox{.58\linewidth}{!}{%
        \begin{tabular}{lcccccc
}
    \toprule
    \multirow{2}{*}{\makecell[l]{\textbf{MNIST}}}
     & \multicolumn{3}{c}{\small $p=\infty$} & \multicolumn{3}{c}{\small $p=2$} \\
    \cmidrule(lr){2-4}\cmidrule(lr){5-7}
     & $\eps=0.43$ & $\eps=0.33$ & $\eps=0.22$ & $\eps=2.70$ & $\eps=2.02$ & $\eps=1.35$ \\ \midrule
$R=1$  & 0.03 {\tiny $\pm$  0.00}  & 0.02 {\tiny $\pm$  0.00}  & 0.01 {\tiny $\pm$  0.00}  & 0.02 {\tiny $\pm$  0.00}  & 0.03 {\tiny $\pm$  0.00}  & 0.02 {\tiny $\pm$  0.00} \\
$R=2$  & 0.03 {\tiny $\pm$  0.00}  & 0.03 {\tiny $\pm$  0.00}  & 0.01 {\tiny $\pm$  0.00}  & 0.05 {\tiny $\pm$  0.00}  & 0.04 {\tiny $\pm$  0.01}  & 0.02 {\tiny $\pm$  0.00} \\
$R=5$  & 0.07 {\tiny $\pm$  0.00}  & 0.04 {\tiny $\pm$  0.00}  & 0.02 {\tiny $\pm$  0.00}  & 0.07 {\tiny $\pm$  0.01}  & 0.06 {\tiny $\pm$  0.00}  & 0.05 {\tiny $\pm$  0.00} \\
$R=10$  & 0.09 {\tiny $\pm$  0.00}  & 0.04 {\tiny $\pm$  0.00}  & 0.02 {\tiny $\pm$  0.00}  & 0.09 {\tiny $\pm$  0.00}  & 0.08 {\tiny $\pm$  0.01}  & 0.05 {\tiny $\pm$  0.00} \\
\bottomrule \\
\end{tabular}
}

\resizebox{.58\linewidth}{!}{%
        \begin{tabular}{lcccccc
}
    \toprule
    \multirow{2}{*}{\makecell[l]{\textbf{Fashion-}\\\textbf{MNIST}}}
     & \multicolumn{3}{c}{\small $p=\infty$} & \multicolumn{3}{c}{\small $p=2$} \\
    \cmidrule(lr){2-4}\cmidrule(lr){5-7}
     & $\eps=0.29$ & $\eps=0.22$ & $\eps=0.14$ & $\eps=3.23$ & $\eps=2.42$ & $\eps=1.62$ \\ \midrule
$R=1$  & 0.01 {\tiny $\pm$  0.01}  & 0.00 {\tiny $\pm$  0.01}  & 0.01 {\tiny $\pm$  0.00}  & -0.03 {\tiny $\pm$  0.01}  & -0.01 {\tiny $\pm$  0.01}  & 0.01 {\tiny $\pm$  0.00} \\
$R=2$  & 0.04 {\tiny $\pm$  0.00}  & 0.04 {\tiny $\pm$  0.00}  & 0.02 {\tiny $\pm$  0.00}  & 0.01 {\tiny $\pm$  0.01}  & 0.03 {\tiny $\pm$  0.00}  & 0.02 {\tiny $\pm$  0.00} \\
$R=5$  & 0.06 {\tiny $\pm$  0.01}  & 0.04 {\tiny $\pm$  0.01}  & 0.04 {\tiny $\pm$  0.00}  & 0.05 {\tiny $\pm$  0.00}  & 0.05 {\tiny $\pm$  0.00}  & 0.03 {\tiny $\pm$  0.00} \\
$R=10$  & 0.07 {\tiny $\pm$  0.00}  & 0.05 {\tiny $\pm$  0.00}  & 0.03 {\tiny $\pm$  0.00}  & 0.06 {\tiny $\pm$  0.01}  & 0.05 {\tiny $\pm$  0.01}  & 0.02 {\tiny $\pm$  0.00} \\
\bottomrule \\
\end{tabular}
}

\resizebox{.58\linewidth}{!}{%
        \begin{tabular}{lcccccc
}
    \toprule
    \multirow{2}{*}{\makecell[l]{\textbf{CIFAR}}}
     & \multicolumn{3}{c}{\small $p=\infty$} & \multicolumn{3}{c}{\small $p=2$} \\
    \cmidrule(lr){2-4}\cmidrule(lr){5-7}
     & $\eps=0.030$ & $\eps=0.023$ & $\eps=0.015$ & $\eps=0.98$ & $\eps=0.88$ & $\eps=0.78$ \\ \midrule
$R=1$  & -0.63 {\tiny $\pm$  0.19}  & 0.01 {\tiny $\pm$  0.25}  & -0.31 {\tiny $\pm$  0.09}  & -0.30 {\tiny $\pm$  0.18}  & -0.11 {\tiny $\pm$  0.03}  & -0.07 {\tiny $\pm$  0.05} \\
$R=2$  & 0.43 {\tiny $\pm$  0.03}  & 0.43 {\tiny $\pm$  0.03}  & 0.23 {\tiny $\pm$  0.09}  & 0.43 {\tiny $\pm$  0.03}  & 0.19 {\tiny $\pm$  0.04}  & -0.01 {\tiny $\pm$  0.04} \\
$R=5$  & 0.06 {\tiny $\pm$  0.01}  & 0.06 {\tiny $\pm$  0.01}  & 0.06 {\tiny $\pm$  0.01}  & 0.06 {\tiny $\pm$  0.01}  & 0.06 {\tiny $\pm$  0.01}  & -0.02 {\tiny $\pm$  0.01} \\
$R=10$  & 0.00 {\tiny $\pm$  0.00}  & 0.00 {\tiny $\pm$  0.00}  & 0.00 {\tiny $\pm$  0.00}  & 0.00 {\tiny $\pm$  0.00}  & 0.00 {\tiny $\pm$  0.00}  & -0.01 {\tiny $\pm$  0.01} \\
\bottomrule \\
\end{tabular}
}

\resizebox{.58\linewidth}{!}{%
        \begin{tabular}{lcccccc
}
    \toprule
    \multirow{2}{*}{\makecell[l]{\textbf{ImageNet}}}
     & \multicolumn{3}{c}{\small $p=\infty$} & \multicolumn{3}{c}{\small $p=2$} \\
    \cmidrule(lr){2-4}\cmidrule(lr){5-7}
     & $\eps=0.090$ & $\eps=0.068$ & $\eps=0.045$ & $\eps=4.80$ & $\eps=3.60$ & $\eps=2.40$ \\ \midrule
$R=1$  & -0.01 {\tiny $\pm$  0.01}  & 0.02 {\tiny $\pm$  0.01}  & 0.00 {\tiny $\pm$  0.01}  & 0.00 {\tiny $\pm$  0.02}  & 0.02 {\tiny $\pm$  0.01}  & 0.00 {\tiny $\pm$  0.01} \\
$R=2$  & 0.04 {\tiny $\pm$  0.03}  & 0.07 {\tiny $\pm$  0.01}  & 0.04 {\tiny $\pm$  0.01}  & 0.08 {\tiny $\pm$  0.02}  & 0.05 {\tiny $\pm$  0.02}  & 0.02 {\tiny $\pm$  0.01} \\
$R=5$  & 0.31 {\tiny $\pm$  0.03}  & 0.15 {\tiny $\pm$  0.01}  & 0.07 {\tiny $\pm$  0.02}  & 0.22 {\tiny $\pm$  0.02}  & 0.15 {\tiny $\pm$  0.01}  & 0.05 {\tiny $\pm$  0.01} \\
$R=10$  & 0.45 {\tiny $\pm$  0.04}  & 0.30 {\tiny $\pm$  0.02}  & 0.14 {\tiny $\pm$  0.01}  & 0.50 {\tiny $\pm$  0.02}  & 0.30 {\tiny $\pm$  0.03}  & 0.15 {\tiny $\pm$  0.02} \\
\bottomrule \\
\end{tabular}
}

\end{table}
}

{
\setlength{\tabcolsep}{3pt} %

\begin{table}[tbp]
    \centering

    \caption{\small \textbf{Standard accuracy gap} $acc_R^{\tstd} - acc_R^{\trob,p,\eps}$ for various choices of $p$ and $\eps$ on {seven} dataset groups. Results are averaged over $5$ runs with different random seeds. (Corresponding to \nth{2} rows of~\Cref{fig:main,fig:app-fmnist,fig:app-syn}.)
    }
    \label{tab:acc}

\resizebox{.58\linewidth}{!}{%
        \begin{tabular}{lcccccc
}
    \toprule
    \multirow{2}{*}{\makecell[c]{\textbf{Synthetic}\\\textbf{Gaussian}}}
     & \multicolumn{3}{c}{\small $p=\infty$} & \multicolumn{3}{c}{\small $p=2$} \\
    \cmidrule(lr){2-4}\cmidrule(lr){5-7}
     & $\eps=1.00$ & $\eps=0.75$ & $\eps=0.50$ & $\eps=5.00$ & $\eps=3.50$ & $\eps=2.50$ \\ \midrule
$R=1$  & 0.16 {\tiny $\pm$  0.00}  & 0.12 {\tiny $\pm$  0.00}  & 0.05 {\tiny $\pm$  0.00}  & 0.26 {\tiny $\pm$  0.00}  & 0.20 {\tiny $\pm$  0.00}  & 0.13 {\tiny $\pm$  0.00} \\
$R=2$  & 0.15 {\tiny $\pm$  0.01}  & 0.11 {\tiny $\pm$  0.00}  & 0.05 {\tiny $\pm$  0.00}  & 0.25 {\tiny $\pm$  0.01}  & 0.20 {\tiny $\pm$  0.00}  & 0.12 {\tiny $\pm$  0.00} \\
$R=5$  & 0.14 {\tiny $\pm$  0.00}  & 0.10 {\tiny $\pm$  0.00}  & 0.04 {\tiny $\pm$  0.00}  & 0.16 {\tiny $\pm$  0.00}  & 0.15 {\tiny $\pm$  0.00}  & 0.10 {\tiny $\pm$  0.00} \\
$R=10$  & 0.09 {\tiny $\pm$  0.00}  & 0.07 {\tiny $\pm$  0.00}  & 0.03 {\tiny $\pm$  0.00}  & 0.09 {\tiny $\pm$  0.00}  & 0.09 {\tiny $\pm$  0.00}  & 0.06 {\tiny $\pm$  0.00} \\
\bottomrule \\
\end{tabular}
}

\resizebox{.58\linewidth}{!}{%
        \begin{tabular}{lcccccc
}
    \toprule
    \multirow{2}{*}{\makecell[c]{\textbf{Synthetic}\\\textbf{Cauchy}}}
     & \multicolumn{3}{c}{\small $p=\infty$} & \multicolumn{3}{c}{\small $p=2$} \\
    \cmidrule(lr){2-4}\cmidrule(lr){5-7}
     & $\eps=53.00$ & $\eps=39.75$ & $\eps=26.50$ & $\eps=70.0$ & $\eps=52.5$ & $\eps=35.0$ \\ \midrule
$R=1$  & -0.00 {\tiny $\pm$  0.01}  & -0.00 {\tiny $\pm$  0.01}  & -0.02 {\tiny $\pm$  0.00}  & 0.00 {\tiny $\pm$  0.02}  & -0.01 {\tiny $\pm$  0.01}  & 0.00 {\tiny $\pm$  0.01} \\
$R=2$  & -0.01 {\tiny $\pm$  0.00}  & -0.01 {\tiny $\pm$  0.00}  & -0.01 {\tiny $\pm$  0.00}  & -0.01 {\tiny $\pm$  0.00}  & -0.01 {\tiny $\pm$  0.00}  & -0.01 {\tiny $\pm$  0.00} \\
$R=5$  & -0.01 {\tiny $\pm$  0.00}  & -0.01 {\tiny $\pm$  0.00}  & -0.00 {\tiny $\pm$  0.00}  & -0.01 {\tiny $\pm$  0.00}  & -0.01 {\tiny $\pm$  0.00}  & -0.01 {\tiny $\pm$  0.00} \\
$R=10$  & -0.01 {\tiny $\pm$  0.00}  & -0.01 {\tiny $\pm$  0.00}  & -0.01 {\tiny $\pm$  0.00}  & -0.02 {\tiny $\pm$  0.00}  & -0.02 {\tiny $\pm$  0.00}  & -0.02 {\tiny $\pm$  0.00} \\
\bottomrule \\
\end{tabular}
}

\resizebox{.58\linewidth}{!}{%
        \begin{tabular}{lcccccc
}
    \toprule
    \multirow{2}{*}{\makecell[c]{\textbf{Synthetic}\\\textbf{Holtsmark}}}
     & \multicolumn{3}{c}{\small $p=\infty$} & \multicolumn{3}{c}{\small $p=2$} \\
    \cmidrule(lr){2-4}\cmidrule(lr){5-7}
     & $\eps=0.56$ & $\eps=0.42$ & $\eps=0.28$ & $\eps=2.94$ & $\eps=2.21$ & $\eps=1.47$ \\ \midrule
$R=1$  & 0.01 {\tiny $\pm$  0.00}  & -0.00 {\tiny $\pm$  0.00}  & -0.01 {\tiny $\pm$  0.00}  & 0.05 {\tiny $\pm$  0.01}  & 0.03 {\tiny $\pm$  0.00}  & -0.00 {\tiny $\pm$  0.00} \\
$R=2$  & 0.07 {\tiny $\pm$  0.02}  & 0.06 {\tiny $\pm$  0.00}  & 0.03 {\tiny $\pm$  0.00}  & 0.13 {\tiny $\pm$  0.01}  & 0.11 {\tiny $\pm$  0.00}  & 0.06 {\tiny $\pm$  0.00} \\
$R=5$  & 0.02 {\tiny $\pm$  0.00}  & 0.02 {\tiny $\pm$  0.00}  & 0.01 {\tiny $\pm$  0.00}  & 0.02 {\tiny $\pm$  0.00}  & 0.02 {\tiny $\pm$  0.00}  & 0.02 {\tiny $\pm$  0.00} \\
$R=10$  & -0.00 {\tiny $\pm$  0.00}  & -0.00 {\tiny $\pm$  0.00}  & -0.00 {\tiny $\pm$  0.00}  & -0.00 {\tiny $\pm$  0.00}  & -0.00 {\tiny $\pm$  0.00}  & -0.00 {\tiny $\pm$  0.00} \\\bottomrule \\
\end{tabular}
}

\resizebox{.58\linewidth}{!}{%
        \begin{tabular}{lcccccc
}
    \toprule
    \multirow{2}{*}{\makecell[l]{\textbf{MNIST}}}
     & \multicolumn{3}{c}{\small $p=\infty$} & \multicolumn{3}{c}{\small $p=2$} \\
    \cmidrule(lr){2-4}\cmidrule(lr){5-7}
     & $\eps=0.43$ & $\eps=0.33$ & $\eps=0.22$ & $\eps=2.70$ & $\eps=2.02$ & $\eps=1.35$ \\ \midrule
$R=1$  & 0.02 {\tiny $\pm$  0.00}  & 0.01 {\tiny $\pm$  0.00}  & 0.01 {\tiny $\pm$  0.00}  & 0.03 {\tiny $\pm$  0.00}  & 0.02 {\tiny $\pm$  0.00}  & 0.01 {\tiny $\pm$  0.00} \\
$R=2$  & 0.02 {\tiny $\pm$  0.00}  & 0.01 {\tiny $\pm$  0.00}  & 0.00 {\tiny $\pm$  0.00}  & 0.03 {\tiny $\pm$  0.00}  & 0.02 {\tiny $\pm$  0.00}  & 0.01 {\tiny $\pm$  0.00} \\
$R=5$  & 0.01 {\tiny $\pm$  0.00}  & 0.00 {\tiny $\pm$  0.00}  & -0.00 {\tiny $\pm$  0.00}  & 0.01 {\tiny $\pm$  0.00}  & 0.01 {\tiny $\pm$  0.00}  & 0.01 {\tiny $\pm$  0.00} \\
$R=10$  & 0.01 {\tiny $\pm$  0.00}  & 0.00 {\tiny $\pm$  0.00}  & 0.00 {\tiny $\pm$  0.00}  & 0.01 {\tiny $\pm$  0.00}  & 0.00 {\tiny $\pm$  0.00}  & 0.00 {\tiny $\pm$  0.00} \\
\bottomrule \\
\end{tabular}
}

\resizebox{.58\linewidth}{!}{%
        \begin{tabular}{lcccccc
}
    \toprule
    \multirow{2}{*}{\makecell[l]{\textbf{Fashion-}\\\textbf{MNIST}}}
     & \multicolumn{3}{c}{\small $p=\infty$} & \multicolumn{3}{c}{\small $p=2$} \\
    \cmidrule(lr){2-4}\cmidrule(lr){5-7}
     & $\eps=0.29$ & $\eps=0.22$ & $\eps=0.14$ & $\eps=3.23$ & $\eps=2.42$ & $\eps=1.62$ \\ \midrule
$R=1$  & 0.04 {\tiny $\pm$  0.00}  & 0.02 {\tiny $\pm$  0.00}  & 0.01 {\tiny $\pm$  0.00}  & 0.05 {\tiny $\pm$  0.00}  & 0.04 {\tiny $\pm$  0.00}  & 0.02 {\tiny $\pm$  0.00} \\
$R=2$  & 0.02 {\tiny $\pm$  0.00}  & 0.02 {\tiny $\pm$  0.00}  & 0.01 {\tiny $\pm$  0.00}  & 0.03 {\tiny $\pm$  0.00}  & 0.02 {\tiny $\pm$  0.00}  & 0.01 {\tiny $\pm$  0.00} \\
$R=5$  & 0.02 {\tiny $\pm$  0.00}  & 0.02 {\tiny $\pm$  0.00}  & 0.01 {\tiny $\pm$  0.00}  & 0.02 {\tiny $\pm$  0.00}  & 0.02 {\tiny $\pm$  0.00}  & 0.01 {\tiny $\pm$  0.00} \\
$R=10$  & 0.01 {\tiny $\pm$  0.00}  & 0.01 {\tiny $\pm$  0.00}  & 0.00 {\tiny $\pm$  0.00}  & 0.01 {\tiny $\pm$  0.00}  & 0.00 {\tiny $\pm$  0.00}  & 0.00 {\tiny $\pm$  0.00} \\
\bottomrule \\
\end{tabular}
}

\resizebox{.58\linewidth}{!}{%
        \begin{tabular}{lcccccc
}
    \toprule
    \multirow{2}{*}{\makecell[l]{\textbf{CIFAR}}}
     & \multicolumn{3}{c}{\small $p=\infty$} & \multicolumn{3}{c}{\small $p=2$} \\
    \cmidrule(lr){2-4}\cmidrule(lr){5-7}
     & $\eps=0.030$ & $\eps=0.023$ & $\eps=0.015$ & $\eps=0.98$ & $\eps=0.88$ & $\eps=0.78$ \\ \midrule
$R=1$  & 0.21 {\tiny $\pm$  0.01}  & 0.18 {\tiny $\pm$  0.01}  & 0.12 {\tiny $\pm$  0.05}  & 0.19 {\tiny $\pm$  0.01}  & 0.14 {\tiny $\pm$  0.01}  & -0.02 {\tiny $\pm$  0.01} \\
$R=2$  & 0.04 {\tiny $\pm$  0.00}  & 0.04 {\tiny $\pm$  0.00}  & -0.00 {\tiny $\pm$  0.02}  & 0.04 {\tiny $\pm$  0.00}  & 0.02 {\tiny $\pm$  0.00}  & -0.01 {\tiny $\pm$  0.01} \\
$R=5$  & 0.00 {\tiny $\pm$  0.00}  & 0.00 {\tiny $\pm$  0.00}  & 0.00 {\tiny $\pm$  0.00}  & 0.00 {\tiny $\pm$  0.00}  & 0.00 {\tiny $\pm$  0.00}  & -0.00 {\tiny $\pm$  0.00} \\
$R=10$  & 0.00 {\tiny $\pm$  0.00}  & 0.00 {\tiny $\pm$  0.00}  & 0.00 {\tiny $\pm$  0.00}  & 0.00 {\tiny $\pm$  0.00}  & 0.00 {\tiny $\pm$  0.00}  & -0.00 {\tiny $\pm$  0.00} \\
\bottomrule \\
\end{tabular}
}

\resizebox{.58\linewidth}{!}{%
        \begin{tabular}{lcccccc
}
    \toprule
    \multirow{2}{*}{\makecell[l]{\textbf{ImageNet}}}
     & \multicolumn{3}{c}{\small $p=\infty$} & \multicolumn{3}{c}{\small $p=2$} \\
    \cmidrule(lr){2-4}\cmidrule(lr){5-7}
     & $\eps=0.090$ & $\eps=0.068$ & $\eps=0.045$ & $\eps=4.80$ & $\eps=3.60$ & $\eps=2.40$ \\ \midrule
$R=1$  & 0.05 {\tiny $\pm$  0.00}  & 0.04 {\tiny $\pm$  0.00}  & 0.02 {\tiny $\pm$  0.00}  & 0.06 {\tiny $\pm$  0.00}  & 0.03 {\tiny $\pm$  0.00}  & 0.02 {\tiny $\pm$  0.00} \\
$R=2$  & 0.07 {\tiny $\pm$  0.00}  & 0.05 {\tiny $\pm$  0.00}  & 0.03 {\tiny $\pm$  0.00}  & 0.06 {\tiny $\pm$  0.00}  & 0.04 {\tiny $\pm$  0.00}  & 0.03 {\tiny $\pm$  0.00} \\
$R=5$  & 0.06 {\tiny $\pm$  0.00}  & 0.04 {\tiny $\pm$  0.00}  & 0.02 {\tiny $\pm$  0.00}  & 0.05 {\tiny $\pm$  0.00}  & 0.04 {\tiny $\pm$  0.00}  & 0.03 {\tiny $\pm$  0.00} \\
$R=10$  & 0.05 {\tiny $\pm$  0.00}  & 0.04 {\tiny $\pm$  0.00}  & 0.03 {\tiny $\pm$  0.00}  & 0.05 {\tiny $\pm$  0.00}  & 0.04 {\tiny $\pm$  0.00}  & 0.03 {\tiny $\pm$  0.00} \\
\bottomrule \\
\end{tabular}
}

\end{table}
}

{
\setlength{\tabcolsep}{3pt} %

\begin{table}[tbp]
    \centering

    \caption{\small \textbf{Per class accuracy} for the standard and robust classifiers on synthetic datasets. 
    We denote the majority class as ``class $-$'' and the minority class as ``class $+$'' following~\Cref{sec:pre}.
    ``std'' refers to the standard classifier and ``rob'' refers to the robust classifier with the specified $p$ and $\varepsilon$.
    The presented results are averaged over $5$ runs.
    }
    \label{tab:acc-per-class-syn}

\vspace{1mm}

\resizebox{.9\textwidth}{!}{%
        \begin{tabular}{lcccccccc
}
    \toprule
    \multirow{2}{*}{\makecell[c]{\textbf{Synthetic Gaussian}}}
     & \multicolumn{2}{c}{\small $R=1$} 
     & \multicolumn{2}{c}{\small $R=2$} 
     & \multicolumn{2}{c}{\small $R=5$} 
     & \multicolumn{2}{c}{\small $R=10$} 
     \\
    \cmidrule(lr){2-3}\cmidrule(lr){4-5}\cmidrule(lr){6-7}\cmidrule(lr){8-9}
     & class $-$ & class $+$
     & class $-$ & class $+$
     & class $-$ & class $+$
     & class $-$ & class $+$\\ \midrule
std & 1.00 {\tiny $\pm$  0.00}& 1.00 {\tiny $\pm$  0.00} & 1.00 {\tiny $\pm$  0.00}& 1.00 {\tiny $\pm$  0.00} & 1.00 {\tiny $\pm$  0.00}& 1.00 {\tiny $\pm$  0.00} & 1.00 {\tiny $\pm$  0.00}& 1.00 {\tiny $\pm$  0.00} \\\midrule
rob ($p=2$, $\varepsilon=5.00)$  & 0.85  {\tiny $\pm$  0.01} & 0.84  {\tiny $\pm$  0.01}  & 0.89  {\tiny $\pm$  0.02} & 0.73  {\tiny $\pm$  0.05}  & 0.99  {\tiny $\pm$  0.00} & 0.15  {\tiny $\pm$  0.02}  & 1.00  {\tiny $\pm$  0.00} & 0.05  {\tiny $\pm$  0.00} \\
rob ($p=2$, $\varepsilon=3.75)$  & 0.87  {\tiny $\pm$  0.00} & 0.89  {\tiny $\pm$  0.00}  & 0.95  {\tiny $\pm$  0.00} & 0.75  {\tiny $\pm$  0.01}  & 0.99  {\tiny $\pm$  0.00} & 0.43  {\tiny $\pm$  0.02}  & 1.00  {\tiny $\pm$  0.00} & 0.26  {\tiny $\pm$  0.01} \\
rob ($p=2$, $\varepsilon=2.50)$  & 0.94  {\tiny $\pm$  0.00} & 0.95  {\tiny $\pm$  0.00}  & 0.98  {\tiny $\pm$  0.00} & 0.89  {\tiny $\pm$  0.00}  & 0.99  {\tiny $\pm$  0.00} & 0.77  {\tiny $\pm$  0.01}  & 0.99  {\tiny $\pm$  0.00} & 0.72  {\tiny $\pm$  0.01} \\
\midrule
rob ($p=\infty$, $\varepsilon=1.00)$  & 0.74  {\tiny $\pm$  0.01} & 0.73  {\tiny $\pm$  0.01}  & 0.84  {\tiny $\pm$  0.03} & 0.56  {\tiny $\pm$  0.09}  & 1.00  {\tiny $\pm$  0.00} & 0.02  {\tiny $\pm$  0.00}  & 1.00  {\tiny $\pm$  0.00} & 0.00  {\tiny $\pm$  0.00} \\
rob ($p=\infty$, $\varepsilon=0.75)$  & 0.80  {\tiny $\pm$  0.01} & 0.80  {\tiny $\pm$  0.01}  & 0.90  {\tiny $\pm$  0.03} & 0.58  {\tiny $\pm$  0.07}  & 1.00  {\tiny $\pm$  0.00} & 0.09  {\tiny $\pm$  0.01}  & 1.00  {\tiny $\pm$  0.00} & 0.02  {\tiny $\pm$  0.00} \\
rob ($p=\infty$, $\varepsilon=0.50)$  & 0.86  {\tiny $\pm$  0.00} & 0.88  {\tiny $\pm$  0.00}  & 0.95  {\tiny $\pm$  0.01} & 0.74  {\tiny $\pm$  0.01}  & 0.99  {\tiny $\pm$  0.00} & 0.40  {\tiny $\pm$  0.02}  & 0.99  {\tiny $\pm$  0.00} & 0.32  {\tiny $\pm$  0.01} \\
\bottomrule
\\
\end{tabular}
}

\resizebox{.9\textwidth}{!}{%
        \begin{tabular}{lcccccccc
}
    \toprule
    \multirow{2}{*}{\makecell[c]{\textbf{Synthetic Cauchy}}}
     & \multicolumn{2}{c}{\small $R=1$} 
     & \multicolumn{2}{c}{\small $R=2$} 
     & \multicolumn{2}{c}{\small $R=5$} 
     & \multicolumn{2}{c}{\small $R=10$} 
     \\
    \cmidrule(lr){2-3}\cmidrule(lr){4-5}\cmidrule(lr){6-7}\cmidrule(lr){8-9}
     & class $-$ & class $+$
     & class $-$ & class $+$
     & class $-$ & class $+$
     & class $-$ & class $+$\\ \midrule
std & 0.50 {\tiny $\pm$  0.00}& 0.52 {\tiny $\pm$  0.01} & 0.95 {\tiny $\pm$  0.00}& 0.04 {\tiny $\pm$  0.00} & 0.98 {\tiny $\pm$  0.00}& 0.02 {\tiny $\pm$  0.00} & 0.98 {\tiny $\pm$  0.00}& 0.01 {\tiny $\pm$  0.00} \\\midrule
rob ($p=2$, $\varepsilon=70.00)$  & 0.48  {\tiny $\pm$  0.01} & 0.54  {\tiny $\pm$  0.01}  & 0.98  {\tiny $\pm$  0.00} & 0.02  {\tiny $\pm$  0.00}  & 0.99  {\tiny $\pm$  0.00} & 0.00  {\tiny $\pm$  0.00}  & 0.99  {\tiny $\pm$  0.00} & 0.00  {\tiny $\pm$  0.00} \\
rob ($p=2$, $\varepsilon=52.50)$  & 0.57  {\tiny $\pm$  0.05} & 0.46  {\tiny $\pm$  0.06}  & 0.98  {\tiny $\pm$  0.00} & 0.02  {\tiny $\pm$  0.00}  & 0.99  {\tiny $\pm$  0.00} & 0.01  {\tiny $\pm$  0.00}  & 0.99  {\tiny $\pm$  0.00} & 0.01  {\tiny $\pm$  0.00} \\
rob ($p=2$, $\varepsilon=35.00)$  & 0.51  {\tiny $\pm$  0.00} & 0.54  {\tiny $\pm$  0.01}  & 0.98  {\tiny $\pm$  0.00} & 0.02  {\tiny $\pm$  0.00}  & 0.99  {\tiny $\pm$  0.00} & 0.01  {\tiny $\pm$  0.00}  & 0.99  {\tiny $\pm$  0.00} & 0.01  {\tiny $\pm$  0.00} \\
\midrule
rob ($p=\infty$, $\varepsilon=53.00)$  & 0.49  {\tiny $\pm$  0.03} & 0.52  {\tiny $\pm$  0.04}  & 1.00  {\tiny $\pm$  0.00} & 0.01  {\tiny $\pm$  0.00}  & 1.00  {\tiny $\pm$  0.00} & 0.00  {\tiny $\pm$  0.00}  & 1.00  {\tiny $\pm$  0.00} & 0.00  {\tiny $\pm$  0.00} \\
rob ($p=\infty$, $\varepsilon=39.75)$  & 0.58  {\tiny $\pm$  0.08} & 0.45  {\tiny $\pm$  0.09}  & 1.00  {\tiny $\pm$  0.00} & 0.00  {\tiny $\pm$  0.00}  & 1.00  {\tiny $\pm$  0.00} & 0.00  {\tiny $\pm$  0.00}  & 1.00  {\tiny $\pm$  0.00} & 0.00  {\tiny $\pm$  0.00} \\
rob ($p=\infty$, $\varepsilon=26.50)$  & 0.51  {\tiny $\pm$  0.09} & 0.50  {\tiny $\pm$  0.09}  & 0.99  {\tiny $\pm$  0.00} & 0.01  {\tiny $\pm$  0.00}  & 1.00  {\tiny $\pm$  0.00} & 0.00  {\tiny $\pm$  0.00}  & 1.00  {\tiny $\pm$  0.00} & 0.00  {\tiny $\pm$  0.00} \\
\bottomrule
\\
\end{tabular}
}

\resizebox{.9\textwidth}{!}{%
        \begin{tabular}{lcccccccc
}
    \toprule
    \multirow{2}{*}{\makecell[c]{\textbf{Synthetic Holtsmark}}}
     & \multicolumn{2}{c}{\small $R=1$} 
     & \multicolumn{2}{c}{\small $R=2$} 
     & \multicolumn{2}{c}{\small $R=5$} 
     & \multicolumn{2}{c}{\small $R=10$} 
     \\
    \cmidrule(lr){2-3}\cmidrule(lr){4-5}\cmidrule(lr){6-7}\cmidrule(lr){8-9}
     & class $-$ & class $+$
     & class $-$ & class $+$
     & class $-$ & class $+$
     & class $-$ & class $+$\\ \midrule
std & 0.81 {\tiny $\pm$  0.00}& 0.82 {\tiny $\pm$  0.00} & 0.94 {\tiny $\pm$  0.00}& 0.60 {\tiny $\pm$  0.00} & 0.97 {\tiny $\pm$  0.00}& 0.27 {\tiny $\pm$  0.02} & 0.99 {\tiny $\pm$  0.00}& 0.10 {\tiny $\pm$  0.01} \\\midrule
rob ($p=2$, $\varepsilon=2.94)$  & 0.80  {\tiny $\pm$  0.01} & 0.82  {\tiny $\pm$  0.00}  & 0.97  {\tiny $\pm$  0.01} & 0.32  {\tiny $\pm$  0.07}  & 1.00  {\tiny $\pm$  0.00} & 0.02  {\tiny $\pm$  0.00}  & 1.00  {\tiny $\pm$  0.00} & 0.01  {\tiny $\pm$  0.00} \\
rob ($p=2$, $\varepsilon=2.21)$  & 0.81  {\tiny $\pm$  0.00} & 0.83  {\tiny $\pm$  0.00}  & 0.97  {\tiny $\pm$  0.00} & 0.35  {\tiny $\pm$  0.01}  & 0.99  {\tiny $\pm$  0.00} & 0.05  {\tiny $\pm$  0.01}  & 1.00  {\tiny $\pm$  0.00} & 0.03  {\tiny $\pm$  0.00} \\
rob ($p=2$, $\varepsilon=1.47)$  & 0.81  {\tiny $\pm$  0.00} & 0.83  {\tiny $\pm$  0.00}  & 0.95  {\tiny $\pm$  0.00} & 0.48  {\tiny $\pm$  0.01}  & 0.99  {\tiny $\pm$  0.00} & 0.10  {\tiny $\pm$  0.01}  & 0.99  {\tiny $\pm$  0.00} & 0.04  {\tiny $\pm$  0.00} \\
\midrule
rob ($p=\infty$, $\varepsilon=0.56)$  & 0.74  {\tiny $\pm$  0.01} & 0.79  {\tiny $\pm$  0.00}  & 0.98  {\tiny $\pm$  0.00} & 0.13  {\tiny $\pm$  0.03}  & 1.00  {\tiny $\pm$  0.00} & 0.02  {\tiny $\pm$  0.00}  & 1.00  {\tiny $\pm$  0.00} & 0.01  {\tiny $\pm$  0.00} \\
rob ($p=\infty$, $\varepsilon=0.42)$  & 0.77  {\tiny $\pm$  0.01} & 0.80  {\tiny $\pm$  0.00}  & 0.98  {\tiny $\pm$  0.00} & 0.14  {\tiny $\pm$  0.01}  & 0.99  {\tiny $\pm$  0.00} & 0.03  {\tiny $\pm$  0.01}  & 1.00  {\tiny $\pm$  0.00} & 0.01  {\tiny $\pm$  0.00} \\
rob ($p=\infty$, $\varepsilon=0.28)$  & 0.81  {\tiny $\pm$  0.00} & 0.82  {\tiny $\pm$  0.00}  & 0.97  {\tiny $\pm$  0.00} & 0.35  {\tiny $\pm$  0.00}  & 0.99  {\tiny $\pm$  0.00} & 0.05  {\tiny $\pm$  0.01}  & 0.99  {\tiny $\pm$  0.00} & 0.02  {\tiny $\pm$  0.00} \\
\bottomrule
\\
\end{tabular}
}

\end{table}

\begin{table}[tbp]
    \centering

    \caption{\small \textbf{Per class accuracy} for the standard and robust classifiers on real-world datasets. 
    We denote the majority class as ``class $-$'' and the minority class as ``class $+$'' following~\Cref{sec:pre}.
    ``std'' refers to the standard classifier and ``rob'' refers to the robust classifier with the specified $p$ and $\varepsilon$.
    The presented results are averaged over $5$ runs.
    }
    \label{tab:acc-per-class-real}

\vspace{1mm}

\resizebox{.9\textwidth}{!}{%
        \begin{tabular}{lcccccccc
}
    \toprule
    \multirow{2}{*}{\makecell[c]{\textbf{MNIST}}}
     & \multicolumn{2}{c}{\small $R=1$} 
     & \multicolumn{2}{c}{\small $R=2$} 
     & \multicolumn{2}{c}{\small $R=5$} 
     & \multicolumn{2}{c}{\small $R=10$} 
     \\
    \cmidrule(lr){2-3}\cmidrule(lr){4-5}\cmidrule(lr){6-7}\cmidrule(lr){8-9}
     & class $-$ & class $+$
     & class $-$ & class $+$
     & class $-$ & class $+$
     & class $-$ & class $+$\\ \midrule
std & 1.00 {\tiny $\pm$  0.00}& 0.99 {\tiny $\pm$  0.00} & 1.00 {\tiny $\pm$  0.00}& 0.99 {\tiny $\pm$  0.00} & 1.00 {\tiny $\pm$  0.00}& 0.98 {\tiny $\pm$  0.00} & 1.00 {\tiny $\pm$  0.00}& 0.97 {\tiny $\pm$  0.00} \\\midrule
rob ($p=2$, $\varepsilon=2.70)$  & 0.99  {\tiny $\pm$  0.00} & 0.96  {\tiny $\pm$  0.00}  & 1.00  {\tiny $\pm$  0.00} & 0.95  {\tiny $\pm$  0.00}  & 1.00  {\tiny $\pm$  0.00} & 0.91  {\tiny $\pm$  0.00}  & 1.00  {\tiny $\pm$  0.00} & 0.88  {\tiny $\pm$  0.00} \\
rob ($p=2$, $\varepsilon=2.02)$  & 1.00  {\tiny $\pm$  0.00} & 0.97  {\tiny $\pm$  0.00}  & 1.00  {\tiny $\pm$  0.00} & 0.96  {\tiny $\pm$  0.00}  & 1.00  {\tiny $\pm$  0.00} & 0.94  {\tiny $\pm$  0.00}  & 1.00  {\tiny $\pm$  0.00} & 0.93  {\tiny $\pm$  0.00} \\
rob ($p=2$, $\varepsilon=1.35)$  & 1.00  {\tiny $\pm$  0.00} & 0.98  {\tiny $\pm$  0.00}  & 1.00  {\tiny $\pm$  0.00} & 0.97  {\tiny $\pm$  0.00}  & 1.00  {\tiny $\pm$  0.00} & 0.96  {\tiny $\pm$  0.00}  & 1.00  {\tiny $\pm$  0.00} & 0.95  {\tiny $\pm$  0.00} \\
\midrule
rob ($p=\infty$, $\varepsilon=0.43)$  & 0.98  {\tiny $\pm$  0.00} & 0.96  {\tiny $\pm$  0.00}  & 0.99  {\tiny $\pm$  0.00} & 0.93  {\tiny $\pm$  0.00}  & 0.99  {\tiny $\pm$  0.00} & 0.91  {\tiny $\pm$  0.01}  & 1.00  {\tiny $\pm$  0.00} & 0.87  {\tiny $\pm$  0.00} \\
rob ($p=\infty$, $\varepsilon=0.33)$  & 0.99  {\tiny $\pm$  0.00} & 0.96  {\tiny $\pm$  0.00}  & 1.00  {\tiny $\pm$  0.00} & 0.94  {\tiny $\pm$  0.01}  & 1.00  {\tiny $\pm$  0.00} & 0.92  {\tiny $\pm$  0.00}  & 1.00  {\tiny $\pm$  0.00} & 0.88  {\tiny $\pm$  0.01} \\
rob ($p=\infty$, $\varepsilon=0.22)$  & 0.99  {\tiny $\pm$  0.00} & 0.97  {\tiny $\pm$  0.00}  & 0.99  {\tiny $\pm$  0.00} & 0.96  {\tiny $\pm$  0.00}  & 1.00  {\tiny $\pm$  0.00} & 0.93  {\tiny $\pm$  0.00}  & 1.00  {\tiny $\pm$  0.00} & 0.92  {\tiny $\pm$  0.00} \\
\bottomrule
\\
\end{tabular}
}

\resizebox{.9\textwidth}{!}{%
        \begin{tabular}{lcccccccc
}
    \toprule
    \multirow{2}{*}{\makecell[c]{\textbf{Fashion-MNIST}}}
     & \multicolumn{2}{c}{\small $R=1$} 
     & \multicolumn{2}{c}{\small $R=2$} 
     & \multicolumn{2}{c}{\small $R=5$} 
     & \multicolumn{2}{c}{\small $R=10$} 
     \\
    \cmidrule(lr){2-3}\cmidrule(lr){4-5}\cmidrule(lr){6-7}\cmidrule(lr){8-9}
     & class $-$ & class $+$
     & class $-$ & class $+$
     & class $-$ & class $+$
     & class $-$ & class $+$\\ \midrule
std & 0.98 {\tiny $\pm$  0.00}& 0.99 {\tiny $\pm$  0.00} & 0.98 {\tiny $\pm$  0.00}& 0.99 {\tiny $\pm$  0.00} & 0.97 {\tiny $\pm$  0.00}& 1.00 {\tiny $\pm$  0.00} & 0.95 {\tiny $\pm$  0.00}& 1.00 {\tiny $\pm$  0.00} \\\midrule
rob ($p=2$, $\varepsilon=3.23)$  & 0.94  {\tiny $\pm$  0.00} & 0.96  {\tiny $\pm$  0.00}  & 0.92  {\tiny $\pm$  0.00} & 0.98  {\tiny $\pm$  0.00}  & 0.90  {\tiny $\pm$  0.00} & 0.99  {\tiny $\pm$  0.00}  & 0.87  {\tiny $\pm$  0.00} & 1.00  {\tiny $\pm$  0.00} \\
rob ($p=2$, $\varepsilon=2.42)$  & 0.96  {\tiny $\pm$  0.00} & 0.97  {\tiny $\pm$  0.00}  & 0.93  {\tiny $\pm$  0.00} & 0.99  {\tiny $\pm$  0.00}  & 0.92  {\tiny $\pm$  0.00} & 0.99  {\tiny $\pm$  0.00}  & 0.89  {\tiny $\pm$  0.00} & 1.00  {\tiny $\pm$  0.00} \\
rob ($p=2$, $\varepsilon=1.62)$  & 0.97  {\tiny $\pm$  0.00} & 0.99  {\tiny $\pm$  0.00}  & 0.95  {\tiny $\pm$  0.00} & 0.99  {\tiny $\pm$  0.00}  & 0.93  {\tiny $\pm$  0.00} & 1.00  {\tiny $\pm$  0.00}  & 0.92  {\tiny $\pm$  0.00} & 1.00  {\tiny $\pm$  0.00} \\
\midrule
rob ($p=\infty$, $\varepsilon=0.29)$  & 0.95  {\tiny $\pm$  0.00} & 0.93  {\tiny $\pm$  0.01}  & 0.94  {\tiny $\pm$  0.01} & 0.97  {\tiny $\pm$  0.00}  & 0.91  {\tiny $\pm$  0.00} & 0.99  {\tiny $\pm$  0.00}  & 0.89  {\tiny $\pm$  0.01} & 1.00  {\tiny $\pm$  0.00} \\
rob ($p=\infty$, $\varepsilon=0.22)$  & 0.95  {\tiny $\pm$  0.01} & 0.94  {\tiny $\pm$  0.01}  & 0.93  {\tiny $\pm$  0.00} & 0.98  {\tiny $\pm$  0.00}  & 0.91  {\tiny $\pm$  0.00} & 0.99  {\tiny $\pm$  0.00}  & 0.89  {\tiny $\pm$  0.00} & 1.00  {\tiny $\pm$  0.00} \\
rob ($p=\infty$, $\varepsilon=0.14)$  & 0.96  {\tiny $\pm$  0.00} & 0.98  {\tiny $\pm$  0.00}  & 0.95  {\tiny $\pm$  0.00} & 0.99  {\tiny $\pm$  0.00}  & 0.94  {\tiny $\pm$  0.00} & 1.00  {\tiny $\pm$  0.00}  & 0.93  {\tiny $\pm$  0.00} & 1.00  {\tiny $\pm$  0.00} \\
\bottomrule
\\
\end{tabular}
}

\resizebox{.9\textwidth}{!}{%
        \begin{tabular}{lcccccccc
}
    \toprule
    \multirow{2}{*}{\makecell[c]{\textbf{CIFAR}}}
     & \multicolumn{2}{c}{\small $R=1$} 
     & \multicolumn{2}{c}{\small $R=2$} 
     & \multicolumn{2}{c}{\small $R=5$} 
     & \multicolumn{2}{c}{\small $R=10$} 
     \\
    \cmidrule(lr){2-3}\cmidrule(lr){4-5}\cmidrule(lr){6-7}\cmidrule(lr){8-9}
     & class $-$ & class $+$
     & class $-$ & class $+$
     & class $-$ & class $+$
     & class $-$ & class $+$\\ \midrule
std & 0.71 {\tiny $\pm$  0.01}& 0.74 {\tiny $\pm$  0.01} & 0.90 {\tiny $\pm$  0.01}& 0.33 {\tiny $\pm$  0.02} & 0.99 {\tiny $\pm$  0.00}& 0.05 {\tiny $\pm$  0.01} & 1.00 {\tiny $\pm$  0.00}& 0.00 {\tiny $\pm$  0.00} \\\midrule
rob ($p=2$, $\varepsilon=0.98)$  & 0.19  {\tiny $\pm$  0.10} & 0.85  {\tiny $\pm$  0.08}  & 1.00  {\tiny $\pm$  0.00} & 0.00  {\tiny $\pm$  0.00}  & 1.00  {\tiny $\pm$  0.00} & 0.00  {\tiny $\pm$  0.00}  & 1.00  {\tiny $\pm$  0.00} & 0.00  {\tiny $\pm$  0.00} \\
rob ($p=2$, $\varepsilon=0.88)$  & 0.53  {\tiny $\pm$  0.12} & 0.54  {\tiny $\pm$  0.14}  & 1.00  {\tiny $\pm$  0.00} & 0.00  {\tiny $\pm$  0.00}  & 1.00  {\tiny $\pm$  0.00} & 0.00  {\tiny $\pm$  0.00}  & 1.00  {\tiny $\pm$  0.00} & 0.00  {\tiny $\pm$  0.00} \\
rob ($p=2$, $\varepsilon=0.78)$  & 0.44  {\tiny $\pm$  0.08} & 0.77  {\tiny $\pm$  0.04}  & 0.98  {\tiny $\pm$  0.01} & 0.18  {\tiny $\pm$  0.07}  & 1.00  {\tiny $\pm$  0.00} & 0.00  {\tiny $\pm$  0.00}  & 1.00  {\tiny $\pm$  0.00} & 0.00  {\tiny $\pm$  0.00} \\
\midrule
rob ($p=\infty$, $\varepsilon=0.030)$  & 0.37  {\tiny $\pm$  0.10} & 0.70  {\tiny $\pm$  0.09}  & 1.00  {\tiny $\pm$  0.00} & 0.00  {\tiny $\pm$  0.00}  & 1.00  {\tiny $\pm$  0.00} & 0.00  {\tiny $\pm$  0.00}  & 1.00  {\tiny $\pm$  0.00} & 0.00  {\tiny $\pm$  0.00} \\
rob ($p=\infty$, $\varepsilon=0.023)$  & 0.51  {\tiny $\pm$  0.02} & 0.65  {\tiny $\pm$  0.01}  & 0.94  {\tiny $\pm$  0.01} & 0.18  {\tiny $\pm$  0.02}  & 1.00  {\tiny $\pm$  0.00} & 0.00  {\tiny $\pm$  0.00}  & 1.00  {\tiny $\pm$  0.00} & 0.00  {\tiny $\pm$  0.00} \\
rob ($p=\infty$, $\varepsilon=0.015)$  & 0.69  {\tiny $\pm$  0.03} & 0.79  {\tiny $\pm$  0.02}  & 0.91  {\tiny $\pm$  0.01} & 0.36  {\tiny $\pm$  0.04}  & 0.99  {\tiny $\pm$  0.00} & 0.07  {\tiny $\pm$  0.01}  & 1.00  {\tiny $\pm$  0.00} & 0.01  {\tiny $\pm$  0.01} \\
\bottomrule
\\
\end{tabular}
}

{
\resizebox{.9\textwidth}{!}{%
        \begin{tabular}{lcccccccc
}
    \toprule
    \multirow{2}{*}{\makecell[c]{\textbf{ImageNet}}}
     & \multicolumn{2}{c}{\small $R=1$} 
     & \multicolumn{2}{c}{\small $R=2$} 
     & \multicolumn{2}{c}{\small $R=5$} 
     & \multicolumn{2}{c}{\small $R=10$} 
     \\
    \cmidrule(lr){2-3}\cmidrule(lr){4-5}\cmidrule(lr){6-7}\cmidrule(lr){8-9}
     & class $-$ & class $+$
     & class $-$ & class $+$
     & class $-$ & class $+$
     & class $-$ & class $+$\\ \midrule
std & 0.99 {\tiny $\pm$  0.00}& 0.99 {\tiny $\pm$  0.00} & 1.00 {\tiny $\pm$  0.00}& 0.98 {\tiny $\pm$  0.00} & 1.00 {\tiny $\pm$  0.00}& 0.97 {\tiny $\pm$  0.00} & 1.00 {\tiny $\pm$  0.00}& 0.95 {\tiny $\pm$  0.00} \\\midrule
rob ($p=2$, $\varepsilon=4.80)$  & 0.93  {\tiny $\pm$  0.01} & 0.94  {\tiny $\pm$  0.01}  & 0.95  {\tiny $\pm$  0.01} & 0.89  {\tiny $\pm$  0.02}  & 0.99  {\tiny $\pm$  0.00} & 0.66  {\tiny $\pm$  0.03}  & 1.00  {\tiny $\pm$  0.00} & 0.49  {\tiny $\pm$  0.04} \\
rob ($p=2$, $\varepsilon=3.60)$  & 0.96  {\tiny $\pm$  0.00} & 0.93  {\tiny $\pm$  0.01}  & 0.98  {\tiny $\pm$  0.00} & 0.89  {\tiny $\pm$  0.01}  & 0.99  {\tiny $\pm$  0.00} & 0.81  {\tiny $\pm$  0.02}  & 0.99  {\tiny $\pm$  0.00} & 0.64  {\tiny $\pm$  0.03} \\
rob ($p=2$, $\varepsilon=2.40)$  & 0.97  {\tiny $\pm$  0.00} & 0.96  {\tiny $\pm$  0.00}  & 0.98  {\tiny $\pm$  0.00} & 0.92  {\tiny $\pm$  0.01}  & 0.99  {\tiny $\pm$  0.00} & 0.89  {\tiny $\pm$  0.01}  & 0.99  {\tiny $\pm$  0.00} & 0.79  {\tiny $\pm$  0.01} \\
\midrule
rob ($p=\infty$, $\varepsilon=0.090)$  & 0.93  {\tiny $\pm$  0.01} & 0.92  {\tiny $\pm$  0.01}  & 0.96  {\tiny $\pm$  0.01} & 0.87  {\tiny $\pm$  0.01}  & 0.99  {\tiny $\pm$  0.00} & 0.74  {\tiny $\pm$  0.02}  & 1.00  {\tiny $\pm$  0.00} & 0.45  {\tiny $\pm$  0.02} \\
rob ($p=\infty$, $\varepsilon=0.068)$  & 0.97  {\tiny $\pm$  0.00} & 0.94  {\tiny $\pm$  0.00}  & 0.97  {\tiny $\pm$  0.00} & 0.91  {\tiny $\pm$  0.01}  & 0.99  {\tiny $\pm$  0.00} & 0.81  {\tiny $\pm$  0.01}  & 0.99  {\tiny $\pm$  0.00} & 0.65  {\tiny $\pm$  0.03} \\
rob ($p=\infty$, $\varepsilon=0.045)$  & 0.97  {\tiny $\pm$  0.00} & 0.96  {\tiny $\pm$  0.00}  & 0.98  {\tiny $\pm$  0.00} & 0.94  {\tiny $\pm$  0.00}  & 0.98  {\tiny $\pm$  0.00} & 0.91  {\tiny $\pm$  0.01}  & 0.99  {\tiny $\pm$  0.00} & 0.79  {\tiny $\pm$  0.02} \\
\bottomrule
\\
\end{tabular}
}

}

\end{table}

}

{
\setlength{\tabcolsep}{3pt} %

\begin{table}[tbp]
    \centering

    \caption{\small \textbf{Overall accuracy} for the standard and robust classifiers on synthetic datasets. 
    ``std'' refers to the standard classifier and ``rob'' refers to the robust classifier with the specified $p$ and $\varepsilon$.
    The presented results are averaged over $5$ runs.
    }
    \label{tab:acc-overall-1}

\vspace{1mm}

\resizebox{.6\textwidth}{!}{%
        \begin{tabular}{lcccc
}
    \toprule
    \multirow{1}{*}{\makecell[c]{\textbf{Synthetic Gaussian}}}
     & \multicolumn{1}{c}{\small $R=1$} 
     & \multicolumn{1}{c}{\small $R=2$} 
     & \multicolumn{1}{c}{\small $R=5$} 
     & \multicolumn{1}{c}{\small $R=10$} 
     \\\midrule
std & 1.00 {\tiny $\pm$  0.00} & 1.00 {\tiny $\pm$  0.00} & 1.00 {\tiny $\pm$  0.00} & 1.00 {\tiny $\pm$  0.00} \\\midrule
rob ($p=2$, $\varepsilon=5.00)$  & 0.87  {\tiny $\pm$  0.01}  & 0.86  {\tiny $\pm$  0.01}  & 0.87  {\tiny $\pm$  0.01}  & 0.86  {\tiny $\pm$  0.01} \\
rob ($p=2$, $\varepsilon=3.75)$  & 0.90  {\tiny $\pm$  0.01}  & 0.90  {\tiny $\pm$  0.01}  & 0.90  {\tiny $\pm$  0.01}  & 0.90  {\tiny $\pm$  0.01} \\
rob ($p=2$, $\varepsilon=2.50)$  & 0.96  {\tiny $\pm$  0.01}  & 0.96  {\tiny $\pm$  0.00}  & 0.96  {\tiny $\pm$  0.00}  & 0.95  {\tiny $\pm$  0.01} \\
\midrule
rob ($p=\infty$, $\varepsilon=1.00)$  & 0.82  {\tiny $\pm$  0.03}  & 0.81  {\tiny $\pm$  0.03}  & 0.81  {\tiny $\pm$  0.03}  & 0.80  {\tiny $\pm$  0.03} \\
rob ($p=\infty$, $\varepsilon=0.75)$  & 0.84  {\tiny $\pm$  0.02}  & 0.84  {\tiny $\pm$  0.02}  & 0.84  {\tiny $\pm$  0.02}  & 0.85  {\tiny $\pm$  0.02} \\
rob ($p=\infty$, $\varepsilon=0.50)$  & 0.90  {\tiny $\pm$  0.01}  & 0.89  {\tiny $\pm$  0.01}  & 0.89  {\tiny $\pm$  0.01}  & 0.90  {\tiny $\pm$  0.01} \\
\bottomrule
\\
\end{tabular}
}

\resizebox{.6\textwidth}{!}{%
        \begin{tabular}{lcccc
}
    \toprule
    \multirow{1}{*}{\makecell[c]{\textbf{Synthetic Cauchy}}}
     & \multicolumn{1}{c}{\small $R=1$} 
     & \multicolumn{1}{c}{\small $R=2$} 
     & \multicolumn{1}{c}{\small $R=5$} 
     & \multicolumn{1}{c}{\small $R=10$} 
     \\\midrule
std & 1.00 {\tiny $\pm$  0.00} & 1.00 {\tiny $\pm$  0.00} & 1.00 {\tiny $\pm$  0.00} & 1.00 {\tiny $\pm$  0.00} \\\midrule
rob ($p=2$, $\varepsilon=5.00)$  & 0.84  {\tiny $\pm$  0.00}  & 0.85  {\tiny $\pm$  0.01}  & 0.86  {\tiny $\pm$  0.00}  & 0.91  {\tiny $\pm$  0.00} \\
rob ($p=2$, $\varepsilon=3.75)$  & 0.88  {\tiny $\pm$  0.00}  & 0.89  {\tiny $\pm$  0.00}  & 0.90  {\tiny $\pm$  0.00}  & 0.93  {\tiny $\pm$  0.00} \\
rob ($p=2$, $\varepsilon=2.50)$  & 0.95  {\tiny $\pm$  0.00}  & 0.95  {\tiny $\pm$  0.00}  & 0.96  {\tiny $\pm$  0.00}  & 0.97  {\tiny $\pm$  0.00} \\
\midrule
rob ($p=\infty$, $\varepsilon=1.00)$  & 0.74  {\tiny $\pm$  0.00}  & 0.75  {\tiny $\pm$  0.01}  & 0.84  {\tiny $\pm$  0.00}  & 0.91  {\tiny $\pm$  0.00} \\
rob ($p=\infty$, $\varepsilon=0.75)$  & 0.80  {\tiny $\pm$  0.00}  & 0.80  {\tiny $\pm$  0.00}  & 0.85  {\tiny $\pm$  0.00}  & 0.91  {\tiny $\pm$  0.00} \\
rob ($p=\infty$, $\varepsilon=0.50)$  & 0.87  {\tiny $\pm$  0.00}  & 0.88  {\tiny $\pm$  0.00}  & 0.90  {\tiny $\pm$  0.00}  & 0.94  {\tiny $\pm$  0.00} \\
\bottomrule
\\
\end{tabular}
}

\resizebox{.6\textwidth}{!}{%
        \begin{tabular}{lcccc
}
    \toprule
    \multirow{1}{*}{\makecell[c]{\textbf{Synthetic Holtsmark}}}
     & \multicolumn{1}{c}{\small $R=1$} 
     & \multicolumn{1}{c}{\small $R=2$} 
     & \multicolumn{1}{c}{\small $R=5$} 
     & \multicolumn{1}{c}{\small $R=10$} 
     \\\midrule
std & 0.51 {\tiny $\pm$  0.00} & 0.65 {\tiny $\pm$  0.00} & 0.82 {\tiny $\pm$  0.00} & 0.89 {\tiny $\pm$  0.00} \\\midrule
rob ($p=2$, $\varepsilon=70.00)$  & 0.51  {\tiny $\pm$  0.00}  & 0.66  {\tiny $\pm$  0.00}  & 0.83  {\tiny $\pm$  0.00}  & 0.90  {\tiny $\pm$  0.00} \\
rob ($p=2$, $\varepsilon=52.50)$  & 0.51  {\tiny $\pm$  0.01}  & 0.66  {\tiny $\pm$  0.00}  & 0.83  {\tiny $\pm$  0.00}  & 0.90  {\tiny $\pm$  0.00} \\
rob ($p=2$, $\varepsilon=35.00)$  & 0.53  {\tiny $\pm$  0.00}  & 0.66  {\tiny $\pm$  0.00}  & 0.82  {\tiny $\pm$  0.00}  & 0.90  {\tiny $\pm$  0.00} \\
\midrule
rob ($p=\infty$, $\varepsilon=53.00)$  & 0.51  {\tiny $\pm$  0.01}  & 0.67  {\tiny $\pm$  0.00}  & 0.83  {\tiny $\pm$  0.00}  & 0.91  {\tiny $\pm$  0.00} \\
rob ($p=\infty$, $\varepsilon=39.75)$  & 0.52  {\tiny $\pm$  0.01}  & 0.67  {\tiny $\pm$  0.00}  & 0.83  {\tiny $\pm$  0.00}  & 0.91  {\tiny $\pm$  0.00} \\
rob ($p=\infty$, $\varepsilon=26.50)$  & 0.50  {\tiny $\pm$  0.01}  & 0.66  {\tiny $\pm$  0.00}  & 0.83  {\tiny $\pm$  0.00}  & 0.91  {\tiny $\pm$  0.00} \\
\bottomrule
\\
\end{tabular}
}

\end{table}

\begin{table}[tbp]
    \centering

    \caption{\small \textbf{Overall accuracy} for the standard and robust classifiers on real-world datasets. 
    ``std'' refers to the standard classifier and ``rob'' refers to the robust classifier with the specified $p$ and $\varepsilon$.
    The presented results are averaged over $5$ runs.
    }
    \label{tab:acc-overall-2}

\vspace{1mm}

\resizebox{.6\textwidth}{!}{%
        \begin{tabular}{lcccc
}
    \toprule
    \multirow{1}{*}{\makecell[c]{\textbf{MNIST}}}
     & \multicolumn{1}{c}{\small $R=1$} 
     & \multicolumn{1}{c}{\small $R=2$} 
     & \multicolumn{1}{c}{\small $R=5$} 
     & \multicolumn{1}{c}{\small $R=10$} 
     \\\midrule
std & 1.00 {\tiny $\pm$  0.00} & 0.99 {\tiny $\pm$  0.00} & 0.99 {\tiny $\pm$  0.00} & 0.99 {\tiny $\pm$  0.00} \\\midrule
rob ($p=2$, $\varepsilon=2.70)$  & 0.98  {\tiny $\pm$  0.00}  & 0.98  {\tiny $\pm$  0.00}  & 0.98  {\tiny $\pm$  0.00}  & 0.98  {\tiny $\pm$  0.00} \\
rob ($p=2$, $\varepsilon=2.02)$  & 0.98  {\tiny $\pm$  0.00}  & 0.98  {\tiny $\pm$  0.00}  & 0.99  {\tiny $\pm$  0.00}  & 0.99  {\tiny $\pm$  0.00} \\
rob ($p=2$, $\varepsilon=1.35)$  & 0.99  {\tiny $\pm$  0.00}  & 0.99  {\tiny $\pm$  0.00}  & 0.99  {\tiny $\pm$  0.00}  & 0.99  {\tiny $\pm$  0.00} \\
\midrule
rob ($p=\infty$, $\varepsilon=0.43)$  & 0.97  {\tiny $\pm$  0.00}  & 0.97  {\tiny $\pm$  0.00}  & 0.98  {\tiny $\pm$  0.00}  & 0.99  {\tiny $\pm$  0.00} \\
rob ($p=\infty$, $\varepsilon=0.33)$  & 0.98  {\tiny $\pm$  0.00}  & 0.98  {\tiny $\pm$  0.00}  & 0.99  {\tiny $\pm$  0.00}  & 0.99  {\tiny $\pm$  0.00} \\
rob ($p=\infty$, $\varepsilon=0.22)$  & 0.98  {\tiny $\pm$  0.00}  & 0.98  {\tiny $\pm$  0.00}  & 0.99  {\tiny $\pm$  0.00}  & 0.99  {\tiny $\pm$  0.00} \\
\bottomrule
\\
\end{tabular}
}

\resizebox{.6\textwidth}{!}{%
        \begin{tabular}{lcccc
}
    \toprule
    \multirow{1}{*}{\makecell[c]{\textbf{Fashion-MNIST}}}
     & \multicolumn{1}{c}{\small $R=1$} 
     & \multicolumn{1}{c}{\small $R=2$} 
     & \multicolumn{1}{c}{\small $R=5$} 
     & \multicolumn{1}{c}{\small $R=10$} 
     \\\midrule
std & 0.99 {\tiny $\pm$  0.00} & 0.99 {\tiny $\pm$  0.00} & 0.99 {\tiny $\pm$  0.00} & 0.99 {\tiny $\pm$  0.00} \\\midrule
rob ($p=2$, $\varepsilon=3.23)$  & 0.95  {\tiny $\pm$  0.00}  & 0.96  {\tiny $\pm$  0.00}  & 0.97  {\tiny $\pm$  0.00}  & 0.98  {\tiny $\pm$  0.00} \\
rob ($p=2$, $\varepsilon=2.42)$  & 0.96  {\tiny $\pm$  0.00}  & 0.96  {\tiny $\pm$  0.00}  & 0.97  {\tiny $\pm$  0.00}  & 0.98  {\tiny $\pm$  0.00} \\
rob ($p=2$, $\varepsilon=1.62)$  & 0.98  {\tiny $\pm$  0.00}  & 0.98  {\tiny $\pm$  0.00}  & 0.99  {\tiny $\pm$  0.00}  & 0.99  {\tiny $\pm$  0.00} \\
\midrule
rob ($p=\infty$, $\varepsilon=0.29)$  & 0.94  {\tiny $\pm$  0.00}  & 0.96  {\tiny $\pm$  0.00}  & 0.97  {\tiny $\pm$  0.00}  & 0.98  {\tiny $\pm$  0.00} \\
rob ($p=\infty$, $\varepsilon=0.22)$  & 0.95  {\tiny $\pm$  0.00}  & 0.96  {\tiny $\pm$  0.00}  & 0.98  {\tiny $\pm$  0.00}  & 0.99  {\tiny $\pm$  0.00} \\
rob ($p=\infty$, $\varepsilon=0.14)$  & 0.97  {\tiny $\pm$  0.00}  & 0.98  {\tiny $\pm$  0.00}  & 0.99  {\tiny $\pm$  0.00}  & 0.99  {\tiny $\pm$  0.00} \\
\bottomrule
\\
\end{tabular}
}

\resizebox{.6\textwidth}{!}{%
        \begin{tabular}{lcccc
}
    \toprule
    \multirow{1}{*}{\makecell[c]{\textbf{CIFAR}}}
     & \multicolumn{1}{c}{\small $R=1$} 
     & \multicolumn{1}{c}{\small $R=2$} 
     & \multicolumn{1}{c}{\small $R=5$} 
     & \multicolumn{1}{c}{\small $R=10$} 
     \\\midrule
std & 0.72 {\tiny $\pm$  0.00} & 0.71 {\tiny $\pm$  0.00} & 0.84 {\tiny $\pm$  0.00} & 0.91 {\tiny $\pm$  0.00} \\\midrule
rob ($p=2$, $\varepsilon=0.98)$  & 0.52  {\tiny $\pm$  0.01}  & 0.67  {\tiny $\pm$  0.00}  & 0.83  {\tiny $\pm$  0.00}  & 0.91  {\tiny $\pm$  0.00} \\
rob ($p=2$, $\varepsilon=0.88)$  & 0.54  {\tiny $\pm$  0.01}  & 0.67  {\tiny $\pm$  0.00}  & 0.83  {\tiny $\pm$  0.00}  & 0.91  {\tiny $\pm$  0.00} \\
rob ($p=2$, $\varepsilon=0.78)$  & 0.61  {\tiny $\pm$  0.05}  & 0.71  {\tiny $\pm$  0.02}  & 0.83  {\tiny $\pm$  0.00}  & 0.91  {\tiny $\pm$  0.00} \\
\midrule
rob ($p=\infty$, $\varepsilon=0.030)$  & 0.53  {\tiny $\pm$  0.01}  & 0.67  {\tiny $\pm$  0.00}  & 0.83  {\tiny $\pm$  0.00}  & 0.91  {\tiny $\pm$  0.00} \\
rob ($p=\infty$, $\varepsilon=0.023)$  & 0.58  {\tiny $\pm$  0.01}  & 0.69  {\tiny $\pm$  0.00}  & 0.83  {\tiny $\pm$  0.00}  & 0.91  {\tiny $\pm$  0.00} \\
rob ($p=\infty$, $\varepsilon=0.015)$  & 0.74  {\tiny $\pm$  0.01}  & 0.72  {\tiny $\pm$  0.01}  & 0.84  {\tiny $\pm$  0.00}  & 0.91  {\tiny $\pm$  0.00} \\
\bottomrule
\\
\end{tabular}
}

{
\resizebox{.6\textwidth}{!}{%
        \begin{tabular}{lcccc
}
    \toprule
    \multirow{1}{*}{\makecell[c]{\textbf{ImageNet}}}
     & \multicolumn{1}{c}{\small $R=1$} 
     & \multicolumn{1}{c}{\small $R=2$} 
     & \multicolumn{1}{c}{\small $R=5$} 
     & \multicolumn{1}{c}{\small $R=10$} 
     \\\midrule
std & 0.99 {\tiny $\pm$  0.00} & 0.99 {\tiny $\pm$  0.00} & 0.99 {\tiny $\pm$  0.00} & 1.00 {\tiny $\pm$  0.00} \\\midrule
rob ($p=2$, $\varepsilon=4.80)$  & 0.94  {\tiny $\pm$  0.00}  & 0.93  {\tiny $\pm$  0.00}  & 0.93  {\tiny $\pm$  0.01}  & 0.95  {\tiny $\pm$  0.00} \\
rob ($p=2$, $\varepsilon=3.60)$  & 0.95  {\tiny $\pm$  0.00}  & 0.95  {\tiny $\pm$  0.00}  & 0.95  {\tiny $\pm$  0.00}  & 0.96  {\tiny $\pm$  0.00} \\
rob ($p=2$, $\varepsilon=2.40)$  & 0.97  {\tiny $\pm$  0.00}  & 0.96  {\tiny $\pm$  0.00}  & 0.97  {\tiny $\pm$  0.00}  & 0.97  {\tiny $\pm$  0.00} \\
\midrule
rob ($p=\infty$, $\varepsilon=0.090)$  & 0.93  {\tiny $\pm$  0.00}  & 0.93  {\tiny $\pm$  0.00}  & 0.94  {\tiny $\pm$  0.00}  & 0.95  {\tiny $\pm$  0.00} \\
rob ($p=\infty$, $\varepsilon=0.068)$  & 0.95  {\tiny $\pm$  0.00}  & 0.95  {\tiny $\pm$  0.00}  & 0.95  {\tiny $\pm$  0.00}  & 0.96  {\tiny $\pm$  0.00} \\
rob ($p=\infty$, $\varepsilon=0.045)$  & 0.97  {\tiny $\pm$  0.00}  & 0.96  {\tiny $\pm$  0.00}  & 0.97  {\tiny $\pm$  0.00}  & 0.97  {\tiny $\pm$  0.00} \\
\bottomrule
\\
\end{tabular}
}

}

\end{table}

}

In this part, we present the full experimental results on {seven} dataset groups---accuracy disparity gap in Table \ref{tab:ad}, standard accuracy gap in Table \ref{tab:acc}, per class accuracy in Table \ref{tab:acc-per-class-syn} and \ref{tab:acc-per-class-real}, and overall accuracy in Table \ref{tab:acc-overall-1} and \ref{tab:acc-overall-2}.
In each table, we present results for the standard and robust classifiers (with various $p$ and $\eps$) on datasets with different imbalance ratios $R$.

\subsection{Computational Resources and Runtime}
\label{adxsubsec:runtime}

We perform experiments on a machine with AMD EPYC 7352 24-Core Processor CPU and $8$ NVIDIA RTX A6000 GPUs.
The computational costs for training both the standard and the robust classifiers for both the synthetic datasets and the real world datasets are low. 
For adversarial training on CIFAR (the most expensive case), each run of training would take less than $20$ minutes.
We parallel the training tasks on all $8$ GPU cards.

\end{document}